\documentclass[11pt]{article}
\usepackage{fullpage}
\usepackage{wustyle}

\usepackage{braket}
\usepackage[a4paper,margin=1in]{geometry}
\usepackage[toc,page,header]{appendix}
\usepackage{minitoc}

\title{A duality framework for analyzing random feature and \\ two-layer neural networks}

\author{
Hongrui Chen\thanks{Department of Mathematics, Stanford University, \texttt{hongrui@stanford.edu}}
\and
Jihao Long\thanks{Institute for Advanced Algorithmic Research, Shanghai, \texttt{longjh1998@gmail.com}}
\and
Lei Wu\thanks{School of Mathematical Sciences, Peking University,  \texttt{leiwu@math.pku.edu.cn}}
\thanks{Center for Machine Learning Research, Peking University}
\thanks{AI for Science Institute, Beijing\\ \hspace*{1.6em}All authors contributed equally, and the order follows the alphabetical convention. This paper has been accepted for publication in  {\it Annals of Statistics}.}
}

\date{\today}

\begin{document}

\maketitle

\vspace*{-1em}
\begin{abstract}
We consider the problem of learning functions within the $\mathcal{F}_{p,\pi}$ and Barron spaces, which play crucial roles in understanding random feature models (RFMs), two-layer neural networks, as well as kernel methods. 
Leveraging tools from information-based complexity (IBC), we establish a  dual equivalence between approximation and estimation, and then apply it to study the learning of the preceding function spaces. The duality allows us to  focus on  the more tractable problem between approximation and estimation. To showcase the efficacy of our duality framework, we delve into  two important but under-explored problems:
\begin{itemize}
\item \textbf{Random feature learning beyond kernel regime:} 
We derive sharp bounds for learning $\cF_{p,\pi}$ using RFMs. Notably, the learning is efficient without the curse of dimensionality for $p>1$. This underscores the extended applicability of RFMs beyond the traditional kernel regime, since $\mathcal{F}_{p,\pi}$ with $p<2$ is strictly larger than the corresponding reproducing kernel Hilbert space (RKHS) where $p=2$.

\item  \textbf{The $L^\infty$ learning of RKHS:}  We establish sharp, spectrum-dependent characterizations for the convergence of $L^\infty$ learning error  in both noiseless and noisy settings. Surprisingly, we show that popular kernel ridge regression can achieve near-optimal performance in $L^\infty$ learning, despite it primarily minimizing square loss.
\end{itemize}

To establish the aforementioned duality, we introduce a type of IBC, termed  $I$-complexity, to measure the size of a function class. Notably, $I$-complexity offers a tight characterization of learning in noiseless settings, yields lower bounds comparable to  Le Cam's  in noisy settings, and is versatile in deriving upper bounds. We believe that our duality framework holds potential for broad application in learning analysis across more scenarios.
\end{abstract}

\doparttoc 
\faketableofcontents 
\part{} 

\vspace*{-1em}
\section{Introduction}
In this paper, we study the properties of  the following feature-based model:
\begin{equation}\label{eqn: rfm}
    f(x;\theta) = \frac{1}{m}\sum_{j=1}^m a_j \phi(x,v_j),
\end{equation}
where $\phi: \cX\times \cV \mapsto\RR$ is a  feature function and $m\in\NN$ denotes the number of features. The domains of input and weight  are represented by  $\cX$ and $\cV$, respectively.
Notable examples of feature-based models include random feature models (RFMs) \cite{Rahimi2007RandomFF} and two-layer neural networks \cite{barron1993universal}. For RFMs, $\{v_j\}_{j=1}^m$ are \iid samples drawn from a prefixed weight distribution $\pi\in \cP(\cV)$; the learnable parameters are the outer coefficients, i.e., $\theta=\{a_j\}_{j=1}^m$. In contrast, in two-layer neural networks, both the outer coefficients and inner-layer weights are learnable, i.e., $\theta=\{(a_j,v_j)\}_{j=1}^m$, making the features \textit{adaptive}.

RFMs were originally proposed as computationally efficient approximation of kernel methods and have been widely used in practice; see the survey~\cite{liu2021random}. More recently, RFMs  also gained in popularity as proxy models to study  over-parametrized neural networks in the lazy/kernel regime~\cites{daniely2017sgd,weinan2020comparative,arora2019exact,cao2019generalization,jacot2018neural,chizat2019lazy}.
Particularly, when the $\ell_2$ norm of the coefficients is regularized, RFMs are equivalent to kernel methods with the kernel given by $\hk_m(x,x'):=\frac{1}{m}\sum_{j=1}^m \phi(x,v_j)\phi(x',v_j)$, according to the representer theorem~\cites{scholkopf2001generalized}. When the feature number $m\to{\infty}$, the law of large numbers implies 
\begin{equation}\label{eqn: kernel}
    \hk_m(x,x')\to k_\pi(x,x') := \int_{\cV} \phi(x,v)\phi(x',v)\dd \pi(v),
\end{equation}
establishing RFMs as approximations of kernel methods with the kernel $k_\pi$. 

Naturally, existing theoretical analyses of RFMs (see, e.g.,~\cites{carratino2018learning,rahimi2008uniform,bach2017equivalence,liu2021random}) mostly focus on target functions within the corresponding RKHS~\cites{aronszajn1950theory}.  
Theoretically speaking, RKHS theory enjoys popularity in the analysis of kernel methods primarily for two major reasons. First, RKHS (under mild conditions on the associated reproducing kernel) can be efficiently learned using kernel methods in high dimensions~\cites{smale2007learning,novak2012tractability} without suffering from the curse of dimensionality (CoD)~\cites{bellman1966dynamic}, i.e., the learning error scales with  the input dimension $d$ at most polynomially.
Second, the inherent Hilbert structure of RKHS, coupled with corresponding spectral decomposition~\cites{mercer1909xvi}, offers a rich suite of mathematical tools that greatly simplify theoretical analysis~\cites{carratino2018learning,sobolevkernel,mei2019generalization,MEI2021,sobolevkernel}.  However, it should be stressed that RFMs are not kernel methods if  the $\ell_p$ norm  with $p<2$ of coefficients is regularized~\cites{Celentano2021MinimumCI,xie2022shrimp,hashemi2021generalization}, under which RKHS theory is not applicable anymore.  

As for two-layer neural networks, the theoretical analysis traces back to the pioneering work by Andrew Barron \cites{barron1993universal,barron1994approximation}. Specifically, Barron considered target functions that satisfy $\int (1+\|\xi\|) |\hat{f}(\xi)|\dd \xi<{\infty}$, where $\hat{f}$ denotes the Fourier transform of $f$. He proved that two-layer neural networks can overcome the CoD, but all linear methods, including RFMs, inevitably suffer the CoD in learning these functions.
Subsequent studies such as \cites{barron1992neural,devore1998nonlinear,kurkova2001bounds,bach2017breaking,weinan2021barron} have refined and extended Barron's results in various aspects. Notably, \cites{weinan2021barron} established a precise connection between two-layer neural networks and RFMs, suggesting that two-layer neural networks can be essentially viewed as RFMs equipped with an optimally-chosen weight distribution.

\subsection{Our contributions}\label{sec: contribution}
In the  work, we introduce a {\em duality framework} coupled with the associated function spaces to unify the analysis of RFMs and two-layer neural networks. This framework  not only allows us to analyze RFMs in the  beyond-kernel regime where $p<2$, but also produces new results within the traditional kernel regime. 

In Section~\ref{sec: information-based complexity}, we first establish  an abstract duality between approximation and estimation by utilizing tools from the information-based complexity (IBC)~\cite{traub2003information} (see the related work section for an overview of IBC). Specifically, we define a type of IBC, coined $I$-complexity, to measure the learnability of a function class. On the one hand, we show that $I$-complexity precisely characterizes the corresponding approximation in the dual spaces (Theorem~\ref{thm:dual_abstract}). On the other hand, we  demonstrate its utility in  bounding estimation errors in the primal spaces across various scenarios: 1) in the noiseless setting, $I$-complexity  provides  matching bounds for estimation errors~(Theorem~\ref{lemma-data-dependent}); 2) in the noisy setting,  it yields a lower bound (Theorem~\ref{lemma-random-data}), which works comparably with Le Cam's~\cite{lecam2012asymptotic}, but without needing explicit construction of hard functions;
3) moreover, $I$-complexity also serves as a versatile tool  for deriving upper bounds of estimation errors (see Section~\ref{sec: deriving-upper-bound} for details).

Next, we  apply this abstract framework to the $\cF_{p,\pi}$ and Barron spaces, which play crucial roles in the analysis of RFMs and two-layer neural networks.  Notably, $\cF_{p,\pi}$ with $p=2$ is a RKHS, corresponds to the kernel regime and the case where $p<2$  corresponds to a beyond-kernel regime. Specifically, we derive  concrete duality between approximation  and  estimation for these function spaces  in Section~\ref{sec: dual-equivalence}. Remarkably, the corresponding approximation problems exhibit  great interpretation and  we refer to Theorems~\ref{thm: dual-simple} and \ref{thm: dual} for insights.

Lastly, to demonstrate the versatility and flexibility of our duality framework, we investigate two important yet less-explored problems: 
\begin{itemize}
\item In Section~\ref{sec: random-app}, we conduct a systematic analysis for the performance of RFMs in learning target functions within the $\cF_{p,\pi}$ spaces. We derive sharp bounds for both the approximation error (Theorem~\ref{thm: approximation-bound}) and total error (Theorem~\ref{Fp-estimate}). Particularly, we show that as long as $p>1$, $\cF_{p,\pi}$ can be learned using RFMs  efficiently across all aspects---approximation, estimation, and optimization---without suffering from the CoD. These results indicate that RFMs can  work well even  beyond the traditional kernel regime.

\item In Section~\ref{sec: uniform-estimation}, we tackle the   $L^\infty$ learning of RKHS functions. Under mild conditions, we show that  the $L^\infty$ learnability  of RKHS can be tightly characterized using the kernel's eigenvalue decay (Theorems~\ref{primal-bound} and \ref{Lq-estimation-rkhs}). Specifically, we establish lower bounds for the minimax estimation errors and derive matching upper bounds for the vanilla kernel ridge regression (KRR) estimator.  Our findings indicate that KRR is near-optimal for $L^\infty$ learning of RKHS in both noisy and noiseless settings. Additionally, we provide a refined analysis of periodic kernels over the one-dimensional torus. 
\end{itemize}

\subsection{Related work}
\label{sec: related-work}

\paragraph*{Information-based complexity (IBC).} IBC has emerged as a branch of continuous computational complexity~\cite{traub2003information}, providing an abstract framework to study the intrinsic challenges  of approximating solutions to problems where information is partial, contaminated, and priced. IBC has been widely applied across enormous areas, including solving differential equations~\cites{traub1982iterative,werschulz1991computational},  high-dimensional integration~\cites{wozniakowski1991average,novak2001quantum}, and function/operator approximation~\cites{ritter2000average,novak1988deterministic,novak2008tractability}. 
We direct interested readers  to \cite{traub2009brief} for a brief history of IBC. It is important to note that in the context of IBC, ``information'' refers specifically to the partial information available about the problem, rather than the Shannon information~\cite{shannon1948mathematical} used in information theory.

In the current work, we leverage IBC to study statistical and machine learning methods in supervised learning contexts. Here, the ``partial information'' specifically refers to the availability of only finite inputs and their corresponding function values. Our approaches build on  Erich Novak's work in applying IBC to study function approximation with finite samples, as detailed in his habilitation thesis~\cites{novak1988deterministic}. We extend Novak's work in various aspects. First, we broaden the application of IBC to the analysis of general estimators, moving beyond Novak's focus on interpolation, by considering a modified $I$-complexity. Second, we apply IBC to analyze  popular  machine learning methods, such as RFMs, two-layer neural networks, and KRR. In contrast, the Novak's work  focused on traditional numerical methods. For more details, we refer to Section~\ref{sec: information-based complexity}.

\paragraph*{Learning the $\cF_{p,\pi}$ spaces.} \cites{Celentano2021MinimumCI} also considered the problem of learning  $\cF_{p,\pi}$ using RFMs but limited to the noiseless setting. Moreover, their analysis applies only  to the highly overparameterized regime: $m \gg (n \log n)^2$, where $m$ is the number of features and $n$ is the sample size. In contrast, we eliminate the need of overparameterization and extend the applicability to both  noiseless and noisy settings  by leveraging our duality framework in Section \ref{sec: random-app}.

\paragraph*{The $L^\infty$ learning of RKHS.} 
\cites{Kuo2008MultivariateLA} studied the $L^\infty$ learning of RKHS in the noiseless setting. They considered  algorithms of the form $A(f)=\sumin L_i(f)\varphi_i$, where $\{\varphi_i\}_{i=1}^n$ are a set of functions in $C_0(\cX)$ and $\{L_i\}_{i=1}^n$ are bounded linear functionals that extract linear information from $f$. Notably, when $\{L_i\}_{i=1}^n$ are  evaluation functionals, i.e., $L_i(f)=f(x_i)$ for some $x_i\in\cX$, this algorithm class includes the KRR estimator. Assuming the kernel's eigenvalue decays as $\mu_j\asymp j^{-1-\beta}$, \cites{Kuo2008MultivariateLA} showed that optimal error rate within the preceding algorithm class is $O(n^{-\beta/2})$  and the spectral method $A(f)=\sumin \langle f,e_i\rangle_\rho e_i$ can achieve this rate, where $\{e_i\}_{i\in\NN^+}$ are the corresponding eigenfunctions. Subsequently, the same authors further examined algorithms that use only function-value information in \cites{kuo2009power}, a setting consistent with statistical learning, yet the obtained rate $O(n^{-a\beta/2})$ with $a=(1+\beta)/(3+\beta)<1$ is suboptimal. More recently,
\cites{Pozharska2022ANO} proposed a weighted least-square interpolation method that achieves the optimal rate. However, this method requires to access  eigenfunctions and eigenvalues to build informative ``weights''. In stark contrast, we demonstrate that simple KRR, which uses only the function-value information, can achieve the optimal rate $O(n^{-\beta/2})$.

As for the noisy settings, \cites{sobolevkernel} established an upper bound of $O(n^{-(\beta-\tau)/(2\beta+4)})$ for well-tuned KRR, where $\tau$ can be arbitrarily small. In contrast, we provide a lower bound $\Omega(n^{-\beta/(2\beta+2)})$ and show that this lower bound is sharp for periodic kernels on the one-dimensional torus $\TT$, where KRR achieves the optimal rate. Our result indicates that the upper bound in \cites{sobolevkernel} is not tight.

\paragraph*{Connection with \cites{perturbation}.}  We acknowledge that \cites{perturbation} employed a similar approach, centering their analysis on reinforcement learning through a concept termed perturbation complexity. Their investigation primarily targets  the reward functions within RKHS (i.e., $\cF_{2,\pi}$). In contrast, our work is dedicated to supervised learning  and offers an exhaustive analysis covering the $\cF_{p,\pi}$ spaces for all $p\geq 1$, as well as extending to the Barron spaces. 

\subsection{Organization}
The rest of the paper is organized as follows. In Section \ref{sec: pre}, we clarify  notations and preliminaries.
In Section \ref{sec: information-based complexity}, we develop the abstract duality framework using information-based complexity.
In Section \ref{sec: spaces}, we define the $\cF_{p,\pi}$ and Barron spaces for RFMs and two-layer neural networks, respectively. In Section \ref{sec: dual-equivalence}, we apply the abstract duality framework to the $\cF_{p,\pi}$ and Barron spaces, establishing concrete duality between  approximation and estimation for them. Lastly,   we present two  applications to demonstrate the power of our duality framework: Random feature learning beyond kernel regime in Section \ref{sec: random-app} and the $L^\infty$ learning of  RKHS in Section  \ref{sec: uniform-estimation}.

\section{Preliminaries}  
\label{sec: pre}

\paragraph*{Notations.} 
 Let $\Omega$ be a subset of a Euclidean space. We denote  by $\cP(\Omega)$ the set of probability measures on $\Omega$ and $\cM(\Omega)$ the space of signed Radon measures equipped with the total variation norm $\|\mu\|_{\cM(\Omega)}=\|\mu\|_{\mathrm{TV}}$.  
 For any $\rho \in \cP(\Omega)$, let $\|\cdot\|_{p,\rho}$ be the $L^p(\rho)$ norm. When $p=2$, we write $\|\cdot\|_{\rho}=\|\cdot\|_{2,\rho}$ for simplicity and  let $\langle f, g\rangle_\rho = \int f(x)g(x)\dd \rho(x)$ for any $f,g\in L^2(\rho)$.
Let $C_0(\Omega)$ be the space of continuous functions vanishing at infinity equipped with the uniform norm ($L^\infty$ norm), i.e., $\|g\|_{C_0(\Omega)} := \sup_{x\in \Omega} |g(x)|$.
We shall frequently use $L^\infty$  to denote $\|\cdot\|_{C_0(\Omega)}$ for clarity and one should not confuse $L^\infty$ with $L^\infty(\rho)$.
Given a Banach space $\cA$, denote by $\cA^*$  the dual space of $\cA$ and let  $\cA(r):= \{f \in \cA: \|f\|_{\cA} \leq r \}$.
Given two Banach spaces $\cA$ and $\cB$, we  use $\cA \hookrightarrow \cB$ to denote that $\cA$ can be continuously embedded into $\cB$, i.e., $\sup_{\|f\|_{\cA} \le 1} \|f\|_{\cB} < \infty$. 

 For any $p\in [1,{\infty}]$, denote by $p'$  its H\"older conjugate, i.e., $1/p+1/p'=1$.
 For a vector $v$, denote by $\|v\|_p$ its $\ell^p$ norm and  when $p=2$, we drop the subscript for brevity. For an integer $n$, let $[n]=\{1,2,\dots,n\}$.
We use $a \lesssim b$ to mean $a \leq Cb$ for an absolute constant $C > 0$ and $a \gtrsim b$ is defined analogously. We write $a \asymp b$ if there exist absolute constants $C_1, C_2 > 0$ such that $C_1b \leq a \leq C_2b$. We will also use standard big-O notations, such as $O(\cdot),\Omega(\cdot)$, to hide constants and $\tilde{O}(\cdot)$ and $\tilde{\Omega}(\cdot)$ to further hide logarithimic factors.

\paragraph*{Kernel and RKHS.}
A function $k:\cX\times\cX \mapsto \RR$ is said to be  a (positive definite) kernel if  $k(x,x') = k(x',x), \forall x,x' \in \cX$ and  for any $n \in\NN^{+}$ and $x_1,\dots,x_n \in \cX$, the kernel matrix $(k(x_i,x_j))_{i,j}\in\RR^{n\times n}$ is positive semidefinite. Given a kernel $k$, there exists a unique Hilbert space $\cH_k$ such that  
\begin{itemize}
\item $k(x,\cdot)\in \cH_k, \forall x\in \cX$
\item $\langle f,k(x,\cdot)\rangle_{\cH_k} = f(x), \forall f\in\cH,x\in\cX$.
\end{itemize}
The second property is known as the reproducing property and consequently, $\cH_k$ is often referred to as the RKHS associated with $k$.  

Given a $\gamma\in\cP(\cX)$, consider 
the   integral operator $ \cT_k: L^2(\gamma)\mapsto L^2(\gamma)$ given  by 
$
        \cT_k f = \int_{\cX} k(\cdot,x) f(x) \dd \gamma(x).
$
When  $\int_{\cX} k(x,x)\dd\gamma(x)<\infty$, the Mercer's theorem  guarantees the existence of spectral decomposition for $\cT_k$:
$k(x,x')=\sum_{j=1}^{\infty} \mu_j e_j(x)e_j(x')$~\cites{mercer1909xvi,steinwart2012mercer}. Here $\{\mu_j\}_{j=1}^{\infty}$ are the eigenvalues in a decreasing order and $\{e_j\}_{j=1}^{\infty}$ are the corresponding orthonormal eigenfunctions. 
Note that this decomposition depends on the  distribution $\gamma$. When necessary, we will denote by $\mu_j^{(k,\gamma)}$ the $j$-th eigenvalue  to explicitly highlight the dependence on  $k$ and $\gamma$. In this paper, we always assume the existence of spectral decomposition for kernels, referring them as \textit{Mercer kernels}. Moreover, it holds that 
\[
\|f\|_{\cH_k}^2=\sum_{j=1}^\infty \mu_j^{-1}\langle f, e_j\rangle_\gamma^2.
\]
For more materials about RKHS, we refer to \cites{steinwart2008support,wainwright2019high}.

\paragraph*{Periodic kernels.} Let $\TT=[0,1)$ be the one-dimensional torus $\RR/\ZZ$, with the metric $|x-y| = \min\{ y - x, 1 + x- y\}$ for $0 \le x \le y < 1$. Consider the kernel $k:\TT\times \TT\mapsto\RR$ given by
\begin{equation}\label{eq:periodic}
 k(x,x') = \sum_{j\in \ZZ}\mu_j e_j(x)\overline{e_j(x')} =\sum_{j\in\ZZ}\mu_j e^{2\pi\ii j(x-x')}, 
\end{equation}
where for $j\in\ZZ$, $e_j(x) := \exp(2\pi \ii j x)$ is the $j$-th eigenfunction  with $\ii$ denoting the imaginary unit. We assume $\mu_j=\mu_{-j}$ to ensure  $k$ is real-valued.
Additionally, we assume  $\mu_j \asymp (|j|+1)^{-\beta-1}$, for which the associated RKHS  is then the  Sobolev space of order $(1+\beta)/2$ on $\TT$~\cites{wahba1990spline}. We will  consider this type of kernels in Sections \ref{sec: random-app} and \ref{sec: uniform-estimation} to obtain sharper characterizations.

\section{Estimation errors and information-based complexity} 
\label{sec: information-based complexity}

Throughout this paper, we consider the problem of learning a target function $f^*:\cX\mapsto\RR$ using the finite training data $S=\{(x_i,f^*(x_i)+\xi_i)\}_{i=1}^n$.  Here, each $x_i\in\cX$ is an input   and $\xi_i$ represents the  label noise. 
We assume the noise to be independent of inputs and the inputs are sampled by $x_i\sim\rho_i$ for each $i\in [n]$. Let $\bar{\rho} = \frac{1}{n}\sum_{i=1}^n \rho_i$ be the average distribution. This setting includes two classical statistical setups:
\begin{itemize}
\item \textbf{Random design.} The inputs $\{x_i\}_{i=1}^n$ are \iid samples drawn from an input distribution $\rho\in\cP(\cX)$, under which $\bar{\rho}=\rho$. This corresponds to the statistical learning setting~\cites{hastie2009elements}.
\item \textbf{Fixed design.} The inputs  $\{x_i\}_{i=1}^n$ are pre-determined and non-stochastic, for which $\bar{\rho} = \frac{1}{n}\sum_{i=1}^n \delta_{x_i} := \hat{\rho}_n$. This setting is relevant in applications such as experimental design, bandits, and scientific computing, where we can design the inputs using certain strategy.
\end{itemize}

To quantify the performance of a learning algorithm, we introduce three function spaces or norms: the space of target functions $\cF$, the norm used to measure training errors  $\cQ$, and  the norm used to measure test performance  $\cM$. 
\begin{assumption}\label{assumption: metrics}
Let $\cF,\cQ,\cM$ be Banach spaces of functions defined on $\cX$. Assume further that $\cF \hookrightarrow \cM$ and $\cF \hookrightarrow \cQ$.
\end{assumption}

We call each measurable map from $(\cX \times \RR)^n $ to $\cF$ an estimator  and denote by $\cA_n$ the set of all possible estimators. Given a $T\in\cA_n$, its performance is evaluated by the estimation error:
\begin{align}\label{eqn: estimation-err}
   \|T\left(\{(x_i,f^*(x_i)+\xi_i)\}_{i=1}^n\right) - f^*\|_{\cM}.
\end{align}
Additionally,  we also consider the minimax estimation error:
\begin{align} \label{random-data}
    \inf_{T \in \cA_n}\sup_{\|f\|_\cF \leq 1} \EE\|T(\{(x_i, f(x_i)+\xi_i)\}_{i=1}^n) - f \|_{\cM},
\end{align}
where the expectation is taken over the sampling of the training data $S$. This minimax estimation error serves as a benchmark for the best possible performance that an estimator can achieve, helping to determine if a specific estimator is information-theoretically optimal.

It is evident that estimation error depends on the complexity of the function class $\cF$. 
In this paper, we employ  the following complexity, considered in \cites{novak1988deterministic}:
\begin{equation}\label{eqn: I_0}
    \II_n(\cF) = \inf_{x_1,\dots,x_n\in\cX}\sup_{f\in\cF, f(x_i)=0} \|f\|_{C_0(\cX)}.
\end{equation}
We refer \eqref{eqn: I_0} as the $I$-complexity of $\cF$.
 Intuitively, $\II_n(\cF)$ measures the size of $\cF$ by  the largest   $L^\infty$ norm of functions within $\cF$ that interpolate the zero function at arbitrary $n$ points. Moreover, $\II_n(\cF)$ can provide a tight bound on the minimax estimation error for learning $\cF$ in a noiseless setting, as detailed in the proposition on page 22 of \cites{novak1988deterministic}.  
 % Moreover, explicit estimates of $\II_n(\cF)$ for classical function classes, including H\"older and Sobolev spaces, have been established there.
% In fact, $I_n$ is equal to the Kolmogorov $n$-width \cites{kolmogoroff1936uber} of $\cB$ with respect to $L^\infty$ for some various function spaces \cite[Section 1.3.7]{novak2006deterministic}.

However, using definition  \eqref{eqn: I_0} presents two major inconveniences in statistical learning setup. First,  the complexity ideally should depend on the specific information of the given $n$ inputs, which, in the setting of  random design,  is determined by the input distribution $\rho$. Yet,
$\II_n(\cF)$ evaluates the worst-case scenario  across all possible $n$ inputs. Secondly, this definition focuses exclusively  on the interpolation regime, whereas interpolation should generally be avoided to prevent overfitting in the presence of noise.
To overcome these limitations, we consider the following modified  $I$-complexity:

\begin{definition} [$I$-complexity]\label{def-complexity} For $\epsilon \geq 0$,  we define
$
\II_{\cQ,\cM}(\cF,\epsilon) =  \sup_{f\in \cF,\, \|f\|_{\cQ} \leq \epsilon } \|f\|_{\cM}.
$
\end{definition}
It is worth noting that this definition  allows to measure training and test errors with different metrics. This flexibility is particularly useful in applications, where the test performance is evaluated under metrics different from the one used during training. For instance,
in Section \ref{sec: uniform-estimation}, we consider the scenario where  the model is trained via minimizing square loss, but the test performance is evaluated under the $L^\infty$ norm. In the subsequent sections, we will demonstrate that this $I$-complexity can bound estimation errors in various settings. 

We note that a similar definition appears in the optimal recovery theory, where  $\II_{\cQ,\cM}(\cF,\epsilon) $ is referred to as the intrinsic error; see, e.g., \cite{micchelli1977survey}. However, it is important to point out that in optimal recovery, uniform bounded noise is typically assumed, and $\epsilon$ represents the noise level. In our context, $\epsilon$ should be treated as a varying threshold, which may differ depending on the specific setting. For instance, in the derivation of lower bound for the noisy setting in Section~\ref{sec: Le Cam}, we set $\epsilon=\sigma/\sqrt{n}$, where $\sigma$ represents the magnitude of the random noise.

The next theorem shows that $I$-complexity characterizes precisely an approximation problem in dual spaces:

\begin{theorem}\label{thm:dual_abstract}
   Under Assumption~\ref{assumption: metrics},  it holds for any $\epsilon \geq 0$ that
    \begin{equation}\label{eqn:dual_abstract}
   \II_{\cQ,\cM}(\cF(1),\epsilon) = \sup_{\|g\|_{\cM^*} \le 1}\inf_{h \in \cQ^* }[\|g - h\|_{\cF^*} + \epsilon\|h\|_{\cQ^*}].
    \end{equation}
\end{theorem}
The right-hand side (RHS)  quantifies  how  well $\cM^*$ can be approximated by $\cQ^*$ with respect to the $\cF^*$ norm. The left-hand side (LHS) is the $I$-complexity that determines the estimation errors of learning $\cF$, as detailed later. As such, this theorem establishes a general dual equivalence between estimation and approximation. Later, we will demonstrate how  this  duality can be applied to study the learning properties of   RFMs and two-layer neural networks in Section~\ref{sec: dual-equivalence}.

The proof of Theorem~\ref{thm:dual_abstract} is based on the following lemma. 
\begin{lemma}\label{lem:dual_core}
    Assume that $\cF \hookrightarrow \cQ$. Then for any $g \in \cF^*$,
    \begin{equation}\label{eqn:dual_core}
        \sup_{\|f\|_{\cF}\le 1, \|f\|_{\cQ} \le \epsilon} g(f) = \inf_{h \in \cQ^*}\left[\|g - h\|_{\cF^*} + \epsilon\|h\|_{\cQ^*}\right].
    \end{equation}
\end{lemma}

Here we only provide an intuitive proof by assuming the validity of  Von Neumann's Minimax Theorem \cites{willem2012minimax}. The rigorous proof is deferred to  Appendix \ref{sec: proof-lemma-dual}, which is based on the Fenchel-Rockafellar theorem (see, e.g.,  \cite[Theorem 1.12]{brezis2011functional}).
\begin{proof}[An intuitive proof of Lemma~\ref{lem:dual_core}]
    Since $\cF \hookrightarrow \cQ$, we have $\cQ^* \hookrightarrow \cF^*$.  Therefore, the RHS of \eqref{eqn:dual_core} is well-defined.
    Observe that
    \begin{equation*}
        \inf_{h \in \cQ^*}[\|g - h\|_{\cF^*} + \epsilon\|h\|_{\cQ^*}] = \inf_{h \in \cQ^*}\sup_{\|f\|_{\cF} \le 1}[g(f)-h(f) + \epsilon\|h\|_{\cQ^*}].
    \end{equation*}
    Note that $\cL(h,f) := g(f) - h(f) + \epsilon\|h\|_{\cQ^*}$ is  convex in $h$ and concave in $f$. Then, by applying Von Neumann's Minimax Theorem, we can exchange the order of $\inf_{h \in \cQ^*}$ and $\sup_{\|f\|_{\cF} \le 1}$, resulting in
    \begin{align*}
         \inf_{h \in \cQ^*}[\|g - h\|_{\cF^*} + \epsilon\|h\|_{\cQ^*}] &=\sup_{\|f\|_{\cF} \le 1}\inf_{h \in \cQ^*}[ g(f) +  \epsilon\|h\|_{\cQ^*} - h(f)]\\ 
         &=\sup_{\|f\|_{\cF} \le 1} \left(g(f) +  \inf_{h \in \cQ^*}\left[\epsilon\|h\|_{\cQ^*} - h(f)\right]\right).
    \end{align*}
Then, the proof is completed by noting
    \begin{align*}
        \inf_{h \in \cQ^*}[ \epsilon\|h\|_{\cQ^*} - h(f)]&=\inf_{\gamma\in\RR_{\geq 0}, \|h\|_{\cQ^*}=1}\gamma[ \epsilon - h(f)] \\ 
        &=\inf_{\gamma\in\RR_{\geq 0}}\gamma[\epsilon-\|f\|_{\cQ}]= \begin{cases}
            0\, &\text{ if }\|f\|_{\cQ} \le \epsilon \\
            -\infty \, &\text{ if }\|f\|_{\cQ} > \epsilon
        \end{cases}
        .
    \end{align*}
\end{proof}

\begin{proof}[Proof of Theorem \ref{thm:dual_abstract}]
Since $\cF \hookrightarrow \cM$, we have $\cM^*\hookrightarrow\cF^*$. If $g\in\cM^*$, we also have $g\in \cF^*$. Then, by Lemma~\ref{lem:dual_core}, 
\begin{equation*}
     \sup_{\|f\|_{\cF}\le 1, \|f\|_{\cQ} \le \epsilon} g(f) = \inf_{h \in \cQ^*}\left[\|g - h\|_{\cF^*} + \epsilon\|h\|_{\cQ^*}\right].
\end{equation*}
Taking the supremum over $\{g \in \cM^*\,:\, \|g\|_{\cM^*} \le 1\}$
on the both sides gives
\begin{align*}
    \sup_{\|g\|_{\cM^*} \le 1}\inf_{h \in \cQ^* }[\|g - h\|_{\cF^*} + \epsilon\|h\|_{\cQ^*}] &=\sup_{\|g\|_{\cM^*} \le 1} \sup_{\|f\|_{\cF}\le 1, \|f\|_{\cQ} \le \epsilon} g(f)\\
    &=  \sup_{\|f\|_{\cF}\le 1, \|f\|_{\cQ} \le \epsilon} \sup_{\|g\|_{\cM^*} \le 1}g(f)\\
   &\stackrel{\mathrm{(i)}}{=} \sup_{\|f\|_{\cF}\le 1, \|f\|_{\cQ} \le \epsilon} \|f\|_{\cM} =  \II_{\cQ,\cM}(\cF,\epsilon),  
\end{align*}
where $\mathrm{(i)}$ uses $\|f\|_{\cM^{**}}=\|f\|_{\cM}$ for $f\in\cM$.
\end{proof}

% We next  show that $I$-complexity can be used to bound estimation errors in various scenarios. 

\subsection{Minimax errors in the noiseless setting: tight bound}\label{sec: I-complexity-noiseless}
We first show that  $I$-complexity surprisingly provides a {\it tight} bound on the minimax estimation errors in the noiseless setting.
\begin{proposition}\label{lemma-data-dependent}
 For fixed $x_1,\cdots,x_n \in \cX$, recall that $\bar{\rho}=\fn\sumin\delta_{x_i}$. Then, 
 \begin{align*}
      \inf_{T_n \in \cA_n}\sup_{\|f\|_{\cF} \leq 1}     \|T(\{(x_i, f(x_i))\}_{i=1}^n) - f \|_{\cM} \asymp   \II_{L^2(\bar{\rho}),\cM}(\cF(1),0).
 \end{align*}
Moreover, this optimal error can be achieved 
by the minimum-norm interpolation (MNI):
\[
       S(\{(x_i,y_i)\}_{i=1}^n) = \argmin_{h \in \cF,\,h(x_i) = y_i} \|h\|_{\cF}.
\]
\end{proposition}
We remark that although MNI is optimal in the noiseless setting, in practice, we may still favor other estimators due to potential challenges in implementing MNI. Nevertheless, Proposition~\ref{lemma-data-dependent} provides a benchmark to assess the performance of a given estimator  in such noiseless setting. 

It is also worth noting that nowadays, statistics and ML methods have been increasingly applied  in scientific computing (aka AI for science)~\cite{weinan2021dawning}. In many such applications, e.g., the numerical approximation of partial differential equation solutions~\cites{sirignano2018dgm,yu2018deep,han2018solving,raissi2019physics}, the data are generated via computer simulations where noise can be effectively  considered negligible. Therefore, a refined analysis of estimation error in the noiseless setting becomes crucial, although existing analysis primarily focused on the noisy setting.
For instance, recent research, such as studies~\cites{cui2021generalization} and \cites{li2024asymptotic}, have delved into KRR in such setting.

\subsection{Minimax errors in the noisy setting: lower bound}
\label{sec: Le Cam}

The next proposition shows that $I$-complexity can also give a lower bound for the minimax estimation error in the noisy setting:
\begin{proposition} \label{lemma-random-data}
 Suppose $x_i \sim \rho_i, \xi_i \stackrel{iid}{\sim} \cN(0, \sigma^2)$ for $i\in [n]$. Recall $\bar{\rho} = \frac{1}{n}\sum_{i=1}^n \rho_i$. We have
\begin{align*}
      \inf_{T \in \cA_n}\sup_{\|f\|_\cF \leq 1}      \mathbf{E}\|T(\{(x_i, f(x_i)+\xi_i)\}_{i=1}^n) - f \|_{\cM} \gtrsim \II_{L^2(\bar{\rho}),\cM}\left(\cF(1),\frac{\sigma}{\sqrt{n}}\right)
\end{align*}
\end{proposition}
% Note that in some specific cases, the $I$-complexity also provides an upper bound of the minimax error \eqref{random-data}. We refer to Section \ref{sec: uniform-estimation} for details.
\begin{proof}
Let $d_n(T,f):=\|T(\{(x_i,f(x_i)+\xi_i)\})-f\|_{\cM}$. Then, the  Le Cam's  method \cite[Chapter 15.2]{wainwright2019high} gives the following lower bound:
\begin{align} \label{uvw}
  \inf_{T \in \cA_n} \sup_{\|f\|_\cF \leq 1} \EE[d_n(T,f)] \geq \frac{\|f_1-f_2\|_\cM}{2}\left(1 - \frac{\sqrt{n}\|f_1-f_2\|_{\brho}}{2\sigma} \right) \quad \forall f_1,f_2\in\cF(1).
\end{align}
Noticing 
$
\II_{L^2(\brho),\cM}\left(\cF(1), \frac{\sigma}{\sqrt{n}}\right)=\sup_{\|f\|_{\cF} \leq 1,\,\|f\|_{\bar{\rho}} \leq \sigma /\sqrt{n}} \|f\|_\cM,
$
there must exist a  $\tf\in \cF(1)$ such that 
\begin{align*}
         \|\tf \|_{\brho} \leq \frac{\sigma}{\sqrt{n}},          \quad    \|\tf\|_{\cM} \geq \frac{2}{3} \II_{L^2(\brho),\cM}\left(\cF(1), \frac{\sigma}{\sqrt{n}}\right).
\end{align*}
Substituting $f_1 = 0, f_2 = \tf$ into \eqref{uvw} yields
\begin{align*}
     \inf_{T\in\cA_n}\sup_{\|f\|_\cF \leq 1} \EE [d_n(T,f)] \geq \frac{1}{6} \II_{L^2(\brho),\cM}\left(\cF(1), \frac{\sigma}{\sqrt{n}}\right).
\end{align*}
Thus, we complete the proof.
\end{proof}

Note that the lower bound~\eqref{uvw} is standard from Le Cam's method and  for completeness, we have provided a proof of it in Appendix~\ref{sec: proof-pro-noisy}. It is evident that our lower bound essentially follows the same strategy as the Le Cam's method. The key step involves constructing a two-point packing, denoted as $\{f_1,f_2\}$, such that the norm $\|f_1 - f_2\|_{\brho}$ is as small as possible, while the norm $\|f_1 - f_2\|_{\cM}$ is as large as possible.  Constructing such hard functions explicitly can be highly challenging, especially when $\cM$ is a norm different from $L^2$. Proposition~\ref{lemma-random-data} provides a way to circumvent  explicit constructions, requiring only the $I$-complexity, which is further equated to an approximation problem by Theorem~\ref{thm:dual_abstract}. Therefore, instead of explicitly constructing hard functions, we can leverage tools from approximation theory to establish the lower bounds.

Note that Le Cam's method can sometimes yield loose bounds in certain cases.  
In contrast, Fano's method, which uses appropriate multiple-point  packing strategy, typically gives sharp bounds~\cite{wainwright2019high}. Moreover, it was shown in \cite{yang1999information} that the explicit packing-set construction  can be obviated by using the information of global metric entropy, known as Yang-Barron version of Fano's method. In spirit, our Proposition~\ref{lemma-random-data}  can be seen as an analogous adaptation to Le Cam's method.   Nevertheless, in Section~\ref{sec: uniform-estimation}, by using Proposition~\ref{lemma-random-data} along with the duality given by Theorem~\ref{thm:dual_abstract}, we obtain tight bounds on the minimax error for the $L^\infty$ learning of RKHS.

\subsection{Deriving upper bounds} 
\label{sec: deriving-upper-bound}
In this section, we demonstrate how  $I$-complexity can be useful in deriving upper bounds of estimation errors. Suppose that we have an existing upper bound with respect to the $Q$ norm: $\|\hf-f^*\|_{\cQ}\leq \epsilon$, where $f^*\in\cF(1)$ and $\hf\in\cF(R)$ is an estimate of $f^*$. We can then obtain an upper bound with respect to  the $M$ norm as follows:
\begin{align}  \label{1}
       \|\hat{f}-f^*\|_{\cM}&\leq     \sup_{\|g\|_{\cF} \leq (1+R),\,\|g\|_{\cQ} \leq \epsilon } \|g\|_{\cM} = \II_{\cQ,\cM}(\cF(1+R),\epsilon),
\end{align}
where the first step is due to $\|\hf-f^*\|_{\cM}\leq \|\hf\|_{\cM} + \|f^*\|_{\cM}\leq 1+R$. This illustrates that $I$-complexity can transfer an upper bound of estimation error from one norm to another. Moreover, the $I$-complexity itself can be further estimated using Theorem~\ref{thm:dual_abstract}. The flexibility to select appropriate norms $\cQ$ and $\cM$, tailored to different scenarios, is a significant advantage of this approach. Specifically, in Sections~\ref{sec: random-app} and \ref{sec: uniform-estimation}, we have extensively applied this approach in two concrete scenarios:
\begin{itemize}
\item \textbf{From training to testing.} By setting $\cQ=L^2(\hrho_n)$ and $\cM=L^2(\rho)$, we can transfer a training error estimate to a test error estimate. 
\item \textbf{From $L^2$ to $L^\infty$.} By taking $\cQ=L^2$ and $\cM=L^\infty$, we can obtain an $L^\infty$-estimation error bound from a known $L^2$-estimation error bound.
\end{itemize}

\section{The $\cF_{p,\pi}$ and Barron spaces}
\label{sec: spaces}

In this section, we define various function spaces associated with the feature-based model~\eqref{eqn: rfm}, which will be utilized in our subsequent duality analysis. Recall that
$\phi:\cX \times \cV \to \RR$ is the feature function, where  $\cX$ and $\cV$ represent the input and weight domains, respectively.

We first define the function spaces for RFMs.
\begin{definition} \label{def:F_p}
Given a  $\pi\in \cP(\cV)$ and $1 \leq p \leq {\infty}$, define
the $\cF_{p,\pi}$ space on $\cX$  as
\begin{align*}
   \cF_{p,\pi}  :=     \left\{f = \int_\cV a(v)\phi(\cdot,v) \d \pi(v):\,a\in L^p(\pi) \right\},
\end{align*}
equipped with the norm
\begin{align}\label{eqn: def1}
        \|f\|_{\cF_{p,\pi}} = \inf_{a \in A_f} \|a\|_{p,\pi},\quad \text{where}\,A_f =\left\{a \in L^p(\pi):\,f =\int_\cV a(v)\phi(\cdot,v)\d \pi(v) \right\}.
\end{align}
\end{definition}
Note that for a given $f\in \cF_{p,\pi}$,  its representation $a(\cdot)$ may not be unique. Taking the infimum over all possible representations in \eqref{eqn: def1} ensures the norm  to be measured using the optimal representation. This  makes the $\cF_{p,\pi}$  norm well-defined and $\cF_{p,\pi}$ is isometrically isomorphic to the  quotient space $L^p(\pi)/A_0$, where $A_0 = \{a \in L^p(\pi): \int_{\cV} a(v)\phi(\cdot,v)\dd\pi(v) = 0\}$. Moreover, it trivially follows from  H\"older's inequality  that
$$
\cF_{\infty, \pi}\subset \cF_{p,\pi}\subset \cF_{q,\pi}\subset \cF_{1,\pi}\quad \text{ for } 1\leq q\leq p\leq \infty.
$$

We highlight that  $\cF_{p,\pi}$ is a natural function space for studying the approximation power of   RFMs.  Specifically, by setting $a_j=a(v_j)$ and $v_j\stackrel{iid}{\sim}\pi$ for $j\in [m]$ in Eq.~\eqref{eqn: rfm},  the law of large numbers (LLN) implies that for each $x\in\cX$, as  $m\to{\infty}$, 
$$
    \frac{1}{m}\sumjm a_j \phi(x,v_j) \to \int_{\cV} a(v)\phi(x,v)\dd \pi(v),
$$
indicating that the RFM converges to a $\cF_{p,\pi}$ function. Moreover,
the Marcinkiewicz-Zygmund type   LLN~\cite[Theorem 2.5.8]{durrett2019probability} suggests that
this convergence occurs at a rate of $O(m^{-1/p'})$ provided $\|a\|_{p,\pi}<\infty$. This insight allows us to derive the rate of approximating $\cF_{p,\pi}$ using RFMs, as the function's representation satisfies $\|a\|_{p,\pi}<\infty$. We refer to   Section \ref{sec: random-app} for details.

In addition, we draw attention to the special case where $p=2$. As demonstrated in 
\cites{rahimi2008uniform},  $\cF_{2,\pi}$ is an RKHS: $\cF_{2,\pi}=\cH_{k_\pi}$ with the kernel $k_\pi$ given by
\begin{equation}\label{eqn: x1}
k_{\pi}(x,x') := \int_\cV \phi(x,v) \phi(x',v)\d \pi(v).
\end{equation}
Furthermore, we establish below that for any Mercer kernel, the associated RKHS is also a $\cF_{2,\pi}$. Thus, studying  $\cF_{2,\pi}$ is relevant for understanding general kernel methods.

\begin{lemma}\label{lemma: F2-rkhs}
Let $k:\cX\times\cX\mapsto\RR$ be a kernel satisfying $\int_{\cX} k(x,x)\dd\gamma(x)<\infty$ for some full-support distribution $\gamma\in\cP(\cX)$. Then,
there exists a symmetric feature function $\phi: \cX \times \cX \to \RR$, i.e., $\phi(x,v)=\phi(v,x)$, such that $k(\cdot,\cdot)$ takes the form  \eqref{eqn: x1} with $\pi = \gamma$. 
%Moreover, when $k$ is a dot-product kernel, there exists an activation function $\sigma: \RR \to \RR$ such that $\phi(x,v) = \sigma(v^\top x)$.
\end{lemma}
\begin{proof}
Let 
$
k(x,x') = \sum_{i=1}^{\infty} \mu_ie_i(x)e_i(x')
$
be the spectral decomposition with respect to $\gamma$.
Choosing $\phi(x,v) = \sum_{i=1}^{\infty} \sqrt{\mu_i}e_i(x)e_i(v)$, we have
\vspace*{-.5em}
\begin{align*}
\int_\cX \phi(x,v)\phi(x',v) \d \gamma(v) & = \sum_{i=1}^{\infty} \sum_{j=1}^{\infty} \sqrt{\mu_i\mu_j}\int_{\cX}e_i(x)e_i(v)e_j(x')e_j(v) \d \gamma(v) \\
& = \sum_{j=1}^{\infty} \mu_je_j(x)e_j(x') = k(x,x').
\end{align*}
\end{proof}

In the proofs of Section \ref{sec: uniform-estimation}, we will frequently use the above construction of feature functions to study the $L^\infty$ estimation of RKHS.

We now turn to define function spaces for two-layer neural networks.
\begin{definition}[Barron space] \label{def:F_1}
We define the Barron space $\cB$ on $\cX$ as
\begin{align*}
   \cB  :=     \left\{f = \int_\cV \phi(\cdot,v) \d \mu(v):\,\mu \in \cM(\cV) \right\},
\end{align*}
equipped with the norm
\begin{align}\label{eqn: def2}
        \|f\|_{\cB} = \inf_{\mu \in M_f} \|\mu\|_{\mathrm{TV}},\quad \text{where}\,M_f =\left\{\mu \in \cM(\cV):\,f =\int_\cV \phi(\cdot,v)\d \mu(v) \right\}.
\end{align}
\end{definition}
The definition above, originally proposed in \cites{bach2017breaking}, is equivalent to the variation-based definition used in \cites{barron1992neural,devore1998nonlinear,kurkova2001bounds,siegel2021characterization}. When $\phi(x,v)=\max(v^\top x,0)$, it is also equivalent to the moment-based definition proposed in  \cites{ma2019priori,weinan2021barron}.
All of these definitions are somehow equivalent \cites{wojtowytsch2022representation,siegel2021characterization}. In our duality analysis, we specifically adopt the above definition as our analysis will use the fact that $\cM(\cV)$ is the dual space of $C_0(\cV)$. However, to better understand the relationship between the $\cB$ and $\cF_{p,\pi}$ spaces, it is preferred to take the moment-based definition. Specifically, by extending \cite[Proposition 3]{weinan2021barron}, we have the following lemma, whose proof can be found in Appendix \ref{sec: proof-union-barron}.

\begin{lemma}[An alternative definition of Barron spaces]\label{lemma: barron-space-alternative}
For any $1 \leq p \leq {\infty}$, we have 
\begin{equation}\label{eqn: barron-moment}
    \cB = \cup_{\pi\in\cP(\cV)} \cF_{p,\pi},\quad \|f\|_{\cB} = \inf_{\pi \in \cP(\cV)}\|f\|_{\cF_{p,\pi}}.
\end{equation}
\end{lemma}
It is shown that $\cB$ is the union of all $\cF_{p,\pi}$ spaces of a fixed $p\in [1,{\infty}]$. Surprisingly, the union is independent of the value of $p$. One can also treat \eqref{eqn: barron-moment} as an alternative definition of Barron spaces, which extends the definition in \cites{ma2019priori,weinan2021barron} to general feature functions.

By choosing $p=2$ and noting $\cF_{2,\pi}=\cH_{k_\pi}$, we have $\|f\|_{\cB}=\inf_{\pi\in \cP(\cV)} \|f\|_{\cH_{k_\pi}}$. This  implies that two-layer neural networks can be viewed as \emph{adaptive kernel methods} \cites{ma2019priori}.
Additionally, by setting $p=1$, we find that $\cB=\cup_{\pi\in \cP(\cV)} \cF_{1,\pi}$, suggesting that $\cB$ is substantially larger than the $L^1$-type space $\cF_{1,\pi}$. In particular, if $\pi$ admits a density on $\cV$, then $\cF_{1,\pi}$ does not encompass functions implemented by finite-neuron neural networks. These observations collectively underscore that the Barron space $\cB$ should not be trivially interpreted as a $L^1$-type space and the feature adaptivity plays a more critical role.

\begin{remark}
We are aware of that Definition \ref{def:F_p} and \ref{def:F_1} can be unified using the RKBS framework \cites{zhang2009reproducing,bartolucci2023understanding}. However, adopting such an approach would make the definitions, and particularly the statements of our duality results in Section \ref{sec: dual-equivalence}, overly abstract and difficult to comprehend. Therefore, we will adhere to the  definitions above and direct interested readers to the concurrent work \cites{spek2022duality}, where the RKBS structures of these spaces are discussed.
\end{remark}

\section{The dual equivalences for $\cF_{p,\pi}$ and Barron spaces}
\label{sec: dual-equivalence}

In this section, we apply the general duality~\eqref{eqn:dual_abstract} to $\cF_{p,\pi}$ and Barron spaces. This establishes a duality framework that links approximation and estimation for the analysis of RFMs and two-layer neural networks. 
To state our result, we need to define the conjugate spaces:

 \begin{definition}[Conjugate space]  \label{def:conj}
 Let $\phi:\cX\times\cV\mapsto\RR$ be the feature function.
For any $\rho \in \cP(\cX)$, let $\tilde{\cF}_{q,\rho}$ be a function space over the weight domain $\cV$, defined as
\begin{align*}
    \widetilde{\cF}_{q,\rho} = \left\{g = \int_\cX b(x)\phi(x,\cdot)\d \rho(x): \,b\in L^q(\rho) \right\}
\end{align*}
equipped with the norm 
\begin{align*}
        \|g\|_{\widetilde{\cF}_{q,\rho}} = \inf_{b \in L^q(\rho)} \|b\|_{q,\rho},\quad \text{where}\,B_f =\left\{b \in L^q(\rho):\,g =\int_\cX b(x)\phi(x,\cdot)\d \rho(x) \right\}.
\end{align*}
Analogously, we let
$$
     \widetilde{\cB} = \left\{g = \int_\cX \phi(x,\cdot)\d \mu(x): \mu \in \cM(\cX) \right\}
$$
and define the $\widetilde{\cB}$ norm in the same way as \eqref{eqn: def2}.
 \end{definition}
It is important to clarify  that the conjugate spaces $\tilde{\cF}_{q,\rho}$ and $\tilde{\cB}$ are defined over the weight domain $\cV$, whereas $\cF_{p,\pi}$ and $\cB$ are defined over the input domain $\cX$. Additionally, these conjugate spaces should not be confused with  dual spaces.

\subsection{The interpolation regime}
For clarity, we begin by presenting the duality for the interpolation regime. Before stating the theorem, recall that for a $p\in [1,\infty]$, we denote by $p'$ its H\"older conjugate, satisfying $1/p+1/p'=1$.
\begin{theorem}\label{thm: dual-simple}
Let $x_1,\cdots,x_n \in \cX$, $\|f\|_{\hL^\infty_n} = \max_{1\le i \le n}|f(x_i)|$, and $p, q \in [1,\infty]$. 
\begin{enumerate}
    \item Suppose $\mathop{\sup}\limits_{\|f\|_{\cF_{p,\pi}} \le 1}[\|f\|_{\hL^\infty_n} + \|f\|_{q,\rho}] < \infty $. Then, we have
\begin{align}\label{dual-eq-4}
    \sup_{\|f\|_{\cF_{p,\pi}} \leq 1,f(x_i) = 0} \|f\|_{q,\rho} = \sup_{\|g\|_{\tilde{\cF}_{q',\rho}} \leq 1 } \inf_{c_1,\cdots,c_n} \left\|g - \sum_{i=1}^n c_i\phi(x_i,\cdot) \right\|_{p',\pi}.
\end{align}    
\item Suppose $\mathop{\sup}\limits_{\|f\|_{\cB} \le 1}[\|f\|_{\hL^\infty_n} + \|f\|_{q,\rho}] < \infty $, Then, we have
\begin{align}\label{dual-eq-5}
    \sup_{\|f\|_{\cB} \leq 1,f(x_i) = 0} \|f\|_{q,\rho} = \sup_{\|g\|_{\tilde{\cF}_{q',\rho}} \leq 1 } \inf_{c_1,\cdots,c_n} \left\|g - \sum_{i=1}^n c_i\phi(x_i,\cdot) \right\|_{C_0(\cV)} 
\end{align} 
\item Suppose $\mathop{\sup}\limits_{\|f\|_{\cF_{p,\pi}} \le 1}[\|f\|_{\hL^\infty_n} + \|f\|_{C_0(\cX)}] < \infty$. Then, we have
\begin{align}\label{dual-eq-6}
      \sup_{\|f\|_{\cF_{p,\pi}}\leq 1,f(x_i) = 0 } \|f\|_{C_0(\cX)} = \sup_{\|g\|_{\widetilde{\cB}} \leq 1 } \inf_{c_1,\cdots,c_n} \left\|g - \sum_{i=1}^n c_i\phi(x_i,\cdot)  \right\|_{p',\pi}.
\end{align}   
\end{enumerate}
\end{theorem}
For each equality stated above, the left-hand side is the $I$-complexity that governs the estimation errors for learning the corresponding function spaces, as detailed in Section \ref{sec: information-based complexity}. The right-hand side is exactly the worse-case error of approximating the corresponding conjugate space using the features $\{\phi(x_i,\cdot)\}_{i=1}^n$. Importantly, these equalities hold for arbitrary $x_1,\dots,x_n\in \cX$, regardless of whether they are sampled from an input distribution or not.
Thus, Theorem \ref{thm: dual-simple} establishes a  duality between estimation and approximation for learning the $\cF_{p,\pi}$ and $\cB$ spaces. 
To better understand these duality, we examine some concrete cases below.

\textbf{Case 1.} Consider the case where $q = 2$. For $p=2$,  we have
 \begin{align}\label{eqn: 12}
      \sup_{\|f\|_{\cF_{2,\pi}} \leq 1,\,f(x_i) = 0}\|f\|_{2,\rho} = \sup_{\|g\|_{\tilde{\cF}_{2,\rho}} \leq 1}\inf_{c_1,\cdots,c_n} \left\|g - \sum_{i=1}^n c_i \phi(x_i,\cdot) \right\|_{2,\pi},
 \end{align}
implying that the $L^2$ estimation of a RKHS is equivalent to the $L^2$ approximation of a associated conjugate RKHS. For the Barron space, we have 
 \begin{equation}\label{eqn: 13}
\sup_{\|f\|_{\cB} \leq 1,\,f(x_i) = 0}\|f\|_{2,\rho} = \sup_{\|g\|_{\tilde{\cF}_{2,\rho}} \leq 1}\inf_{c_1,\cdots,c_n} \left\|g - \sum_{i=1}^n c_i \phi(x_i,\cdot) \right\|_{C_0(\cX)},
 \end{equation}
 implying that the $L^2$ estimation of a  Barron space is equivalent to the $L^\infty$ approximation of the associated conjugate RKHS. 
Comparing \eqref{eqn: 12} and \eqref{eqn: 13} suggests the estimations of RFMs and two-layer neural networks are equivalent to the approximation of the same RKHS, but under different metrics  ($L^2$ vs.~$L^\infty$).  This can also be seen as a quantitative characterization of the size difference between a Barron space and the corresponding RKHS.

\textbf{Case 2.}
When $p = 2, q = {\infty}$,  we have
     \begin{align}
         \sup_{\|f\|_{\cF_{2,\pi}} \leq 1,\,f(x_i) = 0}\|f\|_{C_0(\cX)} = \sup_{\|g\|_{\tilde{\cB}} \leq 1}\inf_{c_1,\cdots,c_n} \left\|g -  \sum_{i=1}^n c_i\phi(x_i,\cdot) \right\|_{2,\pi} 
     \end{align}
The left-hand side corresponds to the $I$-complexity of RKHS with respect to the $L^\infty$ norm. Hence, this equality implies that  the $L^\infty$ estimation error of learning  RKHS can be derived through examining the $L^2$ approximation of the corresponding Barron space. Using this approach, we  conduct a comprehensive analysis of  the $L^\infty$ estimation of RKHS  in Section \ref{sec: uniform-estimation}.

\textbf{Case 3.}
For $1 < p \leq 2,\,q = 2$, we have
    \begin{align}
    \sup_{\|f\|_{\cF_{p,\pi}} \leq 1 }\inf_{c_1,\cdots,c_m} \left\|f - \sum_{j=1}^m c_j\phi(\cdot,v_j) \right\|_{2,\rho}  = \sup_{\|g\|_{\tilde{\cF}_{2,\rho}} \leq 1 ,\,g(v_i)=0} \|g\|_{p',\pi},
    \end{align}
    Suppose that $v_1,\dots,v_m$ are \iid samples drawn from $\pi\in\cP(\cV)$. Then, the left-hand side represents the worst-case error of approximating $\cF_{p,\pi}$ using RFMs. Note that the right-hand side can be bounded using classical empirical process technique, such as Rademacher complexity. Therefore, this duality allows us to study the random feature approximation for target functions in $\cF_{p,\pi}$ with $1<p< 2$, which is strictly larger than the RKHS $\cF_{2,\pi}$. In Section \ref{sec: random-app}, we delve into this observation in detail.

\subsection{The non-interpolation regime}

\begin{theorem}\label{thm: dual}
Let $\nu \in \cP(\cX)$ be a probability distribution and $1 \le p, q, r \le \infty$. If $r = \infty$, we further assume that $\nu$ is a discrete distribution, i.e., the support of $\nu$ is a finite set. 
\begin{enumerate}
    \item  Suppose  $\mathop{\sup}\limits_{\|f\|_{\cF_{p,\pi}} \leq 1}[\|f\|_{r,\nu} + \|f\|_{q,\rho}] < {\infty}$. Then, we have
\begin{align}\label{dual-eq-1}
    \sup_{\|f\|_{\cF_{p,\pi}} \leq 1,\|f\|_{r,\nu} \leq \epsilon} \|f\|_{q,\rho} = \sup_{\|g\|_{\tilde{\cF}_{q',\rho}} \leq 1 }\inf_{h \in \tilde{\cF}_{r',\nu}} \left[ \left\|g - h \right\|_{p',\pi} + \epsilon \|h\|_{\tilde{\cF}_{r',\nu}}\right].
\end{align}    
\item Suppose  $\mathop{\sup}\limits_{\|f\|_{\cB} \leq 1}[\|f\|_{r,\nu} + \|f\|_{q,\rho}]< {\infty}$. Then, we have
\begin{align}\label{dual-eq-2}
    \sup_{\|f\|_{\cB} \leq 1,\|f\|_{r,\nu} \leq \epsilon} \|f\|_{q,\rho} = \sup_{\|g\|_{\tilde{\cF}_{q',\rho}} \leq 1 } \inf_{h \in \tilde{\cF}_{r',\nu}} \left[\left\|g - h \right\|_{C_0(\cV)} + \epsilon \|h\|_{\tilde{\cF}_{r',\nu}}\right].
\end{align} 
\item Suppose   $\mathop{\sup}\limits_{\|f\|_{\cF_{p,\pi}} \leq 1}[\|f\|_{r,\nu}  +\|f\|_{C_0(\cX)}] < {\infty}$. Then, we have
\begin{align}\label{dual-eq-3}
      \sup_{\|f\|_{\cF_{p,\pi}}\leq 1,\|f\|_{r,\nu} \leq \epsilon} \|f\|_{C_0(\cX)} = \sup_{\|g\|_{\tilde{\cB}} \leq 1 } \inf_{h \in \tilde{\cF}_{r',\nu}} \left[\left\|g - h \right\|_{p',\pi} + \epsilon \|h\|_{\tilde{\cF}_{r',\nu}}\right].
\end{align}   
% \item Suppose that  $\mathop{\sup}\limits_{\|f\|_{\cF_{p,\pi}} \leq 1}[\|f\|_{C_0(\cX)}  +\|f\|_{q,\rho}] < {\infty}$, we have
% \begin{align}\label{dual-eq-7}
%       \sup_{\|f\|_{\cF_{p,\pi}}\leq 1,\|f\|_{C_0(\cX)} \leq \epsilon} \|f\|_{q,\rho} = \sup_{\|g\|_{\tilde{\cF}_{q',\rho}} \leq 1 } \inf_{h \in \tilde{\cB}} \left[\left\|g - h \right\|_{p',\pi} + \epsilon \|h\|_{\tilde{\cB}}\right].
% \end{align}   
\end{enumerate}
Recall that  $p'$ denotes the H\"older conjugate $p$, satisfying $1/p+1/p'=1$ and $q',r'$ are given analogously. 
\end{theorem}

These results are derived from applying the general duality established in Theorem~\ref{thm:dual_abstract} and Lemma~\ref{lem:dual_core}  to these specific scenarios. For illustrative purposes, we provide  here a proof of \eqref{dual-eq-1} for $q\in [1,\infty)$. We exclude the endpoint $q=\infty$, which requires additional consideration, as $L^q(\rho)^*=L^{q'}(\rho)$ does not hold when $q=\infty$. The complete proof of \eqref{dual-eq-1}, along with those for \eqref{dual-eq-2} and  \eqref{dual-eq-3} which are similar, is deferred to Appendix~\ref{sec: proof-dual-equivalence-r}.

\begin{proof}[Proof of \eqref{dual-eq-1} for  $q\in [1,\infty)$.]
Applying Theorem~\ref{thm:dual_abstract}, we obtain
\begin{equation*}\label{eqn:dual-1}
    \begin{aligned}
    \sup_{\|f\|_{\cF_{p,\pi}} \leq 1,\|f\|_{r,\nu} \leq \epsilon}\|f\|_{q,\rho}&=\sup_{\|\tb\|_{L^q(\rho)^*}\leq 1}\inf_{\tc \in L^r(\nu)^*}\left[\|\tb - \tc\|_{\cF^*_{p,\pi}} + \epsilon\|\tc\|_{L^r(\nu)^*}\right]\\ 
 &   = \sup_{\|\tb\|_{L^q(\rho)^*}\leq 1}\inf_{\tc \in L^r(\nu)^*}\left[\sup_{f\in\cF_{p,\pi}}[\tb(f) - \tc(f)] + \epsilon\|\tc\|_{L^r(\nu)^*}\right]
\end{aligned}
\end{equation*}
By the Riesz representation theorem, for any $\tb \in L^q(\rho)^*$, there exists a unique $b\in L^{q'}(\rho)$ such that $\|\tb\|_{L^q(\rho)^*}=\|b\|_{L^{q'}(\rho)}$ and $\tb(f)=\int_{\cX} b(x)f(x)\dd \rho(x)$ for all $f\in L^q(\rho)$. An analogous representation also holds for $L^r(\nu)^*$.  Applying these representations to $\tilde{b}(f)$ and $\tilde{c}(f)$ gives
\begin{align}\label{eqn:dual-2}
\sup_{\|f\|_{\cF_{p,\pi}} \leq 1,\|f\|_{r,\nu} \leq \epsilon}\|f\|_{q,\rho} = \sup_{\|b\|_{L^{q'}(\rho)}\leq 1}\inf_{c \in L^{r'}(\nu)}\left[I(b,c;f) + \epsilon\|c\|_{L^{r'}(\nu)}\right],
\end{align}
where 
\begin{align}\label{eqn:dual-3}
    I(b,c;f) =   &\sup_{\|f\|_{\cF_{p,\pi}} \leq 1} \Big(\int_{\cX}b(x) f(x)\dd \rho(x) - \int_{\cX}c(x)f(x)\dd \nu(x)\Big)\notag\\
        =& \sup_{\|a\|_{p,\pi \le 1}}\int_{\cX}\left(\int_{\cV}a(v)\phi(x,v)\dd\pi(v)\right)[b(x)\dd\rho(x) - c(x)\dd\nu(x)] \notag\\
        =& \sup_{\|a\|_{p,\pi \le 1}}\int_{\cV}a(v)\left(\int_{\cX}b(x)\phi(x,v)\dd\rho(x) - \int_{\cX}c(x)\phi(x,v)\dd\nu(x)\right) \dd\pi(v) \notag\\
        =&\left\|\int_{\cX}b(x)\phi(x,\cdot)\dd\rho(x) - \int_{\cX}c(x)\phi(x,\cdot)\dd\nu(x)\right\|_{p',\pi}.
\end{align}
Substituting \eqref{eqn:dual-3} into \eqref{eqn:dual-2}, we obtain
\begin{align*}
    &\sup_{\|f\|_{\cF_{p,\pi}} \leq 1,\|f\|_{r,\nu} \leq \epsilon} \|f\|_{q,\rho} \\
    =& \sup_{\|b\|_{q',\rho} \le 1}\inf_{ c \in L^{r'}(\nu)}\left[\Big\|\int_{\cX}b(x)\phi(x,\cdot)\dd\rho(x) - \int_{\cX}c(x)\phi(x,\cdot)\dd\nu(x)\Big\|_{p',\pi} + \epsilon\|c\|_{r',\nu}\right]\\
    =&\sup_{\|g\|_{\tilde{\cF}_{q',\rho}} \leq 1 }\inf_{h \in \tilde{\cF}_{r',\nu}} \left[ \left\|g - h \right\|_{p',\pi} + \epsilon \|h\|_{\tilde{\cF}_{r',\nu}}\right],
\end{align*}
where the last step uses the definition of  $\tilde{\cF}_{q',\rho}$ and $\tilde{\cF}_{r',\nu}$ (see Definition~\ref{def:conj}). 
% Thus, we complete the proof.
\end{proof}

Note that taking $\nu = \hat{\rho}_n:=\frac{1}{n}\sum_{i=1}^n \delta_{x_i}$, $r = \infty$ and $\epsilon=0$ recovers Theorem \ref{thm: dual-simple}. 
Comparing with Theorem \ref{thm: dual-simple}, we introduce another duality pairing: the $L^r(\nu)$ constraint on fitting errors ($\|f\|_{r,\nu}\leq \epsilon$)  vs. the $\tilde{\cF}_{r',\nu}$ penalization on the approximator.  Specifically, the right-hand sides in \eqref{dual-eq-1}, \eqref{dual-eq-2}, and \eqref{dual-eq-3} can be interpreted as the errors of approximating various function spaces using RFMs with the norm of coefficients penalized. This becomes more intuitive by taking $\nu=\hat{\rho}_n$, in which case the right-hand side of \eqref{dual-eq-1} becomes
     \[
        \sup_{g\in \tilde{\cF}_{q',\rho}}\inf_{c_1,\dots,c_n} \left(\left\|g-\frac{1}{n}\sumin c_i \phi(x_i,\cdot)\right\|_{p',\pi} + \epsilon \left(\fn\sumin |c_j|^r\right)^{1/r}\right).
     \]
% Notably, this allows us to study regularized estimators, for which the coefficient norms can be well controlled. For more details, we refer to Section \ref{sec: random-app}.
It is worth highlighting that in Theorem~\ref{thm: dual}, the choice of $\nu\in\cP(\cX)$ is flexible. This  allows us to tailor the dual equivalence to different scenarios, providing a  unifying framework to examine the trade-off between approximation and estimation. 
For instance,  by taking $\nu=\bar{\rho}$, the left-hand side becomes the $I$-complexity that serves a lower bound for the  minimax estimation error in the noisy setting, as shown in Proposition~\ref{lemma-random-data}.

 \section{Random feature learning beyond kernel regime} 
 \label{sec: random-app}

In this section, we utilize our duality framework to analyze the learning of $\cF_{p,\pi}$ using RFMs. Previous analyses of RFMs have primarily focused on target functions within  RKHS, corresponding to the special case where $p=2$. In contrast,
our duality framework enables us to explore the entire range of $p\in (1,2]$.

The core idea of our proof is to apply duality to transform the complex  approximation problem to an equivalent estimation problem, which is  much more tractable. In particular, the $L^q$ RFM approximation error of the $\cF_{p}$ space is equivalent to the $L^{p'}$ estimation error of the conjugate space $\tilde{\cF}_{q'}$. By leveraging the (local) Rademacher complexity-based bound to control the estimation error of $\tilde{\cF}_{q'}$, we obtain the RFM approximation error bound for $\cF_p$.

Let $\cG$ be a function class over $\cX$. We define the (worse-case) Rademacher complexity of  $\cG$ by 
\begin{equation}
\rad(\cG) = \max_{x_1,\dots,x_n\in\cX}\EE_{\xi\sim\mathrm{Unif}(\{\pm 1\}^n)}\left[\Big|\sup_{g\in\cG}\fn\sumin g(x_i)\xi_i\Big|\right].
\end{equation}

\begin{assumption}[Conditions on the feature function]\label{as1}
For any  $q\in [2,\infty]$, we assume:
\begin{itemize}
\item[(i)] There exists  $M_q>0$ such that 
$
   \sup_{v\in\cV}\|\phi(\cdot,v)\|_{q,\rho} \leq M_q.
$
\item[(ii)]  There exists $R_q>0$ such that 
$
       \rad(\tilde{\cF}_{q',\rho}(1))  \leq R_q/\sqrt{n}.
$
\end{itemize}
\end{assumption}
Note that $\tilde{\cF}_{q',\rho}$ represents the conjugate space associated with  $\phi$ and consequently, Assumption \ref{as1}(ii) essentially imposes a condition on $\phi$. The following lemma confirms that this condition is indeed satisfied by  commonly-used feature functions. The proof can be found in Appendix \ref{sec: proof-lemma-rad}.
% Note that $O(n^{-1/2})$ is the natural scaling of the Rademacher complexity for most function classes of interest~\cites{shalev2014understanding}.

\begin{lemma} \label{rad-lemma}
\begin{itemize}
\item For $q \in [2,\infty)$, suppose
    $
         \sup_{v\in\cV}\|\phi(\cdot,v) \|_{\rho} \leq R.
    $
    Then Assumption \ref{as1}\,$\mathrm{(ii)}$ holds with $R_q \leq \sqrt{q}R$.
\item For $q = {\infty}$, if $\cX$ and $\cV $ are both supported on $\{x: \|x\|_2 \leq R \} $ and $\phi(x,v) = \varsigma(x^\top v)$, where $\varsigma: \RR \to \RR$ is $L$-Lipschitz and $\varsigma(0) = 0$. Then Assumption \ref{as1}\,$\mathrm{(ii)}$  holds with $R_\infty \leq LR^2 $.
\end{itemize}
\end{lemma}

\subsection{Approximation of  $\cF_{p,\pi}$ with random features}

\begin{theorem} \label{thm: approximation-bound}
Suppose $1 \leq p \leq 2$ and $v_j\stackrel{iid}{\sim}\pi$ for $j\in [m]$. If Assumption \ref{as1}  holds, then \wp at least $1-\delta$ over the sampling of $\{v_j\}_{j=1}^m$, there exists an absolute constant $C_1>0$, such that for any $C \geq C_1$, we have for $q\in [2,\infty]$ that
\begin{equation}  \label{app-error}
\begin{aligned}
 & \sup_{\|f\|_{\cF_{p,\pi}} \leq 1}   \inf_{\|c\|_p \leq Cm^{1/p}} \left\|f - \frac{1}{m}\sum_{j=1}^m c_j\phi(\cdot,v_j) \right\|_{q,\rho}  \lesssim \left(\frac{M_q^{p'-2}R_q^2\log^3 m + M_q\log(1/\delta)}{m}\right)^{1/p'}.
 \end{aligned}
\end{equation}
\end{theorem}
% \begin{remark}
The proof is deferred to Appendix \ref{sec: proof-approx-rfm}.
This theorem establishes the uniform approximability of $\cF_{p,\pi}$ with random features. The uniformity implies that,  upon sampling the random features $\phi(\cdot,v_1),\dots, \phi(\cdot,v_m)$, any function in $\cF_{p,\pi}(1)$ can be approximated efficiently using these features.  Note that the approximation rate scales as $O(m^{-(p-1)/p})$, which  deteriorates as $p$ decreases, aligning with the fact that $\cF_{p,\pi}$ becomes larger as $p$ approaches $1$. Moreover, in Theorem~\ref{thm:Fp-approximation-lower-bound}, we provide  lower bounds indicating this approximation rate is sharp, unless additional conditions beyond Assumption~\ref{as1} are imposed on the feature function.

% Additionally, the approximation rates surprisingly are independent of the value of $q$. For example, the $L^2$  and  $L^\infty$ approximations share the same rate. But the constants, which depend on $M_q$ and $R_q$, can differ significantly. Consider the random ReLU feature $\phi(x,v)=\max(x^\top v,0)$ with $x\in \sqrt{d}\SS^{d-1}$ and $v\in \SS^{d-1}$. It can be  easily  verified that $M_2\asymp 1$, whereas $M_\infty\asymp \sqrt{d}$. Consequently, for the corresponding $\cF_{2,\pi}$, the $L^2$ and $L^\infty$ approximation errors are $O(1/\sqrt{m})$ and $O(\sqrt{d/m})$, respectively. This indicates that the separation between $L^2$ and $L^\infty$ approximations is pronounced only in high dimensions. 

The special case where $p = q = 2$, corresponding to the $L^2$ approximation of RKHS with random features, has been examined in \cites{bach2017equivalence}. Their analysis relies heavily  on the Hilbert structure of $\cF_{2,\pi}$ and the associated spectral decomposition. In contrast,
our result holds for all $p\in (1,2]$ and $q\in [2,\infty]$, with the proof substantially simpler. Additionally, the case where $p=2, q=\infty$ corresponds to the $L^\infty$ approximation of RKHS with random features. \cites{rahimi2008uniform} addressed a similar problem, but their result is limited to  the $L^\infty$ approximation of $\cF_{\infty,\pi}$, which is merely a subset of $\cF_{2,\pi}$.

\textbf{A refined analysis.}
We here construct special feature functions to demonstrate that the upper bound in Theorem~\ref{thm: approximation-bound} is sharp. Specifically, 
consider a probability distribution $\rho \in \cP(\cX)$ and let $\{u_j\}_{j \in \ZZ}$ be an arbitrary orthonormal basis in $L^2(\rho)$. Define the feature function $\phi:\cX\times \TT \mapsto \CC$~\footnote{This complex-valued feature can be straightforwardly transformed to an equivalent real-valued feature; see Appendix~\ref{sec: proof-lower-bound-periodic-RFM}.  We here adopt this complex-valued version for notation simplicity.} 
as follows
\begin{equation}\label{eqn:svd}
    \phi(x,v) = \sum_{j \in \ZZ}\mu_j^{\half}  \overline{u_j(x)} e_j(v),
\end{equation}
where $\mu_j\asymp (|j|+1)^{-\beta-1}$ and $e_j(v) := \exp(2\pi \ii j v)$ is the Fourier basis, as in \eqref{eq:periodic}.  The associated conjugate kernel $\tk_{\rho}(v,v'):=\int \phi(x,v)\overline{\phi(x,v')}\dd\rho(x) = \sum_{j \in \ZZ}\mu_j e_j(v) \overline{e_j(v')}$ is periodic on $\TT$ and $\mu_j$ is the $j$-th eigenvalue of $\tk_\rho$ \wrt $\pi_0:=\mathrm{Unif}(\TT)$. For the feature function~\eqref{eqn:svd}, the specific periodic property  allows us to leverage tools from Fourier analysis to obtain following refined results.

\begin{theorem}\label{thm:Fp-approximation-lower-bound}
	Consider the feature function~\eqref{eqn:svd}.  
    For any $\varphi_1,\dots,\varphi_m \in L^2(\rho)$, $p \in [1,2]$, and $q\in [2,\infty]$, it holds  that
    \begin{equation}\label{Fp-approximation-lower-bound}
        \sup_{\|f\|_{\cF_{p,\pi_0}} \le 1}\inf_{c_1,\dots,c_m}\Big\|f - \sum_{j=1}^mc_j\varphi_j\Big\|_{q,\rho} \gtrsim m^{-\frac{1}{p'}-\frac{\beta}{2}}.
    \end{equation}
    Let $v_1,\dots,v_m$ be \iid samples drawn from $\pi_0$. Then,  with probability $1-\delta$, we have
    \begin{equation}\label{Fp-approximation-upper-bound}
         \sup_{\|f\|_{\cF_{p,\pi_0}} \le 1}\inf_{c_1,\dots,c_m}\Big\|f - \sumjm c_j\phi(\cdot,v_j)\Big\|_{\rho} \lesssim m^{-\frac{1}{p'}-\frac{\beta}{2}}(\log(m/\delta))^{-\frac{\beta+1}{2}}.
    \end{equation}
\end{theorem}

The proof can be found in Appendix~\ref{sec: proof-lower-bound-periodic-RFM}.
Taking $q=2$, this theorem establishes  improved matching bounds for the $L^2$ approximation of the $\cF_{p,\pi_0}$ space associated with the feature function \eqref{eqn:svd}.  To compare with Theorem~\ref{thm: approximation-bound}, we must verify Assumption~\ref{as1}. For $q=2$, we have
\begin{align*}
\sup_{v \in \TT}\|\phi(\cdot,v)\|^2_\rho  &=\sup_{v \in \TT}\sum_{j \in \ZZ} \mu_j |e_j(v)|^2 = \sum_{j \in \ZZ} \mu_j:=\kappa <\infty.
\end{align*}
Thus, Lemma~\ref{rad-lemma} implies that Assumption~\ref{as1} is satisfied for $q=2$ with $M_q = \sqrt{\kappa}$ and $R_q = \sqrt{2\kappa}$. For $q\in (2,\infty]$, we can construct specific $\cX,\rho$, and $\{u_j\}_{j\in\ZZ}$ to satisfy Assumption~\ref{as1}. For brevity, we defer this construction  and the verification to Appendix~\ref{sec: special-feature-construction}. 

Since $\beta$ generally can approach $0$, the lower bound~\eqref{Fp-approximation-lower-bound} indicates that the upper bound in Theorem~\ref{thm: approximation-bound}, which is applicable to general feature functions, is sharp. This holds unless additional conditions, beyond those specified in Assumption~\ref{as1}, are imposed.  

\begin{remark}
The upper and lower bounds together imply that $\cF_{p,\pi}$ with $p \in (1,2)$ is indeed strictly larger than $\cF_{2,\pi}$, the corresponding RKHS.  This confirms that we are operating beyond the traditional kernel regime.
\end{remark}

\subsection{Excess risk bounds}
The following theorem provides an excess risk bound for regularized estimators. 
\begin{theorem}\label{Fp-estimate}
Suppose that $f^*\in \cF_{p,\pi}(1)$ {\ wh $p\in (1,2]$} and  $x_i \stackrel{iid}{\sim}\rho,\,y_i = f^*(x_i) + \xi_i$ with $\xi_i \sim \cN(0,\sigma^2)$. Let $v_1,\cdots,v_m$  be \iid random weights sampled from $\pi$. Consider the  estimator
$
    \hat{f} =\frac{1}{m} \sum_{j=1}^m \hat{c}_j \phi(\cdot,v_j)
$
with $\hc\in\RR^m$ given by 
\begin{align} \label{constrain-ERM}
    \hat{c} := \argmin_{\|c\|_p \leq R} \frac{1}{n}\sum_{i=1}^n \left(y_i - \frac{1}{m}\sum_{j=1}^m c_j \phi(x_i,v_j) \right)^2.
\end{align}

Assume $M:=\sup_{x\in\cX,v\in\cV}|\phi(x,v)| <\infty$. Then, with an appropriate choice of $R$, for any $\delta\in (0,1)$, it holds \wp at least $1-\delta$ over the sampling of data $\{(x_i,y_i)\}_{i=1}^n$ and features $\{v_j\}_{j=1}^m$ that
\begin{equation*}%\label{eqn: upperbound-rkhs}
      \|\hat{f}-f^* \|_{\rho}^2 \lesssim  M\sigma \sqrt{\frac{p'+ \log(1/\delta)}{n}}   + \frac{p'M^2\log^3 n + M\log(1/\delta)}{n} + \left(\frac{M^{p'}\log^3 m + M\log(1/\delta)}{m}\right)^{2/p'}.
\end{equation*}

\end{theorem}
This theorem provides an upper bound of the excess risk of learning functions in $\cF_{p,\pi}$ and the proof is deferred to Appendix~\ref{sec: proof-learning-err of RFMs}. Note that here we  focus on  norm-constrained estimators but similar arguments can  be straightforwardly extended to penalized estimators. 

The upper bound  is comprised of two  components: the estimation error and the approximation error. Notably, the estimation error rate is independent of the value of $p$,  exhibiting a scaling of $O(n^{-1})$ in the noiseless setting and $O(n^{-1/2})$ in the noisy setting. { This stands in contrast to the approximation error, for which the rate deteriorates as decreasing $p$ towards $1$. Specifically, in the noisy setting, to ensure the error matches the optimal one---achieved with infinitely-many features---one needs $m\geq \tilde{\Omega}(M^{p'/2}n^{p'/4})$. Consider the random ReLU feature $\phi(x,v)=\max(x^\top v,0)$ with $\cX=\sqrt{d}\SS^{d-1}$, and $\Omega=\SS^{d-1}$ and let $\rho=\mathrm{Unif}(\SS^{d-1})$. The choice of domains ensures  $\EE_{x\sim\rho}[\phi(x,v)^2]\asymp 1$. Then, $M =\sqrt{d}$ and consequently, the feature number needs to satisfy $m\geq \tilde{\Omega}((nd)^{\frac{p}{4(p-1)}})$.
}

We highlight that RFMs are linear models and the corresponding optimization problem \eqref{constrain-ERM} can be solved efficiently, i.e., the time complexity depends on the input dimension $d$ at most polynomially. Therefore, we have established that in terms of all approximation, estimation, and optimization, RFMs can efficiently learn functions in $\cF_{p,\pi}$ under the $L^2$ metric as long as $p>1$. This demonstrates the broad applicability of RFMs beyond the kernel regime where $p=2$.

\paragraph*{Learning $\cF_{p,\pi}$ with two-layer neural networks.} Lastly, we draw a comparison between RFMs and neural networks for learning $\cF_{p,\pi}$, highlighting their respective strengths and weaknesses.
\cites{ma2019priori} proved that the excess risk of learning  $\cB$ with two-layer neural networks is bounded by $O(n^{-1/2}+m^{-1})$, where the two terms correspond to the estimation and approximation errors, respectively. Since  $\cF_{p,\pi}$ is a subset of $\cB$ for all $p\in (1,2]$, thus the excess risk of learning $\cF_{p,\pi}$ functions with neural networks is also bounded by $O(n^{-1/2}+m^{-1})$. Comparing this with \eqref{constrain-ERM}, we observe that the {\em adaptivity} of neural networks reduces the approximation error from $O(m^{-(p-1)/p})$ to $O(m^{-1})$ for $p\in (1,2)$. However, the non-convex nature of neural networks raises questions about the efficiency of solving the corresponding optimization problems.  In contrast, Theorem  \ref{Fp-estimate} directly implies that $\cF_{p,\pi}$ can be efficiently learned using RFMs across all aspects---approximation, estimation, and particularly optimization.

\section{The $L^\infty$ learning of RKHS} \label{sec: uniform-estimation}

It is well-known that  RKHS can be learned efficiently under the $L^2$ metric,  but this may not be adequate for security- and safety-critical applications, where the $L^\infty$ metric is more relevant.
In this section, we demonstrate that  $L^\infty$ learning of  RKHS can be effectively characterized using our duality framework.

\subsection{Lower bounds} \label{lowerbound}

\begin{theorem}[$L^\infty$-estimation] \label{primal-bound}
Let $k:\cX\times\cX\mapsto\RR$ be a Mercer kernel and assume $\cH_k \subset C_0(\cX)$. For a distribution $\gamma \in \cP(\cX)$, recall that $\{\mu^{(k,\gamma)}_i\}_i$ denote the eigenvalues of $k(\cdot,\cdot)$ \wrt  $\gamma$.
Let $\Lambda_{k,\gamma}(m)=\sum_{i=m+1}^{\infty} \mu^{(k,\gamma)}_i$.
% , where  $\{\mu_i^{k,\gamma}\}_{i=1}^{\infty}$ are the eigenvalues of $k$ with respect to the input distribution $\gamma$. 
% For any $\gamma \in \cP(\cX)$, let $\{\lambda_i^{\gamma}\}_{i=1}^{\infty}$ be the eigenvalues of $k$ in a decreasing order on $L^2(\gamma)$. 
\begin{enumerate}
    \item  \label{cor-worst}\textnormal{(\textbf{The noiseless case})} For any input data $x_1,\cdots,x_n \in \cX$, we have
\begin{align}\label{eqn: lower-bound-1}
     \inf_{T \in \cA_n} \sup_{\|f\|_{\cH_k} \leq 1} \|T(\{(x_i,f(x_i))\}_{i=1}^n) -f  \|_{C_0(\cX)} \gtrsim  \sup_{\gamma\in \cP(\cX)}\sqrt{\Lambda_{k,\gamma}(n)}.
\end{align}
     \item\label{cor-minimax} \textnormal{(\textbf{The noisy case})} Given any $\rho_i\in\cP(\cX)$ for $i\in [n]$, suppose $x_i\sim\rho_i,  \xi_i\stackrel{iid}{\sim}\cN(0,\sigma^2)$. Let $\brho = \frac{1}{n}\sum_{i=1}^n\rho_i$ and $n_\brho(\sigma) =\inf \{m: n\mu_{m+1}^{(k,\brho)} \le \sigma^2\}$. Then, we have
     \begin{align}\label{eqn: lower-bound-2}
 \inf_{T \in \cA_n} \sup_{\|f\|_{\cH_k} \leq 1} \EE\|T(\{(x_i,f(x_i)+\xi_i)\}_{i=1}^n) -f  \|_{C_0(\cX)} \gtrsim   \sup_{\gamma \in \cP(\cX)}\sqrt{\Lambda_{k,\gamma}(n_{\brho}(\sigma))}.
     \end{align}
\end{enumerate}
\end{theorem}
The proof of this theorem utilizes the duality~\eqref{dual-eq-3} and can be found in Appendix~\ref{sec: proof-lowerbound-rkhs}. Notably, both lower bounds are expressed in a unified manner using the kernel's  eigenvalue decay. The lower bound in the noisy case depends on the sampling strategy through $n_{\brho}(\sigma)$. Particularly, when the  eigenvalues decay  according to a power law, we have
\begin{itemize}
\item If $\mu_j^{(k,\gamma)}\asymp  j^{-1-\beta}$ for some $\gamma\in \cP(\cX)$,  then $\Lambda_{k,\pi}(n)\asymp n^{-\beta}$. Thus, the  $L^\infty$ estimation error  in the noiseless case is lower bounded by $\Omega(n^{-\beta/2})$.
\item Assuming $\mu_j^{(k,\brho)}\asymp  j^{-1-\beta}$, a simple calculation gives that $n_{\brho}(\sigma) \asymp n^{-1/(\beta+1)}$ and consequently, a lower bound of the $L^\infty$ estimation in the noisy case is $\Omega( n^{-\beta/(2\beta+2)})$ .
\end{itemize}
% Comparing two cases suggests that $L^\infty$ estimation is substantially more challenging in the noisy case when $\beta$ is relatively large. 

It is important to note that  if $\beta=1/\poly(d)$,  both lower bounds exhibit the CoD. This situation commonly occurs when the kernel function is not sufficiently smooth. 
For example, consider dot-product kernels on $\mathbb{S}^{d-1}$ that take the form 
$\kappa(x^\top x')=\EE_{w\sim\tau_{d-1}}[\varsigma(w^\top x)\varsigma(w^\top x')]$, where $\tau_{d-1}=\mathrm{Unif}(\mathbb{S}^{d-1})$ and $\varsigma:\RR\mapsto\RR$ is an activation  function. These random features play crucial roles for understanding neural networks~\cites{cho2009kernel,bach2017breaking}.  When $\varsigma(\cdot)$ is non-smooth, it is often that $\beta=1/\poly(d)$~\cites{bietti2021deep,scetbon2021spectral,Wu2021ASA}. Specifically, consider the ReLU$^\alpha$ activation: $\varsigma(t)=\max(0,t)^\alpha$, where $\alpha\in \ZZ_{\geq 0}$. 
\cite[Proposition 5]{Wu2021ASA} shows that 
$
\Lambda_{k,\tau_{d-1}}(n)\geq C_{\alpha, d} n^{-\frac{2\alpha+1}{d-1}},
$
where $C_{\alpha,d}$ depends on $1/d$ polynomially. Thus, for the noiseless setting, the minimax error is lower bounded by
\begin{align*}
    \inf_{T \in \cA_n} \sup_{\|f\|_{\cH_k} \leq 1} \|T(\{(x_i,f(x_i))\}_{i=1}^n) -f  \|_{C_0(\cX)} &\gtrsim  \sup_{\gamma\in \cP(\cX)}\sqrt{\Lambda_{k,\gamma}(n)}\\ 
    &\geq  \sqrt{\Lambda_{k,\tau_{d-1}}(n)}\geq C_{\alpha,d} n^{-\frac{2\alpha+1}{2(d-1)}}.
\end{align*}
This indicates that $L^\infty$ estimation suffers from the CoD even in the absence of noise, let alone in the noisy setting.
% Applying Theorem \ref{primal-bound}, we can conclude that  the $L^\infty$ learning of the corresponding RKHS is inevitably subject to the CoD. 

Next, we present evidence demonstrating that both lower bounds are tight.

\subsection{Upper bounds} \label{upperbound}

\begin{assumption}\label{assu: kernel}
    Consider an RKHS $\cH_k \subset C_0(\cX)$ and a probability $\rho \in \cP(\cX)$ with full support. Let $\{(\mu_j,e_j)\}_{j=1}^{\infty}$ be the  eigen-pairs \wrt $\rho$. Assume there exists $\beta > 0$ and $M > 0$ such that
    \begin{equation*}
        \mu_j \asymp j^{-\beta-1},\quad  \sup_{j \in \NN^+}\|e_j\|_{C_0(\cX)} \le M.
    \end{equation*}
\end{assumption}

The above uniform boundedness assumption on the eigenfunctions, given by $\sup_{j \in \NN^+}\|e_j\|_{C_0(\cX)} \le M$, has been widely adopted in the convergence analysis of kernel methods~\cites{caponnetto2007optimal,steinwart2009optimal,Pozharska2022ANO,long2023reinforcement}. However, this assumption excludes popular dot-product kernels on the sphere~\cites{burq2014probabilistic}. In this work, we adopt this assumption mainly to avoid unnecessary technical complications, as it is not crucial to our arguments. Alternative assumptions, such as those based on maximal degree of freedom~\cites{bach2017equivalence} or embedding indices~\cites{dicker2017kernel,pillaud2018statistical,sobolevkernel}, which do include dot-product kernels, can yield similar results.

\paragraph*{Kernel ridge regression (KRR).} Consider the popular KRR estimator, namely
\begin{equation}\label{eqn: krr}
    \hat{f}_\lambda = \min_{f \in \cH_{k}}\left(\frac{1}{n}\sum_{i=1}^n(f(x_i) - y_i)^2 + \lambda \|f\|_{\cH_{k}}^2\right),
\end{equation}
where $\lambda\geq 0$ is a regularization hyperparameter. When $\lambda\to 0^{+}$, $\hf_\lambda$ becomes the minimum-norm estimator: $\hat{f}_{0} = \argmin_{f \in \cH_k, f(x_i) = y_i}\| f\|_{\cH_k}$.

\begin{theorem}\label{Lq-estimation-rkhs}
Consider the random design setting and assume $\|f^*\|_{\cH_k} \le 1$. Under Assumption \ref{assu: kernel}, we have:
    \begin{itemize}
        \item \textnormal{(\textbf{The noiseless case})} 
        With probability at least $1-\delta$ over the sampling of $\{(x_i,y_i)\}_{i=1}^n$, we have for the minimum-norm estimator that
 \begin{align} \label{eqn: Lq-learning-upper}
             \|\hat{f}_0 - f^* \|_{C_0(\cX)}\lesssim    n^{-\frac{\beta}{2}}(\log(n/\delta))^{\frac{\beta+1}{2}}.
 \end{align}
 \item \textnormal{(\textbf{The noisy case}, \cite[Remark 3.4]{sobolevkernel})} Suppose $\sigma^2=\EE[\xi_i^2]\asymp 1$ and take $\lambda\asymp n^{-\frac{\beta+1}{\beta+2}}$. Then,  \wp at least $1-\delta$ over the sampling of $\{(x_i,y_i)\}_{i=1}^n$, it holds for any $\tau\in (0,\beta)$ that
 \begin{equation}\label{eqn: uniform-upperbound-noisy}
    \|\hf_\lambda-f^*\|_{C_0(\cX)}\lesssim  \log(4/\delta) n^{-\frac{\beta-\tau}{2(\beta+2)}}.
 \end{equation}
    \end{itemize}
\end{theorem}
The proof can be found in Appendix~\ref{sec: proof-upperbound-rkhs}.
According to Theorem~\ref{primal-bound}, the lower bounds for $L^\infty$ estimation under Assumption \ref{assu: kernel} are $\Omega(n^{-\beta/2})$ in the noiseless case and $\Omega(n^{-\beta/(2\beta+2)})$ in the noisy case.  Comparing these results with the above upper bounds, we observe:
\begin{itemize}
\item In the noiseless case,  the upper bound given in \eqref{eqn: Lq-learning-upper} matches the lower bound. Hence, we establish, for the first time to the best of our knowledge, that \textit{KRR is minimax optimal},  up to logarithmic factors, for the $L^\infty$ learning of RKHS in the noiseless scenario. 
\item In the noisy case, however, a discrepancy arises: the upper bound \eqref{eqn: uniform-upperbound-noisy} does not match the lower bound $\Omega(n^{-\beta/(2\beta+2)})$. This suggests three possibilities: 1) the lower bound we derived in  \eqref{eqn: lower-bound-2} is not tight; 2) the upper bound established in \cites{sobolevkernel} is loose; 3) KRR is essentially not optimal when it comes to $L^\infty$ learning for RKHS functions.
\end{itemize}

To clarify this discrepancy, we further consider a specific case involving periodic kernels on $\TT$ given by \eqref{eq:periodic}: $$ k(x,x') = \sum_{j\in \ZZ}\mu_j e_j(x)\overline{e_j(x')} = \sum_{j \in \ZZ} \mu_j e^{2\pi i j(x-x')}, $$ where  $\mu_j \asymp (|j|+1)^{-\beta-1}$.
% , in accordance with Assumption \ref{assu: kernel}. 
To enhance learning performance under the $L^\infty$ metric, we propose using KRR with a modified kernel, given by  $k_s(x,x') :=  \sum_{j\in \ZZ}\mu_j^s e_j(x)\overline{e_j(x')}$. When $s>1$, the eigenvalues of $k_s$ decay faster than those of $k$, indicating that $k_s$ is smoother. 

\begin{theorem}\label{thm:periodic_upper}
Suppose $f^*\in \cH_k(1)$ and 
    consider the fixed design setting with inputs  $x_i=(i-1)/n$ for $i=1,\dots,n$. Let $\hf_{\lambda,s}$  be the solution produced by KRR~\eqref{eqn: krr} using the kernel $k_s$. Setting $s\geq \max(\beta^{-1},1)$ and $\lambda \asymp (\frac{\log(n/\delta)}{n})^s$, we have that with probability $1-\delta$,
    \begin{equation*}
        \|\hat{f}_{\lambda,s} - f^*\|_{C_0(\TT)} \lesssim \left(\frac{\log(n/\delta)}{n}\right)^{\frac{\beta}{2(\beta+1)}}.
    \end{equation*}
\end{theorem}
The proof can be found in Appendix~\ref{sec: proof-periodic-kernel-upper-bound}.
This upper bound improves~\eqref{eqn: uniform-upperbound-noisy} and matches the lower bound $\Omega(n^{-\beta/(2\beta+2)})$. This indicates that our lower bound for  the noisy setting is generally sharp. It is  worth mentioning that this optimal upper bound is achieved by using a smoother kernel when $\beta<1$ and selecting inputs carefully (specifically, uniform grid points). We defer the discussion on whether these are necessary to future work, as our primary goal here is to affirm the sharpness of our lower bound \eqref{eqn: lower-bound-2}.  

Moreover, for this particular case, we can establish the following stronger lower bound, which holds regardless of the input sampling strategy. The proof is detailed in Appendix~\ref{sec: proof-periodic-kernel-lower-bound}.

\begin{theorem}\label{thm: Linfty-learning-lower-bound}
Let $\rho_i\in \cP(\TT)$ for $i\in[n]$.
    Suppose  $x_i\sim\rho_i,  \xi_i\stackrel{iid}{\sim}\cN(0,\sigma^2)$ with $\sigma\asymp 1$. We have
     \begin{align}
 \inf_{T \in \cA_n} \sup_{\|f\|_{\cH_k} \leq 1} \EE\|T(\{(x_i,f(x_i)+\xi_i)\}_{i=1}^n) -f  \|_{C_0(\TT)} \gtrsim   n^{-\frac{\beta}{2\beta+2}}.
     \end{align}
\end{theorem}

We remark that Theorems \ref{thm:periodic_upper} and \ref{thm: Linfty-learning-lower-bound} can be straightforwardly extended to periodic kernels on the high-dimensional torus $\TT^d$, given by
$
k(x,x') = \sum_{j \in \ZZ^d}\mu_j e^{2\pi \ii j \cdot (x-x')}
$
for $x, x' \in \TT^d$ with $\mu_j \asymp (\|j\|_\infty + 1)^{-d(\beta+1)}$. 
 We focus on the one-dimensional case in this paper to avoid the notation complication, which would detract from the core ideas we wish to convey.

\section{Conclusion}
In this paper, we establish a duality between approximation and estimation for the problem of function learning using information-based complexity. This duality  enables us to convert approximation problems to estimation problems and vice versa. Therefore, one only needs to focus on the more tractable one among the two. Particularly, we instantiate this duality for the $\cF_{p,\pi}$  and Barron spaces, which play  crucial roles in  the analysis of  RFMs, neural networks, as well as kernel methods.  
To demonstrate the power of our duality framework, we provide comprehensive analyses of two specific problems: The efficacy of RFMs in a beyond-kernel regime and the $L^\infty$ learning of RKHS. Our duality framework not only recovers existing results with simpler proofs, but also leads to stronger and new results.

For future work, it would be interesting to leverage our duality framework  to provide a comprehensive analysis for  the  learning of RKHS and Barron spaces, in particular for the noiseless setting. In this paper, we have only examined the $L^\infty$ learning of RKHS for the illustration purpose. Another interesting question is to utilize the duality to study the learning under Sobolev norms, which are especially relevant in solving partial differential equations. Additionally, it is also intriguing to apply our duality framework to other learning paradigms, such as unsupervised learning, reinforcement learning~\cites{perturbation,long2023reinforcement}, and transfer learning~\cites{ma2023optimally,kalavasis2024transfer}.

\section*{Acknowledgements}
Lei Wu is supported by the National Key R\&D Program of China (No.~2022YFA1008200), National Natural Science Foundation of China (No.~2288101), and a startup fund from Peking University. 
Hongrui Chen was partially supported by the elite undergraduate training program of School of
Mathematical Sciences at Peking University. We sincerely thank the anonymous associate editor and   reviewers  for their valuable suggestions and insightful comments, which have helped us significantly improve the quality of this paper.

\bibliography{ref}

\newpage
\appendix
\addcontentsline{toc}{section}{Appendix} 
\part{Appendix} 
\parttoc 

\section{Proofs in Section \ref{sec: information-based complexity}}
\label{sec: proof-sec-minimax}

\subsection{Proof of Proposition \ref{lemma-data-dependent}}
\label{sec: proof-sample-dependent}

\textbf{Upper bound.} Recall the minimum-norm estimator:
$$
       S(\{(x_i,y_i)\}_{i=1}^n) = \argmin_{h \in \cF,\,h(x_i) = y_i} \|h\|_{\cF}.
$$
Suppose $y_i=g(x_i)$ for some $g\in \cF(1)$. Then, we have
\begin{align*}
\|S(\{(x_i,g(x_i)) \}_{i=1}^n)\|_\cF \leq \|g\|_\cF,\qquad S(\{(x_i,g(x_i) \}_{i=1}^n)(x_i) = g(x_i),\quad \forall i\in [n],
\end{align*}
implying
\begin{align}\label{eqn: 001}
     S(\{(x_i,g(x_i)) \}_{i=1}^n) - g \in \{f \in \cF: \|f\|_{\cF} \leq 2,\,f(x_i) = 0,\,i\in [n]\}.
\end{align}
Therefore,
\begin{align*}
 \inf_{T \in \cA_n}\sup_{\|g\|_{\cF} \leq 1}      \|T(\{(x_i, g(x_i))\}_{i=1}^n) - g \|_{\cM} &\leq    \sup_{\|g\|_{\cF} \leq 1} \|  S(\{(x_i,g(x_i)) \}_{i=1}^n) - g \|_\cM \\ & \leq \sup_{\|f\|_\cF \leq 2,\,f(x_i) = 0}\|f\|_\cM \\ & = 2 \sup_{\|f\|_{\cF}\leq 1,\, f(x_i)=0} \|f\|_{\cM} = 2\II_{L^2(\bar{\rho}),\cM}(\cF(1),0),
\end{align*}
where the second step follows from \eqref{eqn: 001}.

\textbf{Lower bound.} By the definition of the $I$-complexity, for any $\epsilon >0$, there must exist a $h \in \cF$ such that $\|h_\epsilon\|_\cF \leq 1,\,h(x_i) = 0$ for $i\in [n]$ and $\|h_\epsilon\|_\cM \geq \II_n(\cF(1),\cM,0) - \epsilon  $. Note that any estimator $T \in \cA_n$ is unable to distinguish $h_\epsilon$ and $-h_\epsilon$ since
\begin{align*}
                 T(\{(x_i,h_\epsilon(x_i))\}_{i=1}^n) =    T(\{(x_i,-h_\epsilon(x_i))\}_{i=1}^n).
\end{align*}
Therefore, for any $T \in \cA_n$, we have
\begin{align*}
 \sup_{\|g\|_{\cF} \leq 1}      \|T(\{(x_i, &g(x_i))\}_{i=1}^n) - g \|_{\cM} \\ 
 & \stackrel{\mathrm{(i)}}{\geq}  \frac{1}{2}  \| T(\{(x_i, h_\epsilon(x_i) )\}_{i=1}^n ) - h_\epsilon\|_\cM + \frac{1}{2}   \| T(\{(x_i, -h_\epsilon(x_i) )\}_{i=1}^n ) - (-h_\epsilon)\|_\cM \\
 & \geq \|h_\epsilon\|_\cM \\
 & \geq \II_{\bar{\rho}}(\cF(1),\cM,0) - \epsilon,
 \end{align*}
where (i) follows from the triangle inequality.
Taking $\epsilon \to 0$ completes the proof.
 \qed

 \subsection{Proof of Lemma \ref{lem:dual_core}}\label{sec: proof-lemma-dual}
\begin{definition}[Convex conjugate]\label{def: convex-conjugate}
    Let $E$ be a Banach space and $\theta :E\mapsto(-\infty,\infty]$ be a function such that $\theta \not\equiv \infty$. We define its convex conjugate  $\theta^*: E^* \mapsto(-\infty,\infty]$ to be
    \begin{equation}
        \theta^*(l) = \sup_{x \in E}[l(x) - \theta(x)].
    \end{equation}
\end{definition}
\begin{theorem}[Fenchel-Rockafellar theorem, see Theorem 1.12 in \cites{brezis2011functional}]\label{thm:Fenchel–Rockafellar}
    Let $\varphi,\psi: E \mapsto (-\infty,\infty]$ be two convex functions. Assume there exists $x_0 \in E$ such that $\varphi(x_0)<\infty$, $\psi(x_0)<\infty$ and $\varphi$ is continuous at $x_0$. Then
    \begin{equation*}
        \inf_{ x \in E}[\varphi(x) + \psi(x)] = -\inf_{l \in E^*}[\varphi^*(l) + \psi^*(-l)].
    \end{equation*}
\end{theorem}

Next, we apply Theorem \ref{thm:Fenchel–Rockafellar} to prove Lemma \ref{lem:dual_core}.
\paragraph*{Proof of Lemma \ref{lem:dual_core}.}
we construct  $\varphi,\psi:\cF\mapsto\RR$ as
   \begin{equation*}
    \varphi(f) = \begin{cases}
   g(f) &\text{ if }\|f\|_\cF \le 1 \\
    \infty &\text{ if } \|f\|_\cF > 1
    \end{cases},\qquad
    \psi(f) = \begin{cases}
   0 &\text{ if } \|f\|_{\cQ}\le \epsilon\\
    \infty &\text{ if } \|f\|_{\cQ}> \epsilon
    \end{cases}.
\end{equation*}
Note that both $\varphi$ and $\psi$ are well-defined due to $g\in\cF^*$. It is easy to verify:
\begin{itemize}
\item Both $\varphi$ and $\psi$ are convex, $\varphi(0) = \psi(0) = 0 <\infty$, and $\varphi$ is continuous at $0$ due to $g\in \cF^*$.
\item By Definition~\ref{def: convex-conjugate}, we have
\begin{equation}\label{eqn:convex_conjugate_2}
    \begin{aligned}
    &\varphi^*(l) = \sup_{f \in \cF}[l(f) - \varphi(f)] = \sup_{\|f\|_{\cF} \le 1}[l(f) - g(f)] = \sup_{\|f\|_{\cF} \le 1}(l-g)(f) = \|l-g\|_{\cF^*},\\
    &\psi^*(l) = \sup_{f \in \cF}[l(f) - \psi(f)] = \sup_{ f \in \cF,\|f\|_{\cQ} \le \epsilon} l(f) = \epsilon \sup_{ f \in \cF,\|f\|_{\cQ} \le 1} l(f).
\end{aligned}
\end{equation}
\item Additionally,
\vspace*{-1em}
\begin{align}\label{eqn:12345}
\notag    \inf_{x \in \cF}[\varphi(x)+\psi(x)] &= \inf_{\|f\|_\cF \le 1, \|f\|_\cQ \le 1} g(f) = -\sup_{\|f\|_\cF \le 1, \|f\|_\cQ \le 1} g(-f) \\ 
    &= -\sup_{\|f\|_\cF \le 1, \|f\|_\cQ \le 1} g(f). 
\end{align}
\end{itemize}
Combining Theorem \ref{thm:Fenchel–Rockafellar},  \eqref{eqn:convex_conjugate_2}, and \eqref{eqn:12345}, we obtain
\begin{align}\label{eqn:12346}
    \sup_{\|f\|_\cF \le 1, \|f\|_\cQ \le 1} g(f) &= -\inf_{x \in \cF}\big[\varphi(x)+\psi(x)\big]= \inf_{l \in E^*}\big[\varphi^*(l) + \psi^*(-l)\big]\notag\\
    &=\inf_{ l \in \cF^*}\Big[\left\|l - g\right\|_{\cF^*} + \epsilon \sup_{ f \in \cF,\|f\|_{\cQ} \le 1} (-l)(f) \Big]\notag \\ 
    &= \inf_{ l \in \cF^*}\Big[\left\|l - g\right\|_{\cF^*} + \epsilon \sup_{ f \in \cF,\|f\|_{\cQ} \le 1} l(f) \Big],
\end{align}
where the last step follows from  changing from $-f$ to $f$. It remains to show that the right hand side of \eqref{eqn:12345} equals $ \inf_{ l \in \cQ^*}\Big[\left\|l - g\right\|_{\cF^*} + \epsilon \|l\|_{\cQ^*}\Big]$.
  
\paragraph*{Step 1:}
Noticing that $\cQ^*\hookrightarrow \cF^*$, we have 
\begin{align}\label{eqn:12347}
    \sup_{\|f\|_\cF \le 1, \|f\|_\cQ \le 1} g(f)& \le \inf_{ l \in \cQ^*}\Big[\left\|l - g\right\|_{\cF^*} + \epsilon \sup_{ f \in \cF,\|f\|_{\cQ} \le 1} l(f) \Big]\notag\\ 
    &\leq \inf_{ l \in \cQ^*}\Big[\left\|l - g\right\|_{\cF^*} + \epsilon \sup_{\|f\|_{\cQ} \le 1} l(f) \Big]\notag\\ 
    & = \inf_{ l \in \cQ^*}\Big[\left\|l - g\right\|_{\cF^*} + \epsilon \|l\|_{\cQ^*}\Big].
\end{align}

\paragraph*{Step 2:} For any $l \in \cF^*$, if $M_l := \sup_{ f \in \cF,\|f\|_{\cQ} \le 1} l(f) < \infty$, then 
\begin{equation*}
    |l(f)| \le M_l \|f\|_{\cQ},\quad \forall f \in \cF.
\end{equation*}
Thus, by Hahn-Banach Theorem~\cite[Section 5]{folland1999real}, we can then extend $l$ from $\cF$ to $\cQ$ with a continuous extension $\tl\in\cQ^*$, satisfying
\begin{equation*}
    \tilde{l}(f) = l(f), \,\forall f \in \cF,\qquad |\tilde{l}(h)|\le M_l\|h\|_{\cQ},\, \forall h \in \cQ.
\end{equation*}
Note that  $\|\tilde{l}\|_{\cQ^*} = \sup_{f \in \cF,\|f\|_{\cQ} \le 1} l(f)$, due to 
\begin{align*}
 \|\tilde{l}\|_{\cQ^*} &\le M_l = \sup_{ f \in \cF,\|f\|_{\cQ} \le 1} l(f), \\ 
 \|\tilde{l}\|_{\cQ^*} &= \sup_{\|f\|_{\cQ}\le 1} l(f) \geq \sup_{f \in \cF,\|f\|_{\cQ} \le 1} l(f).
\end{align*}
Then, using \eqref{eqn:12346}, we obtain
\begin{align}\label{eqn:12348}
\sup_{\|f\|_\cF \le 1, \|f\|_\cQ \le 1} g(f) &= \inf_{ l \in \cF^*}\Big[\left\|l - g\right\|_{\cF^*} + \epsilon \sup_{ f \in \cF,\|f\|_{\cQ} \le 1} l(f) \Big]\notag\\ 
&=\inf_{ l \in \cF^*}\Big[\sup_{\|f\|_{\cF}\leq 1}[l(f) - g(f)] + \epsilon \sup_{ f \in \cF,\|f\|_{\cQ} \le 1} l(f) \Big]\notag\\ 
&=\inf_{ l \in \cF^*}\Big[\sup_{\|f\|_{\cF}\leq 1}[\tl(f) - g(f)] + \epsilon \|\tl\|_{\cQ^*} \Big]\notag\\ 
&=\inf_{ l \in \cF^*}\Big[\|\tl-g\|_{\cF^*} + \epsilon \|\tl\|_{\cQ^*} \Big]\notag\\ 
&\geq \inf_{ l \in \cQ^*}\Big[\|l - g\|_{\cF^*} + \epsilon \|l\|_{\cQ^*}\Big].
\end{align}

\paragraph*{Step 3:} Combining \eqref{eqn:12347} and $\eqref{eqn:12348}$ completes the proof.

\subsection{Proof of the lower bound~\eqref{uvw}}
\label{sec: proof-pro-noisy}
\begin{proof}
For  $f \in \cF$, let $\cD_{f}^n\in \cP((\cX\times\RR)^{n})$ denote the law of $\{(x_i,f(x_i)+\xi_i)\}_{i=1}^n$.
For a given estimator $T\in\cA_n$ and a target function $f \in \cF$, the performance of $T$ is measured by 
\begin{align*}
    d_n(T, f) := \| T(\{ (x_i,f(x_i)+ \xi_i)\}_{i=1}^n) - f \|_\cM.  
\end{align*}
By the  Le Cam's two-point method \cite[Chapter 15.2]{wainwright2019high}, we have for any $f_1,f_2 \in \cF(1)$ that
\begin{align*}
    \inf_{T\in\cA_n}\sup_{\|f\|_\cF \leq 1} \EE [d_n(T,f)] & \geq \inf_{T\in\cA_n}\max
    \big\{\EE [d_n(T,f_1)],  \EE  [d_n(T,f_2)]  \big\} \\
    & \geq \frac{\|f_1-f_2\|_\cM}{2}\Big(1-  \mathrm{TV}(\cD_{f_1}^n,\cD_{f_2}^n) \Big). 
\end{align*}
Applying the Pinsker's inequality~\footnote{For two probability distributions $P,Q$ defined on the same domain, $\|P-Q\|_{\mathrm{TV}}\leq \sqrt{\mathrm{KL}(P||Q)/2}$.} \cite[Lemma 15.2]{wainwright2019high}, we obtain
\begin{align}   \label{1234}
     \inf_{T\in\cA_n}\sup_{\|f\|_\cF \leq 1} \EE [d_n(T,f)]& \geq \frac{\|f_1-f_2\|_\cM}{2}\left(1- \sqrt{\frac{\mathrm{KL}(\cD_{f_1}^n\| \cD_{f_2}^n)}{2}} \right) .
\end{align}
The KL divergence between $D_{f_1}^n$ and $D_{f_2}^n$ can be computed as follows
\begin{align}
    \mathrm{KL}(\cD_{f_1}^n \| \cD_{f_2}^n ) = & \EE_{\cD_{f_1^n}}\left[\log \frac{\d \cD_{f_1}^n }{\d \cD_{f_2}^n}\right] \notag  \\
    & = \EE_{\cD_{f_1}^n}\left[\log  \left( \frac{\prod_{i=1}^n \exp\left(-\frac{\|\xi_i\|^2}{2\sigma^2} \right) }{\prod_{i=1}^n \exp\left(-\frac{\|f_2(x_i) +\xi_i - f_1(x_i)\|^2}{2\sigma^2} \right)}\right) \right] \notag \\
    & = \EE_{x_i\sim \rho_i,\,\xi_i\sim \cN(0,\sigma^2I_d)} \left[\sum_{i=1}^n \left(\frac{\|f_2(x_i)-f_1(x_i)+\xi_i\|^2 - \|\xi_i\|^2}{2\sigma^2} \right)\right] \notag \\
    &=\frac{1}{2\sigma^2}\sumin \|f_2-f_1\|_{\rho_i}^2 =\frac{n}{2\sigma^2} \|f_2 - f_1\|_{\bar{\rho}}^2. \label{2345}
\end{align}
Combining \eqref{1234} and \eqref{2345}, we arrive at
\begin{align} \label{3456}
  \inf_{T \in \cA_n} \sup_{\|f\|_\cF \leq 1} \EE[d_n(T,f)] \geq \frac{\|f_1-f_2\|_\cM}{2}\left(1 - \frac{\sqrt{n}\|f_2-f_1\|_{\brho}}{2\sigma} \right)
\end{align}
\end{proof}

\section{Proofs in Section \ref{sec: spaces}}\label{sec: proof-union-barron}

\subsection{Proof of Lemma \ref{lemma: barron-space-alternative}}
We provide a proof of Lemma \ref{lemma: barron-space-alternative} below.

\paragraph*{Step 1:}
First, we show that $\cB=\cup_{\pi\in \cP(\cV)}\cF_{p,\pi}$ is a consequence of 
\begin{equation}\label{eqn: 010}
\|f\|_{\cB} = \inf_{\pi\in \cP(\cV)} \|f\|_{\cF_{p,\pi}}.
\end{equation}

If $f\in\cup_{\pi\in \cP(\cV)}\cF_{p,\pi}$, there must exist $\pi_f\in \cP(\cV)$ such that $\|f\|_{\cF_{p,\pi_f}}<\infty$.  By \eqref{eqn: 010}, we have 
\[
\|f\|_{\cB}=\inf_{\pi\in \cP(\cV)} \|f\|_{\cF_{p,\pi}}\leq \|f\|_{\cF_{p,\pi_f}} <\infty.
\]
Thus, $f\in \cB$. This implies that $\cup_{\pi\in \cP(\cV)}\cF_{p,\pi}\subset \cB$.

If $f\in \cB$, then $\|f\|_{\cB} = \inf_\pi \|f\|_{\cF_{p,\pi}}<\infty$. This implies that there must exist $\pi_f\in \cP(\cV)$ such that $\|f\|_{\cF_{p,\pi_f}}<\infty$, i.e., $f\in \cF_{p,\pi_f}$. Thus, $\cB\subset \cup_{\pi\in \cP(\cV)}\cF_{p,\pi}$.

\paragraph*{Step 2:}Now, what remains is  to prove \eqref{eqn: 010}.  By Definition \ref{def:F_1}, for any $f \in \cB$, there must exist $\mu \in \cM(\cV)$ satisfying $\|f\|_{\cB}=\|\mu\|_{\mathrm{TV}} <\infty$ and  $f = \int_\cV \phi(\cdot,v) \d \mu(v)$. Let $\mu=\mu_{+}-\mu_{-}$ be the Jordan decomposition of $\mu$ with $A_{+}=\supp(\mu_{+})$  and $A_{-}=\supp (\mu_{-})$. Let $|\mu|=\mu_{+}+\mu_{-}$, $\pi = \frac{|\mu|}{\|\mu\|_{\mathrm{TV}}}$, and 
\[
a(v) = \begin{cases}
 \|\mu\|_{\mathrm{TV}},\quad & \text{if}\,v\in A_{+}, \\
 - \|\mu\|_{\mathrm{TV}},\quad & \text{if}\,v\in A_{-}.\\
     \end{cases} 
\]
Then, $\|a\|_{L^\infty(\pi)} = \|\mu\|_{\mathrm{TV}}<\infty$ and $f=\int a(v)\varphi(\cdot,v)\dd\pi(v)$. This implies that $\|f\|_{\cF_{\infty, \pi}}=\|\mu\|_{\mathrm{TV}}$.
Therefore, for any $p\in [1,\infty]$, we have
\begin{equation}\label{eqn: 002}
    \inf_{\pi'\in \cP(\cV)}\|f\|_{p, \pi'}\leq  \inf_{\pi'\in \cP(\cV)}\|f\|_{\infty, \pi'}\leq \|f\|_{\infty, \pi}=\|f\|_{\cB},
\end{equation}
where the first inequality follows from H\"older's inequality.

For any $f\in\cB$, suppose that $\pi_f$ such that $\|f\|_{1,\pi_f}=\inf_{\pi\in \cP(\cV)}\|f\|_{1,\pi}$. By the definition of $\cF_{p,\pi}$ space, 
there must exist $a(\cdot)$ such that $f=\int a(v)\varphi(\cdot,v)\dd\pi_f(v)$ and $\int |a(v)|\dd \pi_f(v)=\|f\|_{\cF_{1,\pi_f}}$. By choosing $\mu_f \in \cM(\cV)$ such that $\frac{\d \mu}{\d \pi_f} = a$,  we have  $\|\mu_f\|_{\mathrm{TV}} = \int_\cV \d |\mu_f|(v) = \int_\cV |a(v)| \d \pi_f(v)= \|f\|_{\cF_{1,\pi_f}}$ and  $\int_\cV \phi(\cdot,v) \d \mu_f(v) =\int_\cV a(v)\phi(\cdot,v) \d \pi_f(v) =f $. Thus,  we have for any $p\in [1,\infty]$ that
\begin{equation}\label{eqn: 003}
    \|f\|_{\cB}\leq \|f\|_{\cF_{1,\pi_f}} = \inf_{\pi\in \cP(\cV)}\|f\|_{1,\pi}\leq \inf_{\pi\in \cP(\cV)}\|f\|_{p,\pi},
\end{equation}
where the last inequality is due to  H\"older's inequality.

By combining \eqref{eqn: 002} and \eqref{eqn: 003}, we complete the proof of \eqref{eqn: 010}.
\qed

\section{Proofs in Section \ref{sec: dual-equivalence}}
\label{sec: proof-dual-equivalence-r} 

\subsection{Proof of Theorem \ref{thm: dual}}
We note that the proofs of \eqref{dual-eq-2} and \eqref{dual-eq-3}  follow exactly the same steps as the one of \eqref{dual-eq-1}. The only difference is that different representation theorems may be used as the underlying function spaces change. Nevertheless, we provide the proofs below for completeness.

\paragraph*{Proof of \eqref{dual-eq-2} for $q\in [1,\infty)$.}
Applying Theorem~\ref{thm:dual_abstract}, we obtain
\begin{align}\label{eqn: rst}
\notag    \sup_{\|f\|_{\cB} \leq 1,\|f\|_{r,\nu} \leq \epsilon}\|f\|_{q,\rho}&=\sup_{\|\tb\|_{L^q(\rho)^*}\leq 1}\inf_{\tc \in L^r(\nu)^*}\left[\|\tb - \tc\|_{\cB^*} + \epsilon\|\tc\|_{L^r(\nu)^*}\right]\\ 
\notag     &   = \sup_{\|\tb\|_{L^q(\rho)^*}\leq 1}\inf_{\tc \in L^r(\nu)^*}\left[\sup_{\|f\|_\cB\leq 1}[\tb(f) - \tc(f)] + \epsilon\|\tc\|_{L^r(\nu)^*}\right]\\ 
     &=\sup_{\|b\|_{L^{q'}(\rho)}\leq 1}\inf_{c \in L^{r'}(\nu)}\left[I(b,c;f) + \epsilon\|c\|_{L^{r'}(\nu)}\right],
\end{align}
where the third steps use the Riesz representation theorem for $L^q(\rho)^*$ and $L^r(\nu)^*$, and 
\begin{align}\label{eqn: opq}
\notag    I(b,c;f) =   &\sup_{\|f\|_{\cB} \leq 1} \Big(\int_{\cX}b(x) f(x)\dd \rho(x) - \int_{\cX}c(x)f(x)\dd \nu(x)\Big)\\
\notag        \stackrel{\mathrm{(i)}}{=}& \sup_{\|\gamma\|_{\tv}\leq 1}\int_{\cX}\left(\int_{\cV}\phi(x,v)\dd\gamma(v)\right)[b(x)\dd\rho(x) - c(x)\dd\nu(x)] \\
\notag        =& \sup_{\|\gamma\|_{\tv}\leq 1}\int_{\cV}\left(\int_{\cX}b(x)\phi(x,v)\dd\rho(x) - \int_{\cX}c(x)\phi(x,v)\dd\nu(x)\right) \dd\gamma(v) \\
        \stackrel{\mathrm{(ii)}}{=}&\left\|\int_{\cX}b(x)\phi(x,\cdot)\dd\rho(x) - \int_{\cX}c(x)\phi(x,\cdot)\dd\nu(x)\right\|_{C_0(\cV)},
\end{align}
where $\mathrm{(i)}$ uses the definition of $\cB$ and $\mathrm{(ii)}$ is due to $\sup_{\|\gamma\|_{\tv}\leq 1}|\int_{\cV} s(x)\dd\nu(v)|=\|s\|_{C_0(\cV)}$ for any $s\in C_0(\cV)$.
Plugging \eqref{eqn: opq} into \eqref{eqn: rst} and using the definition of $\tilde{\cF}_{q',\rho}$ and $\tilde{\cF}_{r',\nu}$, we complete the proof.

\paragraph*{Proof of \eqref{dual-eq-3}.} Applying Theorem~\ref{thm:dual_abstract}, we obtain
\begin{align}\label{eqn: rst2}
\notag \sup_{\|f\|_{\cF_{p,\pi}}\leq 1,\|f\|_{r,\nu} \leq \epsilon} \|f\|_{C_0(\cX)} &=\sup_{\|\tb\|_{C_0(\cX)^*}\leq 1}\inf_{\tc \in L^r(\nu)^*}\left[\|\tb - \tc\|_{\cF_{p,\pi}^*} + \epsilon\|\tc\|_{L^r(\nu)^*}\right]\\ 
\notag     &   = \sup_{\|\tb\|_{C_0(\cX)^*}\leq 1}\inf_{\tc \in L^r(\nu)^*}\left[\sup_{\|f\|_{p,\pi}\leq 1}[\tb(f) - \tc(f)] + \epsilon\|\tc\|_{L^r(\nu)^*}\right]\\ 
&=\sup_{\|\gamma\|_{\tv}\leq 1}\inf_{c \in L^{r'}(\nu)}\left[I(b,c;f) + \epsilon\|c\|_{L^{r'}(\nu)}\right],
\end{align}
where the second step uses the Riesz representation theorem of $C_0(\cX)$, i.e., $C_0(\cX)^*=\cM(\cX)$, and
\begin{align}\label{eqn: opq2}
\notag    I(b,c;f) =   &\sup_{\|f\|_{\cF_{p,\pi}} \leq 1} \Big(\int_{\cX} f(x)\dd \gamma(x) - \int_{\cX}c(x)f(x)\dd \nu(x)\Big)\\
\notag        =& \sup_{\|a\|_{p,\pi}\leq 1}\int_{\cX}\left(\int_{\cV}a(v)\phi(x,v)\dd\pi(v)\right)[\dd \gamma(x) - c(x)\dd\nu(x)] \\
\notag        =& \sup_{\|a\|_{p,\pi}\leq 1}\int_{\cV}a(v)\left(\int_{\cX}\phi(x,v)\dd\gamma(x) - \int_{\cX}c(x)\phi(x,v)\dd\nu(x)\right) \dd\pi(v) \\
        =&\left\|\int_{\cX}\phi(x,\cdot)\dd\gamma(x) - \int_{\cX}c(x)\phi(x,\cdot)\dd\nu(x)\right\|_{p',\pi}.
\end{align}
Plugging \eqref{eqn: opq2} into \eqref{eqn: rst2} and using the definition of $\tilde{\cB}$ and  $\tilde{\cF}_{r',\nu}$, we complete the proof.

\paragraph*{The case where $q=\infty$ in \eqref{dual-eq-1} and \eqref{dual-eq-2}.} Note that $L^\infty(\rho)=L^1(\rho)^*$. Therefore,  the proof of this endpoint case follows exactly the same steps by applying the following revisited dual equivalence: Assuming $\cF\hookrightarrow\cQ$ and $\cF\hookrightarrow\cM^*$, it holds that
\begin{equation}\label{eqn: thm3-new}
   \II_{\cQ,\cM^*}(\cF(1),\epsilon) = \sup_{\|g\|_{\cM} \le 1}\inf_{h \in \cQ^* }[\|g - h\|_{\cF^*} + \epsilon\|h\|_{\cQ^*}].
\end{equation}
\begin{proof}[Proof of \eqref{eqn: thm3-new}]
Since $\cF \hookrightarrow \cM^*$, we have $\cM\hookrightarrow\cF^*$ and thus, $g\in \cF^*$. Then, by Lemma~\ref{lem:dual_core}, we obtain
\begin{equation*}
     \sup_{\|f\|_{\cF}\le 1, \|f\|_{\cQ} \le \epsilon} g(f) = \inf_{h \in \cQ^*}\left[\|g - h\|_{\cF^*} + \epsilon\|h\|_{\cQ^*}\right].
\end{equation*}
Taking the supremum over $\{g \in \cM: \|g\|_{\cM} \le 1\}$  
on the both sides gives
\begin{align*}
    \sup_{\|g\|_{\cM} \le 1}\inf_{h \in \cQ^* }[\|g - h\|_{\cF^*} + \epsilon\|h\|_{\cQ^*}] &=\sup_{\|g\|_{\cM} \le 1} \sup_{\|f\|_{\cF}\le 1, \|f\|_{\cQ} \le \epsilon} g(f)\\
    & = \sup_{\|f\|_{\cF}\le 1, \|f\|_{\cQ} \le \epsilon} \sup_{\|g\|_{\cM} \le 1}g(f)\\
   &\stackrel{\mathrm{(i)}}{=} \sup_{\|f\|_{\cF}\le 1, \|f\|_{\cQ} \le \epsilon} \|f\|_{\cM^*} =  \II_{\cQ,\cM^*}(\cF,\epsilon),  
\end{align*}
where $\mathrm{(i)}$ uses the definition of the $\|\cdot\|_{\cM^*}$ norm.
\end{proof}

\section{Proofs in Section \ref{sec: random-app}}

We first present a standard local Rademacher complexity-based generalization bound. We will apply this result to establish the estimation error bound on the conjugate space.
\begin{lemma}\label{localRad}[Theorem 1 in \cite{localRad}] Let $\cF$ denote a class of functions from the input space $\cX \subset \RR^d$ to the output space $\cY \subset \RR$, and $\ell: \cX \times \cY \to \RR$ denote a loss function. Let $\mu$ denote the population distribution on $\cX \times \cY$ and $(x_i,y_i) \stackrel{iid}{\sim}$ for $i\in [n]$. Suppose that:
\begin{itemize}
    \item The Rademacher complexity of $\cF$ is bounded by $R_n$, i.e.,
    $
     \rad(\cF) \leq R_n.
    $
    \item The loss function $\ell$ is $H$-smooth with respect to the input, i.e.,
    $
        |\partial_{x,x}^2\ell(x,y)| \leq H 
    $
    $\,\mu$-a.e.
    \item The loss function $\ell$ is bounded by $b$ and non-negative, i.e., 
    $
      0 \leq  |\ell(x,y)| \leq b 
    $
    $\,\mu$-a.e.
\end{itemize}
Consider the empirical error and generalization error:
    \begin{align*}
        \hat{\cR}(h) = \frac{1}{n}\sum_{i=1}^n \ell(h(x_i),y_i), \quad \cR(h) = \EE_{(x,y)\sim\mu} [\ell(h(x),y)].
    \end{align*}
    Then, with probability at least $1-\delta$ over samples $\{(x_i,y_i)\}_{i=1}^n$, we have for any $h \in \cF$,
\begin{align*}
\cR(h) \leq \hat{\cR}(h)+K\left(\sqrt{\hat{\cR}(h)}\left(\sqrt{H} \log ^{3/2} n R_n+\sqrt{\frac{b \log (1 / \delta)}{n}}\right)+H \log ^{3} n R_n^2 +\frac{b \log (1 / \delta)}{n}\right),
\end{align*} 
where $K$ is an absolute constant.
\end{lemma}

\subsection{Proof of Lemma \ref{rad-lemma}}
\label{sec: proof-lemma-rad}
The Khintchine inequality is essential to bound the (local)-Rademacher complexity of the $\cF_p$ space:
\begin{lemma}(Khintchine inequality, Exercise 2.6.5 in \cites{vershynin2019high})\label{lemma: khintchin}
Let $X_1,\cdots,X_n$ be independent $\sigma$-subgaussian random variables with zero means and unit variances, and let $a = \left(a_1, \ldots, a_N\right)^\top \in \mathbb{R}^N$. Then for $p \geq 2$ we have
$$
\left\|\sum_{i=1}^N a_i X_i\right\|_{L^p} \lesssim \sigma\sqrt{p}\left(\sum_{i=1}^N a_i^2\right)^{1 / 2}.
$$
\end{lemma}
\paragraph*{Proof of Lemma \ref{rad-lemma}}

First notice that when $q \in [2,\infty]$, for any $v_1,\dots,v_n \in \cV$, we have
\begin{align*}
     \EE_{\xi} \left[ \sup_{\|g\|_{\tilde{\cF}_{q',\rho} \leq 1}} \frac{1}{n}\left|\sum_{i=1}^n \xi_i g(v_i) \right|\right] & =  \EE_{\xi} \left[ \sup_{\|b\|_{q',\rho} \leq 1} \frac{1}{n}\sum_{i=1}^n \xi_i \int b(x)\phi(x,v_i) \d \rho(x) \right] \\
    & = \EE_\xi \left[\frac{1}{n}\sup_{\|b\|_{q',\rho} \leq 1} \int_\cX b(x)\left(\sum_{i=1}^n \xi_i \phi(x,v_i) \right) \d \rho(x)\right] \\
    & =\frac{1}{n} \EE_\xi  \left\|\sum_{i=1}^n \xi_i \phi(\cdot,v_i)\right\|_{q,\rho}.
\end{align*}
\begin{enumerate}
    \item If $1< q' \leq 2 $, for any $v_1,\cdots,v_n \in \cV$, we have
\begin{align*}
    \EE_{\xi} \left[ \sup_{\|g\|_{\tilde{\cF}_{q',\rho} \leq 1}} \frac{1}{n}\left|\sum_{i=1}^n \xi_i g(v_i) \right|\right] 
    & =\frac{1}{n} \EE_\xi  \left\|\sum_{i=1}^n \xi_i \phi(\cdot,v_i)\right\|_{q,\rho} \\
    & \stackrel{(a)}{\leq} \frac{1}{n}  \EE_{x \sim \rho}  \left[\EE_\xi \left|\sum_{i=1}^n \xi_i \phi(x,v_i) \right|^q\right]^{1/q} \\ 
    & \stackrel{(b)}{\lesssim} \frac{\sqrt{q}}{n}\EE_{x\sim \rho}    \left[ \sum_{i=1}^n \phi(x,v_i)^2\right]^{1/2} \\
    & \stackrel{(c)}{\leq} \frac{\sqrt{q}}{n} \left[\EE_{x\sim \rho}   \sum_{i=1}^n \phi(x,v_i)^2\right]^{1/2} \\ 
    &\stackrel{(d)}{\lesssim}\frac{R\sqrt{q}}{\sqrt{n}}
\end{align*} 
where where $(a)$ and $(c)$ use Jensen's inequality and the concavity of the function $z\mapsto z^{1/q}$ for $q\in (1,2]$,  $(b)$ follows from the  Khintchine inequality (Lemma \ref{lemma: khintchin}), and $(d)$ uses the assumption $\|\phi(\cdot,v)\|_\rho\leq R$ for any $v\in \cV$.
\item If $q' = 1$, for any $v_1,\cdots,v_n \in \cV$, we have
\begin{align*}
    \EE_{\xi} \left[ \sup_{\|g\|_{\tilde{\cF}_{1,\rho} \leq 1}} \frac{1}{n}\left|\sum_{i=1}^n \xi_i g(v_i) \right|\right] & = \frac{1}{n} \EE_\xi  \left\|\sum_{i=1}^n \xi_i \phi(\cdot,v_i)\right\|_{\infty,\rho} \\
    & \le\frac{1}{n} \EE_\xi \sup_{\|x\|\leq R}  \left|\sum_{i=1}^n \xi_i \varsigma(x^\top v_i)\right| \\
    & \stackrel{(a)}{\lesssim} \frac{L}{n}\EE_\xi \sup_{\|x\|\leq R} \left|\sum_{i=1}^n (\xi_i v_i)^\top x \right| \\
    & = \frac{LR}{n} \EE_\xi \left\|\sum_{i=1}^n \xi_i v_i\right\|_2 \\
    & \leq \frac{LR}{n} \sqrt{\EE_\xi \left\|\sum_{i=1}^n \xi_i v_i\right\|_2^2}  \\
    & = \frac{LR}{n} \sqrt{ \sum_{i=1}^n \|v_i\|^2},
\end{align*}
where $(a)$  uses the contraction property of Rademacher complexity \cite[Lemma 26.9]{shalev2014understanding}. 
\end{enumerate}
\qed

\subsection{Proof of Theorem \ref{thm: approximation-bound}}
\label{sec: proof-approx-rfm}

\begin{proof}
Using the dual equivalence stated in Theorem \ref{thm: dual}, the approximation error is equal to the $I$-complexity:
\begin{align*}%\label{11}
    \sup_{\|f\|_{\cF_{p,\pi}} \leq 1}\inf_{c\in\RR^m} \left[\left\|f- \frac{1}{m} \sum_{j=1}^m c_j\phi(\cdot,v_j) \right\|_{q,\rho} + \epsilon \left(\frac{1}{m}\sum_{j=1}^m |c_j|^p\right)^{1/p}\right]= \sup_{\|g\|_{\tilde{\cF}_{q',\rho}} \leq 1,\|g\|_{p',\hat{\pi}_m} \leq \epsilon} \|g\|_{p',\pi}, 
\end{align*}
where $\hat{\pi}_m = \frac{1}{m}\sum_{j=1}^m \delta_{v_j}$ is the empirical weight measure.

Next, we apply Lemma \ref{localRad} to bound the right-hand side. Note that Assumption \ref{as1} implies that functions in $\tilde{\cF}_{q',\rho}(1)$ is uniformly bounded by $M_q$. The loss function $(x,y) \mapsto |x - y |^{p'}$ is at least $p'(p'-1)(2M_{q})^{p'-2} $-smooth. By Lemma \ref{localRad}, with probability at least $1-\delta$, we have
\begin{align*}
\sup_{\|g\|_{\tilde{\cF}_{q',\rho}} \leq 1,\|g\|_{p',\hat{\pi}_m} \leq \epsilon} \|g\|_{p',\pi}^{p'} \lesssim {\epsilon^{p'}} + \frac{p'(p'-1)(2M_{q})^{p'-2}R_q^2 \log^3 m + M_q\log(1/\delta)}{m}.
\end{align*}
Taking the $1/p'$ power on the both sides, we obtain
\begin{align} \label{12}
    \sup_{\|g\|_{\cF_{q',\rho}} \leq 1,\|g\|_{p',\hat{\pi}_m} \leq \epsilon} \|g\|_{p',\pi} \lesssim {\epsilon} + \left(\frac{M_q^{p'-2}R_q^2\log^3 m + M_q\log(1/\delta)}{m} \right)^{1/p'}.
\end{align}
Taking $$\epsilon \asymp \left(\frac{M_q^{p'-2}R_q^2\log^3 m + M_q\log(1/\delta)}{m}\right)^{1/p'}, $$ \eqref{12} implies that for any $f \in \cF_{p,\pi}(1)$, there exists $c_1,\cdots,c_m \in \RR$ such that $\|c\|_p \lesssim {m^{1/p}}$ and 
$$\left\|f - \frac{1}{m}\sum_{j=1}^m c_j \phi(\cdot,v_j)\right\|_{q,\rho} \lesssim  \left(\frac{M_q^{p'-2}R_q^2\log^3 m + M_q\log(1/\delta)}{m} \right)^{1/p'}.$$
We complete the proof.
\end{proof}

\subsection{Proof of Theorem~\ref{thm:Fp-approximation-lower-bound}}
\label{sec: proof-lower-bound-periodic-RFM}

Recall that the feature function $\phi:\cX\times \TT \mapsto \CC$ considered in this section is given by
\begin{equation}\label{eqn:svd2}
    \phi(x,v) = \sum_{j \in \ZZ}\mu_j^{\half}  u_j(x) \overline{e_j(v)},
\end{equation}
with $\mu_j\asymp (|j|+1)^{-\beta-1}$ and $e_j(v) = \exp(2\pi \ii j v)$.

\paragraph*{Transform to an equivalent real-valued feature.}
Define $\tilde{\phi}(0,x,v) = \sqrt{2} \, \mathrm{Re}[\phi(x,v)]$ and $\tilde{\phi}(1,x,v) = \sqrt{2} \, \mathrm{Im}[\phi(x,v)]$. Then,
\[
    \tk_{\rho}(v,v') = \int_{\cX} \phi(x,v) \overline{\phi(x,v')} \dd\rho(x) = \frac{1}{2} \int_{\cX} \left[\tilde{\phi}(0,x,v)\tilde{\phi}(0,x,v') + \tilde{\phi}(1,x,v)\tilde{\phi}(1,x,v')\right] \dd\rho(x).
\]
This implies that our derived conclusions are the same by considering the real-valued feature $\tilde{\phi}: (\{0,1\} \times \cX) \times \TT \to \RR$, with the input distribution given by $\tilde{\rho}(\{j\} \times B) := \rho(B)/2$ for any $j \in \{0,1\}$ and  Borel measurable set $B \subset \cX$. Therefore, without loss of generality, we will adopt the complex-valued version for simplicity.

\subsubsection{Proof of Theorem~\ref{thm:Fp-approximation-lower-bound}}

\paragraph*{Proof of the lower bound \eqref{Fp-approximation-lower-bound}.}
    Given any $y \in \TT$ and $r \in \TT$, let
    \begin{equation}\label{eqn: def-a}
        a(v;y,r) = \sum_{ j \in \ZZ} r^{|j|} e^{2\pi \ii j (y-v)} = \frac{1-r^2}{1+r^2-2r\cos(2\pi(y-v))} > 0
    \end{equation}
    be the Poisson's kernel. Noticing that
    \begin{equation}\label{eqn: qqq}
    \begin{aligned}
       &\|a(\cdot;y,r)\|_{1,\pi_0} = \int_{\TT}|a(v;y,r)|\dd\pi_0(v) = \int_{\TT}a(v;y,r)\dd\pi_0(v) = 1, \\
       &\|a(\cdot;y,r)\|_{2,\pi_0} = \sqrt{\sum_{j \in \ZZ} r^{2|j|} |e^{2\pi \ii j y}|^2} = \sqrt{\frac{1+r^2}{1-r^2}}.
    \end{aligned}
    \end{equation}
  By H\"older's inequality, we know that for any $p \in [1,2]$,
    \begin{align}\label{eqn: xop}
         \|a(\cdot;y,r)\|_{p,\pi_0} &\le \|a(\cdot;y,r)\|_{1,\pi_0}^{\frac{2-p}{p}} \|a(\cdot;y,r)\|_{\pi_0}^{\frac{2(p-1)}{p}} \leq \left(\frac{1+r^2}{1-r^2}\right)^{\frac{2(p-1)}{2p}}\leq (1-r)^{-\frac{1}{p'}},
    \end{align}
    where the third step uses Eq.~\eqref{eqn: qqq}. We will establish lower bounds of approximating $\cF_{p,\pi_0}$  by considering its following subset:
    \[
        G_r=\left\{g_y=\int a(v;y,r)\phi(\cdot,v)\dd\pi_0(v): y\in\TT\right\},
    \] 
    where for each $y\in\TT$, it holds that
    \begin{align}\label{eqn: qqq2}
    \|g_y\|_{p,\pi_0} &= \|a(\cdot;y,r)\|_{p,\pi_0}\stackrel{\mathrm{use \eqref{eqn: xop}}}{\leq }(1-r)^{-\frac{1}{p'}}.
    \end{align}
    and
    \begin{align}\label{eqn: qqq3}
    g_y(x)=&\int_{\TT}a(v;y,r)\phi(x,v) \dd\pi_0(v) \nonumber\\ 
    &\stackrel{\text{use \eqref{eqn:svd} and \eqref{eqn: def-a}}}{=} \int_{\TT}\left(\sum_{ j \in \ZZ} r^{|j|} e^{2\pi \ii j (y-v)}\right)\left(\sum_{i \in \ZZ}\mu_i^{\half}  \overline{u_i(x)} e_i(v)\right)\dd\pi_0(v)\nonumber\\ 
    &=\sum_{j \in \ZZ} r^{|j|} \mu_j^\half e^{2\pi\ii jy} \overline{u_j(x)}
    \end{align}

   Without loss of generality, we assume that $\{\varphi_i\}_{i=1}^n$ are orthonormal in $L^2(\rho)$. Otherwise we can first perform Gram-Schmidt orthonormalization. Then, we have
       \begin{align}\label{eqn: qqq4}
       \notag\sup_{\|f\|_{\cF_{p,\pi_0}} \le 1}&\inf_{c\in\RR^n}\left\|f - \sum_{i=1}^nc_i\varphi_i\right\|_{\rho}^2 \geq \sup_{y\in\TT}\inf_{c\in\RR^n}\left\|\|g_y\|_{p,\pi_0}^{-1} g_y - \sum_{i=1}^nc_i\varphi_i\right\|_{\rho}^2\\ 
       % &\gtrsim \EE_{y \sim \pi_0}\inf_{c\in\RR^n}\left\|\int_{\TT}\|a(\cdot;y,r)\|_{p,\pi_0}^{-1}a(v;y,r)\phi(\cdot,v)\dd\pi_0(v)-\sum_{i=1}^nc_i\varphi_i\right\|_{\rho}^2\\
       \notag&\stackrel{\text{use \eqref{eqn: qqq2}}}{\gtrsim} (1-r)^{\frac{2}{p'}}\EE_{y \sim \pi_0}\inf_{c\in\RR^n}\left\|g_y-\sum_{i=1}^nc_i\varphi_i\right\|_{\rho}^2\qquad \\
       \notag&=(1-r)^{\frac{2}{p'}}\EE_{y\sim\pi_0}\left[\Big\|g_y\Big\|_{\rho}^2 - \sum_{i=1}^n \<g_y,\varphi_i\>^2\right]\\ 
       &=(1-r)^{\frac{2}{p'}}\left[\EE_{y\sim\pi_0}\Big\|g_y\Big\|_{\rho}^2 - \sum_{i=1}^n \EE_{y\sim\pi_0}\<g_y,\varphi_i\>^2\right].
    \end{align}
    Using \eqref{eqn:svd2} and \eqref{eqn: qqq2}, we have
   \begin{align}\label{eqn:5111}
     \notag     \EE_{y\sim\pi_0}[\|g_y\|^2_\rho]&=\EE_{y\sim\pi_0}\left\|\int_{\TT}a(v;y,r)\phi(\cdot,v)\dd\pi_0(v)\right\|_{\rho}^2 \\ 
          &= \EE_{y\sim\pi_0}\sum_{j \in \ZZ}r^{2|j|}\mu_j |e^{2\pi\ii j y}|^2 = \sum_{j \in \ZZ}r^{2|j|}\mu_j
\end{align}
and
 \begin{align}\label{eqn:5112}
\notag \EE_{y\sim\pi_0}\<g_y,\varphi_i\>^2 &= 
\EE_{y\sim\pi_0}\left|\int_{\TT}a(v;y,r)\phi(x,v)\overline{\varphi_i(x)}\dd\pi_0(v)\dd\rho(x)\right
|^2 \\ 
\notag&= \EE_{y\sim\pi_0}\left|\sum_{j \in \ZZ}r^{|j|}\mu_j^{\half} e^{2\pi \ii jy} \int_{\cX}\overline{u_j(x)\varphi_i(x)} \dd \rho(x) \right|^2 \\
\notag&= \sum_{j \in \ZZ}r^{2|j|}\mu_j \left|\int_{\cX}u_j(x)\varphi_i(x)\dd\rho(x)\right|^2 \\
&= \int_{\cX}\int_{\cX}k_r(x,x')\varphi_i(x)\overline{\varphi_i(x')}\dd \rho(x)\dd\rho(x'),
       \end{align}
   where $k_r(x,x') = \sum_{j \in \ZZ} r^{2|j|}\mu_j u_j(x)\overline{u_j(x')}$. 

   Combining \eqref{eqn: qqq4}, \eqref{eqn:5111}, and \eqref{eqn:5112}, we obtain
   \begin{align}\label{eqn: qqq5}
       \notag&\sup_{\|f\|_{\cF_{p,\pi_0}} \le 1}\inf_{c_1,\dots,c_n}\left\|f - \sum_{i=1}^nc_i\varphi_i\right\|_{\rho}^2  \\
       \notag&\gtrsim(1-r)^{\frac{2}{p'}}\left[\sum_{j \in \ZZ}r^{2|j|}\mu_j - \sum_{j=1}^n\int_{\cX}\int_{\cX}k_r(x,x')\varphi_i(x)\overline{\varphi_i(x')}\dd \rho(x)\dd\rho(x')\right] \\
       &\gtrsim (1-r)^{\frac{2}{p'}}\sum_{|j| > \lfloor\frac{n}{2}\rfloor} r^{2|j|}\mu_j
   \end{align}
   where the last inequality is due to 
   \[
   \sum_{j=1}^n\int_{\cX}\int_{\cX}k_r(x,x')\varphi_i(x)\overline{\varphi_i(x')}\dd \rho(x)\dd\rho(x') \le \sum_{|j| \le  \lfloor\frac{n}{2}\rfloor} r^{2|j|}\mu_j,
   \] 
   which performs PCA for the covariance operator $T_{k_r}$, defined as $T_{k_r} f = \int k_r(x,x')f(x')\dd\rho(x)$. Notably, the maximum is obtained when $\{\phi_i\}_{i=1}^n$ are the leading eigenfunctions. 

   Choosing $r = 1-\frac{1}{n}$ and $n$ sufficiently large, then we have
   \begin{align*}
       \sup_{\|f\|_{\cF_{p,\rho}} \le 1}\inf_{c_1,\dots,c_n}\left\|f - \sum_{i=1}^nc_i\varphi_i\right\|_\rho^2 &\gtrsim n^{-\frac{2}{p'}}
       \sum_{ \lfloor\frac{n}{2}\rfloor < |j| \le n }r^{2|j|} 
       j^{-\beta-1}\\ 
       &\gtrsim n^{-\frac{2}{p'}} n \left(1-\frac{1}{n}\right)^{2n} n^{-\beta-1}\gtrsim n^{ - \frac{2}{p'}-\beta}.
   \end{align*}

\paragraph*{Proof of the upper bound \eqref{Fp-approximation-upper-bound}.}
Let $\tilde{\cH}$ be the RKHS on $\cV$ with the kernel $\tk_{\rho}:\TT\times\TT\mapsto\CC$, which coincides with $\tilde{\cF}_{2,\rho}$ by \eqref{eqn: x1}.
By the duality \eqref{dual-eq-4} and a similar argument with inequality \eqref{eqn: l_infty_estimation_noiseless} and Lemma \ref{lem_l2_noiseless}, we have that with probability $1-\delta$,
\begin{align*}
    &\sup_{\|g\|_{\tilde{\cH}}\le 1, g(x_i) = 0}\|g\|_{\pi_0} =  \sup_{\|f\|_{\cF_{2,\pi_0}} \le 1}\inf_{c_1,\dots,c_n}\|f - \sum_{i=1}^nc_i\phi(\cdot,v_i)\|_{\rho} \lesssim n^{-\frac{1}{2}-\frac{\beta}{2}}(\log(n/\delta))^{-\frac{\beta+1}{2}}, \\
     &\sup_{\|g\|_{\tilde{\cH}}\le 1, g(x_i) = 0}\|g\|_{\infty,\pi_0} =  \sup_{\|f\|_{\cF_{1,\pi_0}} \le 1}\inf_{c_1,\dots,c_n}\|f - \sum_{i=1}^nc_i\phi(\cdot,v_i)\|_{\rho} \lesssim n^{-\frac{\beta}{2}}(\log(n/\delta))^{-\frac{\beta+1}{2}}.
\end{align*}
Use the duality \eqref{dual-eq-4} with H\"older inequality, we have that with probability $1-\delta$,
\begin{align*}
    \sup_{\|f\|_{\cF_{p,\pi_0}} \le 1}\inf_{c_1,\dots,c_n}\|f - \sum_{i=1}^nc_i\phi(\cdot,v_i)\|_{\rho} &= \sup_{\|g\|_{\tilde{\cH}}\le 1, g(x_i) = 0}\|g\|_{p',\pi_0} \\
    &\lesssim \sup_{\|g\|_{\tilde{\cH}}\le 1, g(x_i) = 0}(\|g\|_{\infty,\pi_0})^{1-\frac{2}{p'}}(\|g\|_{2,\pi_0})^{\frac{2}{p'}} \\
    &\lesssim \left(\sup_{\|g\|_{\tilde{\cH}} \le 1, g(x_i) = 0}\|g\|_{\infty,\pi_0}\right)^{1-\frac{2}{p'}} \left(\sup_{\|g\|_{\tilde{\cH}} \le 1, g(x_i) = 0}\|g\|_{\pi_0}\right)^{\frac{2}{p'}} \\
    &\lesssim n^{-\frac{1}{p'}-\frac{\beta}{2}}(\log(n/\delta))^{-\frac{\beta+1}{2}}.
\end{align*}
\qed

\subsubsection{Special feature functions that satisfy  Assumption~\ref{as1}}
\label{sec: special-feature-construction}

\indent\textbf{Case 1.} For $q=2$, we have for the feature function~\eqref{eqn:svd2} that
\begin{align*}
\sup_{v \in \TT}\|\phi(\cdot,v)\|^2_\rho &= \sup_{v\in\TT}\sum_{i,j\in\ZZ}\mu_i^{1/2}\mu_j^{1/2}e_i(v)\overline{e_j(v)}\langle u_i,u_j\rangle_\rho=\\ 
&=\sup_{v \in \TT}\sum_{j \in \ZZ} \mu_j |e_j(v)|^2 = \sum_{j \in \ZZ} \mu_j:=\kappa <\infty,
\end{align*}
indicating that the feature~\eqref{eqn:svd2} satisfies Assumption \ref{as1}(i) with $M_q = \sqrt{\kappa}$ for $q=2$. As for Assumption~\eqref{as1}(ii), it follows from Lemma~\ref{rad-lemma} that it holds with $R_q = \sqrt{2\kappa}$. 

\textbf{Case 2.}
For $q\in (2,\infty)$, recall that we consider the special feature where
 $\cX = \ZZ$, $\rho(x=j) = \mu_j/\kappa$ and $u_j(x) = \sqrt{\kappa/\mu_j}\delta_{j,x}$. Hence, $\phi(x,v) = \sqrt{\kappa} e_x(v)$, for which
 \[
    \|\phi(\cdot,v)\|_{q,\rho} = \left(\sum_{j\in\ZZ}\frac{\mu_j}{\kappa} |\sqrt{\kappa}e_j(v)|^q\right)^{1/q} = \sqrt{\kappa}.
 \]
This implies that Assumption~\ref{as1}(i) is satisfied with $M_q=\sqrt{\kappa}$ for all $q\in [2,\infty)$. As for Assumption~\eqref{as1}(ii), it follows from Lemma~\ref{rad-lemma} that it holds with $R_q = \sqrt{2\kappa}$.

\textbf{Case 3.} For $q = +\infty$, set $\cX = \ZZ$, $\rho(x=j) = \tilde{\mu}_j/\tilde{\kappa}$ and $\tilde{e}_j(x) = \sqrt{\tilde{\kappa}/\tilde{\mu}_j}\delta_{j,x}$, where $\tilde{\mu}_j = (|j|+1)^{\frac{\beta}{2}} \mu_j \asymp (|j|+1)^{-1-\frac{\beta}{2}}$ and $\tilde{\kappa} = \sum_{j \in \ZZ}\tilde{\mu}_j$. Then $\phi(x,v) = \sqrt{\tilde{\kappa}}(|x|+1)^{-\frac{\beta}{2}} e_x(v)$. Hence,
$    \sup_{v \in \TT}\|\phi(\cdot,v)\|_\infty = \sqrt{\tilde{\kappa}}$.
 Moreover,
\begin{align*}
    \rad(\tilde{\cF}_{1,\rho}(1)) &= \max_{v_1,\dots,v_n}\EE_{\xi_1,\dots,\xi_n}\left[\sup_{g\in\tilde{\cF}_{1,\rho}(1)}\fn\sumin g(v_i)\xi_i\right] \\
    &= \max_{v_1,\dots,v_n}\EE_{\xi_1,\dots,\xi_n}\left[\sup_{\|b\|_{1,\rho}\le 1}\int_{\ZZ}\left(\fn\sumin \phi(x,v_i)\xi_i\right)b(x)\dd\rho(x)\right] \\
    &=\sqrt{\tilde{\kappa}}\max_{v_1,\dots,v_n}\EE_{\xi_1,\dots,\xi_n}\sup_{j\in \ZZ}\fn\left|\sum_{i=1}^n(|j|+1)^{-\frac{\beta}{2}} e_j(v_i)\xi_i\right| \\
    &\le 2\sqrt{\tilde{\kappa}}\max_{v_1,\dots,v_n}\EE_{\xi_1,\dots,\xi_n}\sup_{j\in \ZZ}\fn\sum_{i=1}^n(|j|+1)^{-\frac{\beta}{2}} e_j(v_i)\xi_i.
\end{align*}
Let $\cG = \{(|j|+1)^{-\frac{\beta}{2}} e_j(v)\}_{j \in \ZZ}$, then $\rad(\tilde{\cF}_{1,\rho}(1)) \le 2 \sqrt{\tilde{\kappa}}\rad (\cG)$. We will then bound the $\rad(\cG)$ by the Dudley's chain rule (see, e.g.~\cite[Lemma~27.4]{shalev2014understanding}). 
\begin{align*}
    \rad(\cG) &\le \sum_{k = 1}^{\infty}\max_{v_1,\dots,v_n}\EE_{\xi_1,\dots,\xi_n}\sup_{|j|\in [2^{k-1} - 1,2^k - 1)}\fn\sum_{i=1}^n(|j|+1)^{-\frac{\beta}{2}} e_j(v_i)\xi_i \\
    &\le \sum_{k = 1}^{\infty}2^{-\frac{\beta}{2}(k-1)}\max_{v_1,\dots,v_n}\EE_{\xi_1,\dots,\xi_n}\sup_{|j|\in [2^{k-1} - 1,2^k - 1)}\fn\sum_{i=1}^n e_j(v_i)\xi_i \\
    &\stackrel{(i)}{\le}\sum_{k=1}^{\infty}2^{-\frac{\beta}{2}(k-1)}\sqrt{\frac{2\log(2^{k-1})}{n}}\stackrel{(ii)}{\lesssim} n^{-\half}
\end{align*}
where $(i)$ follows Massart's lemma (see, e.g.~\cite[Lemma~26.8]{shalev2014understanding}) and $(ii)$ follows that $\sum_{k=1}^{\infty} 2^{-\frac{\beta}{2}(k-1)}\sqrt{2(k-1)\log 2} < \infty$. Therefore, $\phi$ satisfies Assumption \ref{as1}  with $M_\infty = \sqrt{\tilde{\kappa}}$ and $R_\infty \lesssim \sqrt{\tilde{\kappa}}$.
\qed

\subsection{Proof of Theorem \ref{Fp-estimate}}
\label{sec: proof-learning-err of RFMs}

\begin{lemma} \label{rad-linear}
 Consider the class of RFMs under the $\ell_p$ coefficient constraint:
\[
\hat{\cF}_{p,m}(C) :=  \left\{\frac{1}{m}\sum_{j=1}^m c_j \phi(\cdot,v_j): \|c\|_p \leq {C}m^{1/p} \right\}. 
\]
Under Assumption \ref{as1}, for any $x_1,\cdots,x_n \in \cX$, the  Rademacher complexity and Gaussian complexity of $\hat{\cF}_{p,m}$ are bounded by
\begin{align*}
 &   \rad(\hat{\cF}_{p,m}(C)) := \sup_{x_1,\dots,x_n}\EE_{\xi \sim \mathrm{Unif}(\{\pm 1\}^n)} \left[ \sup_{f \in \hat{\cF}_{p,m}(C)} \frac{1}{n}\left|\sum_{i=1}^n \xi_i f(x_i) \right|\right] \lesssim CM\sqrt{\frac{p'}{n}} \\
 &   \cG_n(\hat{\cF}_{p,m}(C)) := \sup_{x_1,\dots,x_n}\EE_{z \sim \cN(0,I_n)} \left[ \sup_{f \in \hat{\cF}_{p,m}(C)} \frac{1}{n}\left|\sum_{i=1}^n z_i f(x_i) \right|\right] \lesssim CM\sqrt{\frac{p'}{n}}
\end{align*}
\end{lemma}
\begin{proof}
Let $\eta_1,\dots,\eta_n$ be \iid Rademacher or standard Gaussian random variables. For $i\in [n]$, let
$$
y_i = (\phi(x_i,v_1),\cdots,\phi(x_i,v_m))^\top\in\RR^m.
$$
Then the Rademacher/Gaussian complexity is bounded by 
\begin{align*}
    \frac{1}{n}\EE_{\eta} \sup_{\|c\|_p \leq r }\left[\sum_{i=1}^n   \eta_i c^\top y_i   \right] & =   \frac{r}{n}\EE_{\eta} \sup_{\|c\|_p \leq 1} c^\top \left(\sum_{i=1}^n   \eta_i  y_i \right)  \\
    & = \frac{r}{n}\EE_\eta \left\|\sum_{i=1}^n \eta_i y_i \right\|_{p'} \\
    & \stackrel{(a)}{\leq} \frac{r}{n} \left(\sum_{j=1}^m \EE_\eta \left|\sum_{i=1}^n \eta_i y_{ij}\right|^{p'}\right)^{1/p'} \\
    &\stackrel{(b)}{\lesssim}   \frac{r\sqrt{p'}}{n} \left[ \sum_{j=1}^m \left(\sum_{i=1}^n y_{ij}^2\right)^{p'/2}\right]^{1/p'},
\end{align*}
where $(a)$ and $(b)$ follow from Jensen's and  Khintchine inequalities (Lemma \ref{lemma: khintchin}), respectively. Notice that $|y_{ij}| \leq M$ for any $1\leq i\leq n$ and $1\leq j \leq m$, we have 
\begin{align*}
     \frac{r\sqrt{p'}}{n} \left[ \sum_{j=1}^m  \left(\sum_{i=1}^n y_{ij}^2\right)^{p'/2} \right]^{1/p'} \leq \frac{rM\sqrt{p'}m^{1/p'}}{\sqrt{n}}.
\end{align*}
Taking $r = \frac{C}{m^{1/p'}}$ we complete the proof.
\end{proof}

Given the approximation bound in Theorem \ref{thm: approximation-bound} and the Rademacher/Gaussian complexity bound in Lemma \ref{rad-linear}, the proof of Theorem~\ref{Fp-estimate} is standard (see, e.g., \cite[Chapter 13.3]{wainwright2019high}). 
\begin{lemma}  \label{empirical-app}
Suppose that $\sup_{x\in \cX,\,v\in \cV} |\phi(x,v) | \leq M $ holds, for any $x_1,\cdots,x_n \in \cX$, let $\hat{\rho}_n = \frac{1}{n}\sum_{i=1}^n \delta_{x_i}$ be the empirical measure. Then, \wp  at least $1-\delta$ over the sampling of $\{v_j\}_{j=1}^m$, there exists coefficients $\tc \in \RR^m$ with $\|\tc\|_p \lesssim m^{1/p}$ such that the RFM $\tilde{f} = \frac{1}{m}\sum_{j=1}^m \tilde{c}_j \phi(\cdot,v_j)$ satisfies 
\begin{align*}
    \|\tilde{f} - f^* \|_{L^q(\hat{\rho}_n)} & \lesssim \left(\frac{M^{p'} \log^3 m + M\log(1/\delta)}{m} \right)^{2/p'}.
\end{align*}
\end{lemma}
\begin{proof}
 The assumption $\sup_{x\in \cX,v \in \cV} |\phi(x,v)|\leq M $ implies that Assumption \ref{as1} holds for $\rho = \hat{\rho}_n$ with $M_q \lesssim M,\,R_q \lesssim M$. Then, the result is directly obtained by applying Theorem \ref{thm: approximation-bound} to $\rho = \hat{\rho}_n$.
\end{proof}

\begin{lemma} \label{noisy-bound}
Consider the RFM class
\begin{align*}
    \hat{\cF}_{p,m}(C) := \left\{\frac{1}{m}\sum_{j=1}^m c_j\phi(\cdot,v_j): \|c\|_p \leq {C}{m^{1/p}} \right\},
\end{align*}
Let $\{\xi_i\}_{i=1}^n$ be independent noise distributed as $\cN(0,\sigma^2)$. Since $\sup_{x\in \cX,\,v\in \cV}|\phi(x,v)|\leq 1$, for any $x_1,\cdots,x_n \in \cX $ and $v_1,\cdots,v_m \in \cV $. Then, with probability at least $1-\delta$ over $\{\xi_i\}_{i=1}^n$, we have
\begin{align*}
    \sup_{f \in \cF_{p,m}(C)} \frac{1}{n}\sum_{i=1}^n \xi_if(x_i) \lesssim CM\sigma\sqrt{\frac{p' + \log(1/\delta)}{n}}.
\end{align*}
\end{lemma}
\begin{proof}
Note that any function in $\hat{\cF}_{p,m}(C)$ is uniformly bounded by $C_1M$. Thus $(z_1,\cdots,z_n) \mapsto \frac{1}{n}\sum_{i=1}^n \xi_if(\xi_i)$ is $\frac{CM}{\sqrt{n}}$-Lipschitz. The proof is completed by Gaussian concentration inequality (see, e.g., \cite[Theorem 5.2.2]{vershynin2019high}) and the Gaussian complexity bound stated in Lemma \ref{rad-linear}.
\end{proof}

\paragraph*{Proof of Theorem \ref{Fp-estimate}}
From Lemma \ref{empirical-app}, with probability at least $1-\delta$, there exists a RFM $\tilde{f} = \frac{1}{m}\sum_{j=1}^m \tilde{c}_j \phi(\cdot,v_j)$ such that $\|\tc\|_p \leq C_1m^{1/p}$ and 
\begin{align} \label{emapp-eq}
    \frac{1}{n}\sumin\left(\tilde{f}(x_i) - f^*(x_i) \right)^2 \lesssim  \left(\frac{M^{p'}\log^3 m + M \log (1/\delta)}{m}\right)^{2/p'}
\end{align}
where $C_1$ is an absolute constant. 
We choose the constraint coefficient as $R = C_1m^{1/p} $.
Since $\hat{f} $ is the minimizer of the empirical loss, we have
\begin{align} \label{opt}
    \frac{1}{n}\sum_{i=1}^n \left(f^*(x_i) + \xi_i - \hat{f}(x_i) \right)^2 \leq   \frac{1}{n}\sum_{i=1}^n \left(f^*(x_i) + \xi_i - \tilde{f}(x_i) \right)^2.
\end{align}
Equivalently, \eqref{opt} can be rewritten as
\begin{align} \label{333}
    \frac{1}{n}\sum_{i=1}^n \left(\hat{f}(x_i) - {f^*}(x_i)\right)^2 \leq \frac{1}{n}\sum_{i=1}^n \left(\tilde{f}(x_i) - {f^*}(x_i)\right)^2 +  \frac{2}{n}\sum_{i=1}^n \xi_i\left(\hat{f}(x_i) - \tilde{f}(x_i)\right).
\end{align}
By applying Lemma \ref{noisy-bound} to bound the second term in the right-hand side of \eqref{333} and combining \eqref{emapp-eq} we arrive at
\begin{align*}
    \frac{1}{n} \sum_{i=1}^n \left(\hat{f}(x_i) - f^*(x_i)\right)^2 \lesssim M\sigma\sqrt{\frac{p'+\log(1/\delta)}{n}} + \left(\frac{M^{p'} \log^3 m + M\log(1/\delta)}{m} \right)^{2/p'}.
\end{align*}
Finally, note that $\|\tc \|_p \lesssim m^{1/p}$, we combine Lemma \ref{localRad} and the Rademacher complexity bound of the RFM class stated in Lemma \ref{rad-linear} to complete the proof.
\qed

\section{Proofs in Section \ref{sec: uniform-estimation}}
\label{app: uniform-estimation}

\subsection{Proof of Theorem \ref{primal-bound}}
\label{sec: proof-lowerbound-rkhs}

Our proof needs the following two lemmas, which characterize the approximation properties of a RKHS and Barron space using the eigenvalue decay of the associated kernel.
\begin{lemma}\label{kernel-kol}
Given a kernel $k:\cX\times\cX$ satisfying $\int k(x,x)\dd\gamma(x)<\infty$, let $\{(\mu_j,e_j)\}_{j\in \NN}$ be the eigen-pairs  \wrt  $\gamma\in\cP(\cX)$. Then, we have
\begin{align*}
     \sup_{\|h\|_{\cH_k} \leq 1} \inf_{c\in\RR^n} \left\|h - \sum_{i=1}^n c_i e_i \right\|_{\gamma} \leq \sqrt{\mu_{n+1}}.
\end{align*}    
\end{lemma}
\begin{proof}
   Since $e_1,\cdots,e_n$ are orthonormal in $L^2(\gamma)$, for any $h \in \cH_k$, the optimal approximation is given by $c_i = \langle h,e_i \rangle_\gamma$ and 
   \begin{align}\label{4567}
       \inf_{c_1,\cdots,c_n}\left\|h - \sum_{i=1}^n c_ie_i \right\|_\gamma^2 & = \left\|h - \sum_{i=1}^n \langle h,e_i  \rangle_\gamma  e_i \right\|_\gamma^2   = \sum_{j=n+1}^{\infty} \langle h, e_j\rangle_\gamma^2. 
   \end{align}
  By the definition of RKHS norm, we have 
$
 \|h\|_{\cH_{k}}^2=\sum_{j=1}^{\infty} \mu_j^{-1} \langle h,\phi_j \rangle_\gamma^2  \leq 1.
$
Thus,   the right-hand side of \eqref{4567} is bounded by
  \begin{align*}
      \sum_{j=n+1}^{\infty} \langle h, e_j\rangle_\gamma^2 \leq \mu_{n+1} \sum_{j=n+1}^{\infty}\frac{1}{\mu_j} \langle h, e_j\rangle_\gamma^2\leq\mu_{n+1}\sum_{j=1}^{\infty} \frac{1}{\mu_j} \langle h, e_j\rangle_\pi^2\leq  {\mu_{n+1}}.
  \end{align*}
Thus, 
  we complete the proof.
\end{proof}

\begin{lemma}(A restatement of \cite[Proposition 1]{Wu2021ASA}) \label{barron-kol}
Given a feature function $\phi:\cX\times \cX\mapsto\RR$, consider the Barron space
$$\cB(1) = \left\{\int_\cX \phi(x, \cdot) \d \mu(x) : \|\mu\|_{\mathrm{TV}} \le 1 \right\}. $$
For any $\gamma\in \cP(\cX)$, define the kernel $k(x,x')=\int_{\cX}\phi(x,v)\phi(x',v)\dd\gamma(v)$. We have
\begin{align*}
    \inf_{\varphi_1,\dots,\varphi_n\in L^2(\gamma)}\sup_{g \in \cB(1)} \inf_{c_1,\cdots,c_n \in \RR}\left\|g   - \sum_{i=1}^n c_i\varphi_i\right\|_\gamma \geq \sqrt{\Lambda_{k,\gamma}(n)}.
\end{align*}
\end{lemma}

\paragraph*{Proof of Theorem \ref{primal-bound}.}
Without loss of generality, for any $\gamma \in \cP(\cX)$, we will assume that $\gamma$ has full support. If this is not the case, we can replace $\cX$ with $\mathrm{supp}(\gamma)$, which does not increase the minimax error:
\begin{equation*}
         \inf_{T \in \cA_n} \sup_{\|f\|_{\cH_k} \leq 1}  \|T(\{(x_i,f(x_i))\}_{i=1}^n) -f  \|_{C_0(\cX)} \geq     \inf_{T \in \cA_n} \sup_{\|f\|_{\cH_k} \leq 1}  \|T(\{(x_i,f(x_i))\}_{i=1}^n) -f  \|_{C_0(\mathrm{supp}(\gamma))}.
\end{equation*}
By Lemma \ref{lemma: F2-rkhs}, for any distribution $\gamma \in \cP(\cX)$, there exists $\phi: \cX \times \cX \to \RR$ such that the kernel $k(\cdot,\cdot)$ admits the integral representation
\begin{align}\label{gamma-integral} k(x,x') = \int_\cX \phi(x,v)\phi(x',v) \d \gamma(v). 
\end{align}
% In the following proof, we first fix the distribution $\gamma$ and feature function $\phi$ and then take supremum over all the choice of $\gamma$. 

\underline{\textbf{Part I: the noiseless case.}} 
By Proposition \ref{lemma-data-dependent}:
\begin{align*}
     \inf_{T \in \cA_n} \sup_{\|f\|_{\cH_k} \leq 1}  \|T(\{(x_i,f(x_i))\}_{i=1}^n) -f  \|_{C_0(\cX)}
     &\asymp \sup_{\|f\|_{\cH_k} \leq 1,\,f(x_i) = {0}} \|f\|_{C_0(\cX)},
\end{align*}
By employing the integral structure $\cH_k  = \cF_{2,\gamma}$ and applying the dual equivalence  \eqref{dual-eq-6} in Theorem \ref{thm: dual}, we get 

\begin{align*}
     \sup_{\|f\|_{\cH_k} \leq 1,\,f(x_i) = {0}} \|f\|_{C_0(\cX)}& = \sup_{\|\mu\|_{\mathrm{TV}} \le 1} \inf_{c_1,\cdots,c_n}\left\|\int_\cX \phi(x,\cdot) \d \mu(x)  - \sum_{i=1}^n c_i\phi(x_i,\cdot)\right\|_\gamma \\
     & \geq \inf_{\varphi_1,\cdots,\varphi_n \in L^2(\pi)}\sup_{\|\mu\|_{\mathrm{TV}} \le 1} \inf_{c_1,\cdots,c_n \in\RR} \left\|\int_\cX \phi(x,\cdot) \d \mu(x)   - \sum_{i=1}^n c_i \varphi_i \right\|_{\gamma}.
\end{align*}
Finally, applying Lemma \ref{barron-kol} and using the arbitrariness of the choice of $\gamma$, we get 
$$  \sup_{\|f\|_{\cH_k} \leq 1,\,f(x_i) = {0}} \|f\|_{C_0(\cX)} \geq \sup_{\gamma \in \cP(\cX)} \sqrt{\Lambda_{k,\gamma}(n)}.$$ We complete the proof.
\qed

\vspace*{1em}
\underline{\textbf{Part II: the noisy case.}} By Proposition \ref{lemma-random-data}, the minimax error is lower bounded by $\II_{L^2(\brho),C_0(\cX)}\left(\cH_k,\frac{\sigma}{\sqrt{n}}\right)$. Then, we complete the proof by using the following lemma.
\begin{lemma}\label{lemma: xyz}
$$
    \II_{L^2(\brho),C_0(\cX)}\left(\cH_k,\frac{\sigma}{\sqrt{n}}\right)\geq \sup_{\gamma \in \cP(\cX)}\sqrt{\Lambda_{k,\gamma}(n_{\brho}(\sigma))}.
$$
\end{lemma}
\begin{proof}
By definition, 
\begin{align} \label{Linfty-minimax}
      \II_{L^2(\brho),C_0(\cX)}\left(\cH_k,\frac{\sigma}{\sqrt{n}}\right) =  \sup_{\|f\|_{\cH_k} \leq 1,\,\|f\|_{\brho}\leq \sigma n^{-1/2}} \|f\|_{C_0(\cX)}
\end{align} 
We now provide a lower bound for the quantity \eqref{Linfty-minimax}. By employing the integral structure $\cH_k = \cF_{2,\gamma}$ \eqref{gamma-integral}and using the dual equivalence \eqref{dual-eq-3}  in Theorem \ref{thm: dual},  we have for any $\epsilon >0$ that
\begin{align}\label{5678}
\begin{aligned}
  \sup_{\|f\|_{\cH_k} \leq 1,\,\|f\|_{\brho}\leq \epsilon} \|f\|_{C_0(\cX)}=      \sup_{\|g\|_{\tilde{\cB}}\leq 1}\,\,\inf_{h \in \tilde{\cF}_{2,\brho}} \left(\left\|g - h \right\|_{\gamma} + \epsilon \|h\|_{\tilde{\cF}_{2,\brho}}\right),
 \end{aligned}
\end{align}
where $\tilde{\cB}$ and $\tilde{\cF}_{2,\brho}$ are the conjugate spaces:
\begin{align*}
& \tilde{\cB}(1)  = \left\{\int_\cX \phi(x,\cdot) \d \mu(x): \, \|\mu\|_{\mathrm{TV}} \le 1  \right\},   \\
&  \tilde{\cF}_{2,\brho}(1) = \left\{\int_\cX a(x)\phi(x,\cdot) \d \brho(x): \, \|a\|_{\brho} \le 1  \right\}.
\end{align*}
Note that  $\tilde{\cF}_{2,\brho}$ is the RKHS associate with the kernel $$\tilde{k}_{\brho}(v,v')  =\int_{\cX} \phi(x,v)\phi(x,v') \d \bar{\rho}(x). $$
Let $\{(\tilde{\mu}_j,\tilde{e}_j\}_{j\geq 1})\}_{j=1}^\infty$ be the eigen-pairs of $\tk_\brho(\cdot,\cdot)$ with respect to the distribution $\gamma$. 
Then, for any $g\in \tilde{\cB}, h\in \tilde{\cF}_{2,\brho}$ and $m \in \NN^+$, we have for any $c_1,\dots, c_m\in\RR$ that
\begin{align}\label{eqn: 020}
\notag \|g-h\|_\gamma +\epsilon \|h\|_{\tilde{\cF}_{2,\brho}} &\geq  \left\|g-\sum_{j=1}^m c_j \tilde{e}_j\right\|_\gamma - \left\|h-\sum_{j=1}^m c_i \tilde{e}_j\right\|_\gamma + \epsilon \|h\|_{\tilde{\cF}_{2,\brho}}\\ 
&\geq \left\|g-\sum_{j=1}^m c_i \tilde{e}_j\right\|_\gamma - \sqrt{\tilde{\mu}_{m+1}}\|h\|_{\tilde{\cF}_{2,\brho}} + \epsilon \|h\|_{\tilde{\cF}_{2,\brho}},
\end{align}
where the second step follows from Lemma \ref{kernel-kol}. Thus, taking $\epsilon=\sqrt{\tilde{\mu}_{m+1}}$, we have 
\begin{align}\label{eqn: xx1}
\notag \sup_{\|f\|_{\cH_k} \leq 1,\,\|f\|_{\brho}\leq  \sqrt{\tilde{\mu}_{m+1}}} \|f\|_{C_0(\cX)} &= \sup_{\|g\|_{\tilde{\cB}}\leq 1}\,\,\inf_{h \in \tilde{\cF}_{2,\brho}}\left( \left\|g - h \right\|_{\pi} +\sqrt{\tilde{\mu}_{m+1}} \|h\|_{\tilde{\cF}_{2,\brho}}\right)\\ 
\notag &\geq \sup_{\|g\|_{\tilde{\cB}}\leq 1}\,\,\inf_{c_1,\dots,c_m\in\RR}\left\|g-\sum_{j=1}^m c_i\tilde{e}_j\right\|_\pi\\ 
&\geq \sqrt{\Lambda_{k,\gamma}(m)},
\end{align}
where the second step follows from \eqref{eqn: 020} and the third step uses Lemma \ref{barron-kol}. By choosing $m = n_{\brho}(\sigma) =\inf \{m: n\mu^{(k,\brho)}_{m+1} \le \sigma^2\}$ and taking the supremum over $\gamma\in\cP(\cX)$, we complete the proof. 
\end{proof}

\subsection{Proof of Theorem \ref{Lq-estimation-rkhs}}\label{sec:proof_2}
\label{sec: proof-upperbound-rkhs}

We provide proofs for the noiseless and noisy cases separately.
\subsubsection{The noiseless case}
Consider the integral representation of $k(\cdot,\cdot)$ \wrt $\rho$ as in Lemma \ref{lemma: F2-rkhs}:
$$ k(x,x') = \int_\cX \phi(x,v) \phi(x',v) \dd \rho(v),\, \phi(x,v) = \sum_{j=1}^\infty \mu_j^\half e_j(x)e_j(v). $$
Following the same argument of \eqref{eqn: 001}, we have that
\begin{equation*}
    \|\hat{f} - f^*\|_{C_0(\cX)} \le 2 \II_{L^\infty(\hat{\rho}_n), C_0(\cX)}(\cH_k(1), 0).
\end{equation*}
Combining this with Theorem \ref{thm: dual-simple}, we have that,
\begin{align*}
    \|\hat{f} - f^*\|_{C_0(\cX)} \le \sup_{ f\in \tilde{\cB}(1)}\inf_{c_1,\dots,c_n}\left\|f - \sum_{i=1}^n c_i\phi(x_i,\cdot)\right\|_\rho.
\end{align*}
The remainder of this part is to derive a random feature approximation error bound for the space $\tilde{\cB}(1)$. That is, with probability $1-\delta$,
\begin{align}\label{eqn: l_infty_estimation_noiseless}
    & \sup_{ f\in \tilde{\cB}(1)}\inf_{c_1,\dots,c_n}\left\|f - \sum_{i=1}^n c_i\phi(x_i,\cdot)\right\|_\rho \lesssim n^{-\frac{\beta}{2}}(\log(n/\delta))^{-\frac{\beta+1}{2}}.
\end{align}
For this, we first need the following lemma to characterize the approximation property of the $\cF_2$ space.
\begin{lemma}\label{lem_l2_noiseless}
    With probability $1-\delta$,
    \begin{equation*}
        \sup_{ f\in \tilde{\cF}_{2,\rho}(1)}\inf_{n^{\frac{1}{2}}\|c\|_2 \le 2}\left\|f - \sum_{i=1}^n c_i\phi(x_i,\cdot)\right\|_\rho \lesssim \left(\frac{\log(n/\delta)}{n}\right)^{-\frac{\beta+1}{2}}  .
    \end{equation*}
\end{lemma}
\begin{proof}
    This is almost a restatement of Proposition 1 in \cites{bach2017equivalence}. We need only estimate $d_{\max}(\lambda) = \sup_{x \in \cX} \langle \phi(x,\cdot), (\cT_k+ \lambda I)^{-1}\phi(x,\cdot)\rangle_\rho$ in this proposition. Noticing that $\phi(x,\cdot) = \sum_{i=1}^{\infty} \sqrt{\mu_i} e_i(x) e_i$ and $(\cT_k + \lambda I)^{-1} e_i = (\mu_i+ \lambda)^{-1} e_i$, we have that
    \begin{align*}
        d_{\max}(\lambda) &= \sup_{x \in \cX}\left\langle \sum_{i=1}^{\infty} \sqrt{\mu_i} e_i(x) e_i, (\cT_k + \lambda I)^{-1}\sum_{i=1}^{\infty} \sqrt{\mu_i} e_i(x) e_i\right\rangle_\rho \\
        &=\sup_{x \in \cX}\left\langle \sum_{i=1}^{\infty} \sqrt{\mu_i} e_i(x) e_i, \sum_{i=1}^{\infty} \sqrt{\mu_i} (\mu_i + \lambda )^{-1}e_i(x) e_i\right\rangle_\rho \\
        &=\sup_{x \in \cX} \sum_{i=1}^{\infty} \frac{\mu_i}{\mu_i+\lambda}|e_i(x)|^2 
        \le M^2 \sum_{i=1}^{\infty} \frac{\mu_i}{\mu_i + \lambda},
    \end{align*}
    where the last inequality uses Assumption \ref{assu: kernel}. Then, noticing that $\mu_i \asymp i^{-\beta-1}$, we have that
    \begin{equation*}
        d_{\max}(\lambda) \lesssim \left[\sum_{i \le \lambda^{-\frac{1}{\beta+1}}} 1 + \frac{1}{\lambda}\sum_{i > \lambda^{-\frac{1}{\beta+1}}} i^{-\beta-1} \right] \lesssim \lambda^{-\frac{1}{\beta+1}}.
    \end{equation*}
    Hence, there exists $C > 0$ such that when $\lambda = C (\frac{\log(n/\delta)}{n})^{-\beta-1}$, we have that $d_{\max}(\lambda)\le  \frac{n}{16\log(n/\delta)}$, which means that
    \begin{equation*}
        5d_{\max}(\lambda) \log(16d_{\max}(\lambda)/\delta) \le \frac{n}{\log(n/\delta)}\log(n/\delta) \le n.
    \end{equation*}
    Then, we can apply Proposition 1 in \cites{bach2017equivalence} to obtain our result.
\end{proof}
\paragraph*{Proof of Theorem \ref{Lq-estimation-rkhs}}
First, using Lemma \ref{lemma: barron2single-neuron}, we notice that
\begin{equation}\label{eqn:11123}
    \begin{aligned}
       \sup_{ f\in \tilde{\cB}(1)}\inf_{c_1,\dots,c_n}\left\|f - \sum_{i=1}^n c_i\phi(x_i,\cdot)\right\|_\rho &=    \sup_{\|\gamma\|_{\mathrm{TV}} \le 1}\inf_{c1,\dots,c_n}\left\|\int_{\cX}\phi(x,\cdot)\dd \gamma(x) - \sum_{i=1}^n c_i\phi(x_i,\cdot)\right\|_\rho \\
        &=    \sup_{ x \in \cX}\inf_{c1,\dots,c_n}\left\|\phi(x,\cdot) - \sum_{i=1}^n c_i\phi(x_i,\cdot)\right\|_\rho. 
\end{aligned}
\end{equation}
For any $x \in \cX$, let $\tilde{\phi}_x = \sum_{i=1}^n \sqrt{\mu_i} e_i(x) e_i$. We have that 
    \begin{equation}\label{eqn:123}
        \begin{aligned}
    &\|\tilde{\phi}_x - \phi(x,\cdot)\|_ \rho^2 = \sum_{i=n+1}^{\infty} \mu_i^2|e_i(x)|^2 \lesssim \sum_{i=n+1}^{\infty} \mu_i \asymp n^{-\beta},\\
        &\|\tilde{\phi}_x\|_{\tilde{F}_{2,\rho}}^2 = \|\tilde{\phi}_x\|_{\cH_k}^2 =  \sum_{i=1}^n|e_i(x)|^2 \lesssim n.
    \end{aligned}
    \end{equation}
Therefore, combining the Lemma \ref{lem_l2_noiseless} and \eqref{eqn:123}, we have that with probability $1-\delta$.
\begin{align*}
    \inf_{c_1,\dots,c_n}\left\|\phi(x,\cdot) -\sum_{i=1}^n c_i \phi(x_i,\cdot)\right\|_\rho &\le \inf_{c_1,\dots,c_n}\left(\left\|\tilde{\phi}_x -\sum_{i=1}^n c_i \phi(x_i,\cdot)\right\|_\rho  + \left\|\tilde{\phi}_x - \phi(x,\cdot)\right\|_{\rho} \right)\\
    &\lesssim n^{\half} \sup_{ f\in \tilde{\cF}_{2,\rho}(1)}\inf_{n^{\frac{1}{2}}\|c\|_2 \le 2}\left(\left\|f - \sum_{i=1}^n c_i\phi(x_i,\cdot)\right\|_\rho + n^{-\frac{\beta}{2}}\right) \\
    &\lesssim n^{-\frac{\beta}{2}} + n^{\half}\left(\frac{\log(n/\delta)}{n}\right)^{-\frac{\beta+1}{2}} \lesssim n^{-\frac{\beta}{2}}(\log(n/\delta))^{-\frac{\beta+1}{2}}.
\end{align*}
We can then complete the proof by combining the last inequality with \eqref{eqn:11123}. \qed
%Combining the Lemma \eqref{lem_l2_noiseless} with Lemma \eqref{lemma_approximation}, we know that for any $q \in [2,{\infty})$,
%\begin{align*}
%    \sup_{ f\in \tilde{\cF}_{q',\rho}(1)}\inf_{c1,\dots,c_n}\left\|f - \sum_{i=1}^n c_i\phi(x_i,\cdot)\right\|_\rho &\lesssim n^{-\frac{1}{q} - \frac{\beta}{2}} + n^{\frac{1}{2}-\frac{1}{q}} \sup_{ f\in \tilde{\cF}_{2,\rho}(1)}\inf_{c1,\dots,c_n}\left\|f - \sum_{i=1}^n c_i\phi(x_i,\cdot)\right\|_\rho \\
%    &\lesssim n^{-\frac{\beta}{2} - \frac{1}{q}}(\log(n/\delta))^{-\frac{\beta+1}{2}}.
%\end{align*}

\subsubsection{The noisy case}\label{sec: proof-rkhs-uniform-noisy}

We first restate the result in \cite{sobolevkernel}. Consider a Mercer's kernel with the spectral decomposition given by $k(x,x')=\sum_{j=1}^\infty\mu_je_j(x)e_j(x')$. For any $s\geq 0$, define an interpolation space given by
\[
    \cH_k^s = \left\{\sum_{j=1}^\infty a_j \mu_j^{s/2} e_j:\sum_{j=1}^\infty a_j^2<\infty\right\},
\]
equiped with the norm
\[
    \left\|\sum_{j=1}^\infty a_j \mu_j^{s/2} e_j\right\|_{\cH_k^s}^2 = \sum_{j=1}^\infty a_j^2.
\]
Then, by H\"older inequality, it holds for any $0\leq s_1\leq 1\leq s_2\leq \infty$ that $$\cH_k^{s_2}\subset \cH_k^1=\cH_k\subset \cH_k^{s_1}.$$

\begin{theorem}[A restatement of Remark 3.4 in \cites{sobolevkernel}]\label{thm: sobolev}
Suppose the kernel's eigenvalue decays as $\mu_j\asymp j^{-1-\beta}$. Let $s>0$ be an index such that 
\begin{equation}\label{eqn: embedding}
\|\cH_k^s\hookrightarrow L^\infty(\rho)\|<\infty
\end{equation}
where $\|\cA\hookrightarrow \cB\|:=\sup_{\|x\|_{\cA}\leq 1}\|x\|_{\cB}$.  Then, there exists a choice of $\lambda_n$ such that the KRR estimator satisfies 
\[
    \|\hf_{\lambda_n}-f^*\|_{L^\infty(\rho)}\lesssim \log(4/\delta)\left(\frac{1}{n}\right)^{\frac{(1-s)(1+\beta)}{2(2+\beta)}}
\]
\end{theorem}
\begin{proof}[Proof of \eqref{eqn: uniform-upperbound-noisy}]
Recalling Assumption \ref{assu: kernel}, $\sup_{j\in\NN}\|e_j\|_{C_0(\cX)}\leq M$, $\rho$ is full support, and $\mu_j\asymp j^{-1-\beta}$. For any $s>0$ and $f\in \cH_k^s(1)$, by definition, there exist $a\in\ell^2$ such that $f=\sum_{j=1}^\infty a_j\mu_j^{s/2}e_j$ with $\sum_{j=1}^\infty a_j^2 = \|f\|_{\cH_k^s}^2\leq 1$. Then, for any $s>1/(1+\beta)$ it holds that
\begin{align*}
\|f\|_{C_0(\cX)} &= \left\|\sum_{j=1}^\infty a_j\mu_j^{s/2}e_j\right\|_{C_0(\cX)}\leq \sum_{j=1}^\infty |a_j|\mu_j^{s/2} \sup_{j\in\NN}\|e_j\|_{C_0(\cX)}\\ 
&\leq M \sum_{j=1}^\infty |a_j|\mu_j^{s/2}\leq M\sqrt{\sum_{j=1}^\infty a_j^2} \sqrt{\sum_{j=1}^\infty \mu_j^{s}} = M \sqrt{\sum_{j=1}^\infty j^{-(1+\beta)s}}<\infty.
\end{align*}
This means that Eq.~\eqref{eqn: embedding} holds for any $s>1/(1+\beta)$ and moreover, 
\[
	(1-s)\frac{1+\beta}{2+\beta}>\left(1-\frac{1+\delta}{1+\beta}\right)\frac{1+\beta}{2+\beta} = \frac{\beta-\delta}{2+\beta}.
\]
We thus complete the proof by applying Theorem~\ref{thm: sobolev}.
\end{proof}

\newcommand{\bee}{\hat{\ee}}
\subsection{Proof of Theorem \ref{thm:periodic_upper}}
\label{sec: proof-periodic-kernel-upper-bound}

We begin by stating the key lemmas necessary for proving our theorem with  proofs deferred to  Appendix~\ref{sec: proof-of-key-lemmas-28}.
To facilitate our analysis, we define the $L^\infty$ norm on training samples as
\[
\|g\|_{\hL^\infty_n} := \max_{1\le i \le n}|g(x_i)|.
\]

The first lemma  bounds the $I$-complexity, which  quantifies how  the empirical $L^\infty$ error can be transferred to the population $L^\infty$ error for target functions in $\cH_{k_s}$.
\begin{lemma}[$I$-complexity bound]\label{lem:linfty_generalization}
For any $\epsilon \geq 0$, we have that
    \begin{equation*}
\sup_{\|f\|_{\cH_{k_s}} \le 1, \|f\|_{\hL^\infty_n} \le\epsilon} \|f\|_{C(\TT)} \lesssim \epsilon + n^{-\frac{s(\beta+1)-1}{2}}.
    \end{equation*}
\end{lemma}
The next lemma concerns the properties of the KRR estimator $\hf_{\lambda,s}$ including the training error and norm control.

\begin{lemma}[Properties of the KRR estimator]\label{lemma: krr-hks1}
For any $\delta\in(0,1/2)$, \wp at least $1-\delta/2$ over the noise $(\xi_i)_{1\le i\le n}$ we have
\begin{align*}
    \|\hat{f}_{\lambda,s} - f^*\|_{\hL_n^\infty} &\lesssim \sigma \sqrt{\frac{\log (n/\delta)}{n}} \lambda^{-\frac{1}{2s(\beta+1)}} + \lambda^{\frac{\beta}{2s(\beta+1)}}\\ 
    \|\hat{f}_{\lambda,s}\|_{\cH_{k_s}} &\lesssim \lambda^{ \frac{1-s}{2s}} + \frac{\sigma}{\sqrt{n}}\left(\lambda^{-\half - \frac{1}{2s(\beta+1)}} + \lambda^{-\half} \sqrt{\log(1/\delta)}\right)
\end{align*}
\end{lemma}

The final lemma examines how well $\cH_{k}$ can be approximated by $\cH_{k_s}$ under the $C(\TT)$ norm.
\begin{lemma}\label{lemma: smooth-approx}
    There exists $\tilde{f}^* \in \cH_{k_s}$ such that $\|\tilde{f}^*\|_{\cH_{k_s}} \lesssim n^{\frac{s-1}{2}}$ and $\|\tilde{f}^* - f^*\|_{C(\TT)} \lesssim n^{-\frac{\beta}{2(\beta+1)}}$.
\end{lemma}

\begin{proof}[\underline{Proof of Theorem~\ref{thm:periodic_upper}}]
Let $\tf^* \in \cH_{k_s}$ be an approximation of $f^*$ satisfying the condition of Lemma~\ref{lemma: smooth-approx}:
\begin{align} \label{xxz}
\|\tilde{f}^*\|_{\cH_{k_s}} \lesssim n^{\frac{s-1}{2}}, \quad \|\tilde{f}^* - f^*\|_{C(\TT)} \lesssim n^{-\frac{\beta}{2(\beta+1)}}.
\end{align}
We have
\begin{equation}\label{eqn: periodic-2}
\|\hf_{\lambda,s}-f^*\|_{C(\TT)}\leq \|\hf_{\lambda,s}-\tf^*\|_{C(\TT)} + \|\tf^*-f^*\|_{C(\TT)} \lesssim \|\hf_{\lambda,s}-\tf^*\|_{C(\TT)} + n^{-\frac{\beta}{2(\beta+1)}}
\end{equation}Therefore, it remains to bound the first term on the RHS, which can be accomplished by employing $I$-complexity.

According to the estimates in Lemma~\ref{lemma: krr-hks1}, we 
take $\lambda \asymp (\frac{\log(n/\delta)}{n})^s$. Then,  with probability at least $1-\delta$, we have
\begin{equation}\label{eqn:8888}
    \begin{aligned}
    &\|\hat{f}_{\lambda,s} - f^*\|_{\hL^\infty_n} \lesssim \Big(\frac{\log(n/\delta)}{n}\Big)^{\frac{\beta}{2(\beta+1)}} =:\epsilon_n\\
    &\|\hat{f}_\lambda\|_{\cH_{k_s}} \lesssim n^{\frac{s-1}{2}} + n^{\frac{s-1}{2} + \frac{1}{2(\beta + 1)}} + n^{\frac{s-1}{2}}\sqrt{\frac{\log(1/\delta)}{\log^{s}(n/\delta)}} \lesssim n^{\frac{s-1}{2} + \frac{1}{2(\beta + 1)}}=:A_n.
\end{aligned}
\end{equation}
Applying the condition \eqref{xxz} once more, we obtain
\begin{equation}
\begin{aligned}
    \|\hat{f}_{\lambda,s} - \tilde{f}^*\|_{\hL^\infty_n} \lesssim \epsilon_n,\quad  
     \|\hat{f}_{\lambda,s}-\tilde{f}^*\|_{\cH_{k_s}}\lesssim A_n.
\end{aligned}
\end{equation}

Finally, by applying Lemma \ref{lem:linfty_generalization} to bound the $I$-complexity, we obtain
\begin{align}\label{eqn: periodic-3}
\notag    \|\hat{f}_\lambda - \tilde{f}^*\|_{C(\TT)} &\leq \sup_{ \|g\|_{\cH_{k_s}}\leq A_n,\|g\|_{\hL^{\infty}_n}\lesssim \epsilon_n}\|g\|_{C(\TT)}=A_n\sup_{\|g\|_{\cH_{k_s}}\leq 1, \|g\|_{\hL_n^\infty}\leq \epsilon_n/A_n}\|g\|_{C(\TT)}\\ 
\notag    &\lesssim A_n\left(\frac{\epsilon_n}{A_n}+n^{-\frac{s(\beta+1)-1}{2}}\right) \\
\notag    & \lesssim n^{\frac{s}{2}-\frac{\beta}{2(\beta+1)}}\left(n^{-\frac{s}{2}}(\log(n/\delta))^{\frac{\beta}{2(\beta+1)}} +n^{-\frac{s(\beta+1)-1}{2}} \right) \\
\notag    &= n^{-\frac{s\beta}{2} + \frac{1}{2(\beta+1)}} + \Big(\frac{\log(n/\delta)}{n}\Big)^{\frac{\beta}{2(\beta+1)}}\\ 
    &\lesssim \Big(\frac{\log(n/\delta)}{n}\Big)^{\frac{\beta}{2(\beta+1)}},
\end{align}
where the last step uses the condition $s \geq 1/\beta$.  Substituting \eqref{eqn: periodic-3} into \eqref{eqn: periodic-2} completes the proof.
\end{proof}

\subsubsection{Proofs of key lemmas}\label{sec: proof-of-key-lemmas-28}
\textbf{Kernel matrices of periodic kernels.}
Let $\omega=e^{2\pi\ii/n}$. For the uniform grid points $x_i=(i-1)/n$ for $i\in [n]$,  the corresponding ``discretized'' eigenfunctions are the Fourier modes:
\[
 \bee_j=\frac{1}{\sqrt{n}}(e_j(x_1),\dots,e_j(x_n))^\top=\frac{1}{\sqrt{n}}(1, \omega,\omega^{2j},\cdots,\omega^{(n-1)j})^\top.
\]
  $\{\bee_j\}_{j=1}^n$ forms an orthonormal basis of $\CC^n$ and moreover, $\bee_{j'n+j}=\bee_j$ for any $j'\in\ZZ$ and $j\in [n]$.  It is important to note that the kernel matrix $K_{s} = \frac{1}{n}(k_s(x_i,x_j))_{1\le i ,j \le n}\in\RR^{n\times n}$ is  circulant,  as for any $i,j\in [n]$, $k_s(x_i,x_j)=\sum_{t\in\ZZ}\mu_t^{s}e^{2\pi\ii t(x_i-x_j)}$ depends only on $x_i-x_j=(i-j)/n$. 

Therefore, our subsequent analysis will frequently use the following \textbf{properties of circulant matrices}:
\begin{itemize}
\item Any circulant matrix can be diagonalized using Fourier modes and thus, any two circulant matrices are commutable.
% , i.e., $AB=BA$ for any two circulant matrices $A,B$.
\item Let $A,B\in\RR^{n\times n}$ be two circulant matrices. Then, $AB$ and $A^\top$ are circulant. If $A$ is non-singular,  $A^{-1}$ is also circulant. 
\item All diagonal entries of a circulant matrix are the same, namely for any $i\in [n]$:
\begin{equation}\label{eqn: circulant-trace}
(A)_{i,i} = \frac{1}{n}\sum_{i=1}^n (A)_{i,i} = \fn\tr[A].
\end{equation}
\end{itemize}
These properties of circulant matrices  can be found on Wikipedia\footnote{\url{https://en.wikipedia.org/wiki/Circulant_matrix}}. 

We first note that $K_s$'s share the same eigenvectors (Fourier modes) for all $s\geq 1$:
\begin{lemma}\label{lem:eigen}
For any $j\in \ZZ$, $K_{s} \bee_j = \tilde{\mu}_{j,s} \bee_j$ with $\tilde{\mu}_{j,s} = \sum_{ j' \in \ZZ}\mu_{nj' + j}^s.$ Moreover, if we use $\{\hat{\mu}_{j,s}\}_{j = 1}^n$ to denote the non-decreasing permutation of $\{\tilde{\mu}_{j,s}\}_{j = 1}^n$, then
$\hat{\mu}_{j,s}\asymp (j+1)^{-s(\beta+1)}$.
\end{lemma}
\begin{proof}
By the aforementioned properties of circulant matrices,  $\{\bee_j\}_{j=1}^n$ are the $n$ orthonormal eigenvectors of $K_s$ and moreover, for $t\in [n]$, we have
\begin{align*}
K_s \bee_t &= \left(\sum_{j\in\ZZ} \mu_j^s \bee_j\bee_j^*\right) \bee_t = \left(\sum_{j=1}^n\sum_{j'\in\ZZ} \mu^s_{j'n+j}\bee_{j'n+j}\bee_{j'n+j}^*\right)\bee_t \\ 
&=\left(\sum_{j=1}^n\sum_{j'\in\ZZ} \mu^s_{j'n+j}\bee_{j}\bee_{j}^*\right)\bee_t= \left(\sum_{j'\in\ZZ} \mu^s_{j'n+j}\right)\left(\sum_{j=1}^n\bee_{j}\bee_{j}^*\bee_t\right)\\ 
&=  \left(\sum_{j'\in\ZZ} \mu^s_{j'n+t}\right)\bee_t = \tilde{\mu}_{t,s}\bee_t.
\end{align*}
In addition, if $|j| \le n/2$,
 \begin{align*}
     \tilde{\mu}_{j,s} &= \mu_j^s + \sum_{j' \neq 0}\mu_{nj'+j}^s \geq \mu_j^s \asymp (|j|+1)^{-s(\beta+1)}, \\
     \tilde{\mu}_{j,s} &= \mu_j^s + \sum_{j' \neq 0}\mu_{nj'+j}^s \asymp (|j|+1)^{-s(\beta+1)} + \sum_{j' \neq 0}(|nj'+j|+1)^{-s(\beta+1)}\\
     &\lesssim (|j|+1)^{-s(\beta+1)} + \sum_{k = 1}^\infty (k(n/2+1))^{-s(\beta+1)} \\ 
     &\lesssim (|j|+1)^{-s(\beta+1)} + (n/2+1)^{-s(\beta+1)} \lesssim (|j|+1)^{-s(\beta+1)}.
 \end{align*}
Therefore, 
\begin{equation*}
    \mu_j \asymp \begin{cases}
        (j+1)^{-s(\beta+1)},\quad &\text{ if }1\le j \le n/2;\\
        (n-j+1)^{-s(\beta+1)}, \quad &\text{ if } n/2 < j \le n.
    \end{cases}
\end{equation*}
Noticing that $\{\hat{\mu}_{j,s}\}_{j = 1}^n$ is the non-decreasing permutation of $\{\tilde{\mu}_{j,s}\}_{j = 1}^n$, we have that $\hat{\mu}_{j,s}  \asymp (j+1)^{-s(\beta+1)}$.
\end{proof}

\paragraph*{The KRR estimator.}
Let $\hk_{s}(\cdot) = (k_s(\cdot,x_1),\dots, k_s(\cdot,x_n))^\top$, and $y = (y_1,\dots,y_n)^\top$. Then,
\begin{equation*}
  \hat{f}_\lambda = \widehat{\alpha}_\lambda^\top \hk_{s}(\cdot)=\sum_{i=1}^n (\widehat{\alpha}_{\lambda})_jk_s(x_j,\cdot), \text{ where }  \widehat{\alpha}_\lambda = \frac{1}{n}(K_{s}+\lambda I)^{-1} y.
\end{equation*}
We will use $K$ and $\hk(\cdot)$ as the short of $K_{1}$ and $\hk_{1}(\cdot)$ for notation simplicity. 
\begin{lemma}
Let $\tilde{f} = \min_{f \in \cH_k, f(x_i) = f^*(x_i)} \|f\|_{\cH_k}$. Then, 
\[
    \|\tilde{f}\|_{\cH_k} \le 1,\,\quad \|\tilde{f} - f^*\|_{C(\TT)} \lesssim n^{-\frac{\beta}{2(\beta+1)}}.
\]
\end{lemma}
\begin{proof}
First, $\|\tf\|_{\cH_k}\leq \|f^*\|_{\cH_k}\leq 1$. Notice that $\|\tilde{f} - f^*\|_{\hat{L}_n} = 0$, $\|\tilde{f} - f^*\|_{\cH_k} \le 2$. By Lemma \ref{lem:linfty_generalization} with $s = 1$,  we have that
$$\|\tf-f^*\|_{\cH_k} \le 2 \sup_{\|f\|_{\cH_k} \le 1,\|f\|_{\hat{L}_n} = 0}\|f\|_{C(\TT)} \lesssim n^{-\frac{\beta
}{2}} \lesssim n^{-\frac{\beta}{2(\beta+1)}}.$$
\end{proof}

Note that KRR can not  distinguish if the target function is $f^*$ or $\tf$ as KRR only accesses the labels $\{y_i=f^*(x_i)+\xi_i\}_{i=1}^n$. Therefore, in the analysis of KRR, we can replace $f^*$ with $\tf$ without altering the conclusion. Denote by $\alpha^*\in\RR^n$ the minimum-norm solution and thus $\tf=\alpha^*\cdot \hk(\cdot)$, for which 
\[
n {\alpha^*}^\top K \alpha^*=\|\tf\|_{\cH_k}^2\leq 1,\qquad y = nK\alpha^* + \xi.
\]
Let $\xi = (\xi_1,\dots,\xi_n)^\top$ and $r := (\widehat{f}_{\lambda,s}(x_1) - f^*(x_1),\cdots,\widehat{f}_{\lambda,s}(x_n) - f^*(x_n))^\top\in\RR^n$.  Then, we have
\begin{equation}\label{eqn: alpha_lambda}
    \widehat{\alpha}_\lambda = \fn(K_{s} + \lambda I)^{-1}y = \fn(K_{s}+\lambda I)^{-1}\xi + (K_{s}+\lambda I)^{-1}K\alpha^*
\end{equation}
and  
\begin{align}
r = nK_s\widehat{\alpha}_\lambda  - nK\alpha^* = K_s(K_{s}+\lambda I)^{-1}\xi -n\lambda (K_s+\lambda I)^{-1}K\alpha^* =: A\xi + BK^{1/2}\alpha^*,
\end{align}
where $A=K_s(K_s+\lambda I)^{-1}$ and $B=-n\lambda (K_s+\lambda I)^{-1}K^{1/2}$.  Note that by the  properties of circulant matrices, both $A,B$ and $AA^\top$ and $BB^\top$ are circulant.

\paragraph*{Proof of Lemma~\ref{lemma: krr-hks1}: Part I.}
We next bound $\|A\xi\|_\infty$ and $\|BK^{1/2}\alpha^*\|_\infty$ separately.
\begin{itemize}
\item First, $A\xi$ is a mean-zero Gaussian random variable and for any $i\in [n]$, it holds that
\begin{align*}
\EE[(A_s\xi)_i^2] =\sigma^2 (AA^\top)_{i,i} =\frac{\sigma^2}{n} \tr[AA^\top] = \frac{\sigma^2}{n}\sumjn \frac{\hat{\mu}_{j,s}^2}{(\hat{\mu}_{j,s}+\lambda)^2}=:\sigma_n^2,
\end{align*}
where the second step uses \eqref{eqn: circulant-trace}.
Then, by the maximal inequality (see, e.g., \cite[Lemma 5.2]{van2014probability}), we have  with probability $1 - \delta/2 $ that
\begin{equation}\label{eqn:7111}
       \|A\xi\|_\infty  \lesssim \sigma_n\sqrt{\log(n/\delta)}=\sigma\sqrt{\frac{\log (n/\delta)}{n}\sum_{j=1}^{n}\frac{\hat{\mu}_{j,s}^2}{(\hat{\mu}_{j,s} + \lambda)^2}}\lesssim \sigma \sqrt{\frac{\log (n/\delta)}{n}} \lambda^{-\frac{1}{2s(\beta+1)}}
\end{equation}
where the last step uses  $\hat{\mu}_{j,s} \asymp (j+1)^{-s(\beta+1)} $ (see Lemma~\ref{lem:eigen}) and Lemma~\ref{lemma: dof-estimate}.

\item Second, let $q_i\in\RR^n$ be the canonical basis of $\RR^n$. Then, for any $i\in [n]$, we have 
\begin{align*}
|BK^{1/2}\alpha^*|_i&= |q_i^\top BK^{1/2}\alpha^*|\leq \|B^\top q_i\|\|K^{1/2}\alpha^*\| \leq  \sqrt{q_i^\top BB^\top q_i} \sqrt{\frac{1}{n}}\\ 
&=\sqrt{\frac{1}{n^2} \tr[BB^\top]} = \sqrt{\lambda^2\sumjn \frac{\hat{\mu}_{j,1}}{(\hat{\mu}_{j,s}+\lambda)^2}}\leq \lambda^{\frac{\beta}{2s(\beta+1)}},
\end{align*}
where the last steps uses $\hat{\mu}_{j,s} \asymp (j+1)^{-s(\beta+1)} $ (see Lemma~\ref{lem:eigen}) and Lemma~\ref{lemma: dof2}.
\end{itemize}
Combining together, we obtain 
 \begin{equation*}
    \|r\|_{\infty} =  \|A\xi\|_{\infty} + \|BK^{1/2}\alpha^*\|_\infty \lesssim \sigma \sqrt{\frac{\log (n/\delta)}{n}} \lambda^{-\frac{1}{2s(\beta+1)}} + \lambda^{\frac{\beta}{2s(\beta+1)}}.
\end{equation*}
\qed

\paragraph*{Proof of Lemma~\ref{lemma: krr-hks1}: Part II.}
By \eqref{eqn: alpha_lambda}, we have
\begin{align}\label{eqn: fHks}
\notag    \|\widehat{f}_{\lambda,s}\|_{\cH_{k_s}}^2 &= n\widehat{\alpha}_\lambda^\top K_{s}\widehat{\alpha}_\lambda \lesssim \frac{1}{n}\xi^\top K_{s}(K_{s}+\lambda I)^{-2}\xi + n(\alpha^*)^\top K(K_s+\lambda I)^{-1} K_{s}(K_{s}+\lambda I)^{-1}K\alpha^*\\ 
    &\leq \frac{1}{n}\xi^\top K_{s}(K_{s}+\lambda I)^{-2}\xi + n{\alpha^*}^\top K(K_s+\lambda I)^{-1}K\alpha^*.
\end{align}
We next turn to bound the two terms separately.
\begin{itemize}
\item First, we have
\begin{align}\label{eqn: Kalpha}
\notag n{\alpha^*}^\top K(K_s+\lambda I)^{-1}K\alpha^*&\leq n\|K^{1/2}(K_s+\lambda I)^{-1}K^{1/2}\|_{2}\|K^{1/2}\alpha^*\|^2\\ 
\notag &=\|(K_s+\lambda I)^{-1}K\|_{2}=\max_{j\in [n]} \frac{\hmu_{j,1}}{\hmu_{j,s}+\lambda}\\ 
&\asymp \sup_{j\in[n]}\frac{j^{-1}}{j^{-s}+\lambda}\leq \sup_{t\in (0,1]}\frac{t}{t^s+\lambda}\lesssim_{s}\lambda^{1/s-1},
\end{align}
where the last step uses Lemma \ref{lemma: dof3}.
\item Second, let $B=K_{s}^{1/2}(K_{s}+\lambda I)^{-1}$, we have
\begin{align*}
    \EE[\|B\xi\|^2] &= \sigma^2\mathrm{tr}( K_{s}(K_{s}+\lambda I)^{-2})= \sigma^2\sumjn \frac{\hmu_{j,s}}{(\hmu_{j,s}+\lambda)^2} \\ 
    &\asymp \sigma^2\sumjn \frac{j^{-s(\beta+1)}}{(j^{-s(\beta+1)}+\lambda)^2}\lesssim \sigma^2\lambda^{-1-\frac{1}{s(\beta+1)}},
\end{align*}
where the last step uses Lemma \ref{lemma: dof2}.
Notice that $\|B\|_{2} \le \lambda^{-\half}$. Then, by the Theorem 6.3.2 in \cites{vershynin2019high}, we have that with probability $1-\delta$,
\begin{equation}\label{eqn: Bxi}
    \|B\xi\|_2 \lesssim \sigma[\lambda^{-\half - \frac{1}{2s(\beta+1)}} + \lambda^{-\half} \sqrt{\log(1/\delta)}].
\end{equation}
\end{itemize}
Lastly, plugging  \eqref{eqn: Kalpha} and \eqref{eqn: Bxi} into \eqref{eqn: fHks}, we completes proof.
\qed

\paragraph*{Proof of Lemma~\ref{lemma: smooth-approx}}
Since $f^*\in \cH_k(1)$, we can write $f^* = \sum_{j \in \ZZ} \mu_j^{1/2} \alpha_j e_j$ with $\sum_{j\in\ZZ} a_j^2\leq 1$. Let $f^*_m = \sum_{|j| \le m} \mu_j^{1/2}\alpha_j e_j$. Then, we have
    \begin{align*}
        &\|f^*_m - f^*\|_{C(\TT)} = \sup_{x \in \TT}\left|\sum_{|j| > m}\mu_j^{1/2} \alpha_j e_j(x)\right| \le \sup_{x \in \TT}\sqrt{\sum_{|j| > m}\mu_j|e_j(x)|^2}\leq M\sqrt{\sum_{|j| > m}\mu_j} \asymp m^{-\frac{\beta}{2}},\\
        &\|f_m\|_{\cH_{k_s}} = \sqrt{\sum_{|j| \le m}\mu_j^{1-s} \alpha_j^2}\lesssim\sqrt{\left(\max_{|j|\leq m}\mu_j^{1-s}\right)\sum_{|j|\leq m}\alpha_j^2}\lesssim m^{\frac{(s-1)(\beta+1)}{2}}.
    \end{align*}
    By taking $m \asymp n^{\frac{1}{\beta+1}}$ and $\tilde{f}^* = f_m$, we complete the proof.
   
\qed

\paragraph*{Proof of Lemma \ref{lem:linfty_generalization}.}
By Theorem \ref{thm: dual} and Lemma \ref{lemma: F2-rkhs}, we know that
\begin{align}\label{eqn: torous-1}
\notag    \sup_{\|f\|_{\cH_{k_s}} \le 1, \|f\|_{\hL^\infty_n} \le \epsilon}\|f\|_{C(\TT)} &= \sup_{f \in \tilde{\cB}(1)}\inf_{c\in\RR^n}\left(\Big\|f - \sum_{i=1}^n c_i \phi(x_i,\cdot)\Big\|_\rho + \epsilon \|c\|_{1}\right)\\ 
    &=\sup_{x \in \TT}\inf_{c\in\RR^n}\left(\big\|\phi(x,\cdot) - \sum_{i=1}^n c_i \phi(x_i,\cdot)\big\|_\rho + \epsilon \|c\|_{1}\right),
\end{align}
where $\phi(x,v) = \sum_{j=1}^{\infty}\mu_j^{\frac{s}{2}} e_j(x)\overline{e_j(v)} = \sum_{j\in \ZZ} \mu_j^{\frac{s}{2}}\exp(2\pi \ii j (x-v))$  and the second step uses Lemma~\ref{lemma: barron2single-neuron}.
Moreover, for any $x \in \TT$,  it holds that
\begin{align}\label{eqn: torous0}
    \left\|\phi(x,\cdot) - \sum_{i=1}^n c_i \phi(x_i,\cdot)\right\|_\rho^2  = \sum_{ j \in \ZZ} \mu_j^s\big|e^{2\pi \ii jx} - \sum_{i=1}^n c_i e^{2\pi \ii j x_i}\big|^2 = \sum_{ j \in \ZZ} \mu_j^s\big|1 - \sum_{i=1}^n c_i e^{2\pi \ii j (x_i-x)}\big|^2.
\end{align}

 We will explicitly construct the coefficients $c\in\RR^n$ to complete the proof. Let $D_m(x)=\sum_{|j|\leq m}e^{2\pi\ii j x}$ be the $m$-th order Dirichlet kernel and 
\[
    F_m(x) = \frac{1}{m}\sum_{j=0}^{m-1} D_j(x) = \frac{1}{m}\left(\frac{\sin(\pi m x)}{\sin(\pi x)}\right)^2
\]
be the $m$-th order Fej\'{e}r kernel.  Without loss of generality, we assume  $n = 4 n'$ with $n' \in \NN^+$. Then, for $i\in [n]$, define the coefficient functions
\begin{equation*}
    c_i(x) = \frac{2}{n^2} \sum_{j = n'}^{3n'-1} D_j(x-x_i) = \frac{1}{n}\left[\frac{3}{2} F_{3n'}(x-x_i) -\frac{1}{2}  F_{n'}(x-x_i)\right].
\end{equation*}
\begin{itemize}
\item Notice that for any $j \in \ZZ$,
\begin{equation*}
    \frac{1}{n}\sum_{i=1}^n e^{2\pi \ii j x_i} = \begin{cases}
        0, \, \text{ if }n\nmid  j,\\
        1, \, \text{ if }n\mid j.
    \end{cases}
\end{equation*}
Therefore, for any $|j| \le  n'$ and $|j'| \le 3n' - 1$, we have that $|j-j'| < n$ and hence 
\begin{align}\label{eqn: torus1}
\notag    \sum_{i=1}^n c_i(x) e^{2\pi \ii j (x_i-x)} &= \frac{1}{2n'}\sum_{j = n'}^{3n'-1}\sum_{|j'| \le  j} e^{2\pi \ii (j'-j) x}\frac{1}{n}\sum_{i=1}^ne^{2\pi \ii (j-j')x_i} \\ 
&=\frac{1}{2n'}\sum_{j = n'}^{3n'-1}\sum_{|j'| 
    \le  j} e^{2\pi \ii (j'-j) x} \mathrm{1}_{\{j' = j\}} = 1.
\end{align}
\item Moreover, noticing that for any $|j|  \le n$ and $x \in \TT$, we have that
\begin{equation*}
   \frac{1}{n} \sum_{i=1}^n F_j(x-x_i) = \frac{1}{j}\sum_{j' = 0}^{j-1}\sum_{|j''| \le j'}e^{2\pi \ii j'' x}\frac{1}{n} \sum_{i = 1}^n e^{-2\pi \ii j''x_i} = \frac{1}{j}\sum_{j' = 0}^{j-1}\sum_{|j''| \le j'}e^{2\pi \ii j'' x}\mathrm{1}_{\{j'' = 0\}} = 1.
\end{equation*}
Therefore, 
\begin{equation}\label{eqn: torus2}
    \sum_{i = 1}^n|c_i(x)| \le \frac{3}{2n}\sum_{i = 1}^n F_{3n'}(x-x_i) + \frac{1}{2n}\sum_{i=1}^n F_{n'}(x-x_i) \le 2.
\end{equation}
\end{itemize}
Therefore, substituting \eqref{eqn: torus1} and \eqref{eqn: torus2} into \eqref{eqn: torous0} gives
\begin{align*}
  \Big\|\phi(x,\cdot) - \sum_{i=1}^n c_i(x) \phi(x_i,\cdot)\Big\|_\rho^2&= \sum_{|j| > n'} \mu_j^s\left| 1 - \sum_{i=1}^n c_i(x)e^{2\pi \ii j (x_i-x)}\right|^2  \\
   &\le\left(\sum_{|j| > n'} \mu_j^s\right) \sup_{x \in \TT}\left(1+\sum_{i=1}^n|c_i(x)|\right)^2\\ 
   & \lesssim  n^{-s(\beta+1)+1}
\end{align*}
and consequently,
\begin{align*}
    \sup_{\|f\|_{\cH_{k_s}} \le 1, \|f\|_{\hL^\infty_n} \le \epsilon}\|f\|_{C(\TT)} &= \sup_{x\in\TT}\inf_{c\in\RR^n}\left(\Big\|\phi(x,\cdot) - \sum_{i=1}^n c_i \phi(x_i,\cdot)\Big\|_\rho + \epsilon \|c\|_{1}\right)\\ 
    &\leq \sup_{x\in\TT}\left(\Big\|\phi(x,\cdot) - \sum_{i=1}^n c_i(x) \phi(x_i,\cdot)\Big\|_\rho + \epsilon \sum_{i=1}^n |c_i(x)|\right)\\ 
    &\lesssim n^{-\frac{s(\beta+1)-1}{2}} + \epsilon.
\end{align*}
We thus complete the proof.
\qed

\subsection{Proof of Theorem \ref{thm: Linfty-learning-lower-bound}}
\label{sec: proof-periodic-kernel-lower-bound}

We remark that  the proof follows a procedure similar to Part II of the proof to Theorem \ref{primal-bound}. Recall 
 $\rho=\mathrm{Unif}(\TT)$ and consider the feature function $\phi:\TT\times\TT\mapsto\RR$ is given by $$
 \phi(x,v)=\sum_{j=1}^{\infty}\sqrt{\mu_j}e_j(x)\overline{e_j(v)} = \sum_{j=1}^{\infty}\sqrt{\mu_j}e^{2\pi \ii j(x-v)}.
 $$ 
By Lemma \ref{lemma: F2-rkhs},
$\cH_k=\cF_{2,\rho}$ with respect to $\phi$. Then, the proof can be completed in three steps.
\begin{itemize}
\item \textbf{Step I.} 
By Proposition \ref{lemma-random-data} and the dual equivalence \eqref{dual-eq-3} in Theorem \ref{thm: dual}, we have
    \begin{align*}
       \inf_{T \in \cA_n}  &\sup_{\|f\|_{\cH_k} \leq 1} \EE\|T(\{(x_i,f(x_i)+\xi_i)\}_{i=1}^n) -f  \|_{C_0(\TT)} \\
        \gtrsim &\sup_{\|f\|_{\cH_k}\leq 1, \|f\|_{L^2(\brho)}\leq \sigma/\sqrt{n}} \|f\|_{C_0(\TT)}=\sup_{\|f\|_{\cF_{2,\rho}}\leq 1, \|f\|_{L^2(\brho)}\leq \sigma/\sqrt{n}} \|f\|_{C_0(\TT)}\\\\
        =&\sup_{\|g\|_{\tilde{\cB}}\leq 1}\,\,\inf_{h \in \tilde{\cF}_{2,\bar{\rho}}} \left(\left\|g - h \right\|_{\rho} + n^{-\half} \|h\|_{\tilde{\cF}_{2,\bar{\rho}}} \right) \\
        =&\sup_{\|g\|_{\tilde{\cB}}\leq 1}\,\,\inf_{c \in L^2(\brho)} \left(\left\|g - \int_{\TT}c(x)\phi(x,\cdot)\dd \brho(x) \right\|_{\rho} + n^{-\half} \|c\|_{\brho} ]\right)\\
        = &\sup_{y \in \TT}\,\,\inf_{c \in L^2(\brho)} \left(\left\|\phi(y,\cdot) - \int_{\TT}c(x)\phi(x,\cdot)\dd \brho(x) \right\|_{\rho} + n^{-\half} \|c\|_{\brho}\right)
    \end{align*}

     \item \textbf{Step II.}
    We shall utilize the translation invariance to show 
    \begin{align}\label{eqn: low-bound2}
\notag       \sup_{ y \in \TT}\,\, &\inf_{c \in L^2(\brho)} \left(\left\|\phi(y,\cdot) - \int_{\TT}c(x)\phi(x,\cdot)\dd \brho(x) \right\|_{\rho} + n^{-\half} \|c\|_{\brho}\right) \\
        \gtrsim &\sup_{y \in \TT}\,\,\inf_{c \in L^2(\rho)}\left( \left\|\phi(y,\cdot) - \int_{\TT}c(x)\phi(x,\cdot)\dd {\rho}(x) \right\|_{\rho} + n^{-\half} \|c\|_{\rho}\right).
    \end{align}

\item \textbf{Step III.} By using again the dual equivalence, we have
\begin{align*}
\sup_{y \in \TT}&\,\,\inf_{c \in L^2(\rho)}\left( \left\|\phi(y,\cdot) - \int_{\TT}c(x)\phi(x,\cdot)\dd {\rho}(x) \right\|_{\rho} + n^{-\half} \|c\|_{\rho}\right) \\ 
&=\sup_{\|g\|_{\tilde{\cB}} \le 1}\,\,\inf_{c \in L^2(\rho)}\left( \left\|g - \int_{\TT}c(x)\phi(x,\cdot)\dd {\rho}(x) \right\|_{\rho} + n^{-\half} \|c\|_{\rho}\right) \\
  &=\sup_{\|g\|_{\tilde{\cB}}\leq 1}\,\,\inf_{h \in \tilde{\cF}_{2,\rho}} \left(\left\|g - h \right\|_{\rho} + n^{-\half} \|h\|_{\tilde{\cF}_{2,\rho}} \right) \\
&= \II_{L^2(\rho), C_0(\TT)}\left(\cF_{2,\rho},\frac{\sigma}{\sqrt{n}}\right)
\end{align*}
Note that $\mu_j^{k,\rho}\asymp j^{-1-\beta}$. Then, $n_{\brho}(\sigma)\asymp n^{-1/(1+\beta)}$. By Lemma \ref{lemma: xyz}, 
\[
    \II_{L^2(\rho), C_0(\TT)}\left(\cF_{2,\rho},\frac{\sigma}{\sqrt{n}}\right)\gtrsim \sup_{\gamma\in\cP(\cX)} \Lambda_{k,\gamma}(n_{\rho}(\sigma)) = (n^{-1/(1+\beta)})^{-\beta/2} = n^{-\beta/(2\beta+2)}.
\]
\end{itemize}

\paragraph*{Proof of Eq.~\eqref{eqn: low-bound2}.}
    Noticing that $\sqrt{A^2+B^2}\le A+ B  \le \sqrt{2(A^2+B^2)}$ for any $A,B \geq 0$, it suffices to prove that
    \begin{align*}
        &\sup_{ y \in \TT}\,\,\inf_{c \in L^2(\brho)} \left\|\phi(y,\cdot) - \int_{\TT}c(x)\phi(x,\cdot)\dd \brho(x) \right\|_{\rho}^2 + n^{-1} \|c\|_{\brho}^2 \\\geq &\sup_{y \in \TT}\,\,\inf_{c \in L^2(\rho)} \left\|\phi(y,\cdot) - \int_{\TT}c(x)\phi(x,\cdot)\dd {\rho}(x) \right\|_{\rho}^2 + n^{-1} \|c\|_{\rho}^2.
    \end{align*}
    Noticing that $F_y(c):= \left\|\phi(y,\cdot) - \int_{\TT}c(x)\phi(x,\cdot)\dd \brho(x) \right\|_{\rho}^2 + n^{-1} \|c\|_{\brho}^2$ is a strongly-convex quadratic function in $L^2(\brho)$,  there must exist a unique minimizer (see, e.g. Corollary 11.16 in \cites{bauschke2017correction}) for $F_y(\cdot)$. Therefore, we can define $\bar{c}:\TT\times\TT\mapsto\RR$ by
    \begin{equation*}
        \bar{c}(y,\cdot) = \argmin_{c \in L^2(\brho)} F_y(c).
    \end{equation*}
    Since for any $v\in\TT$, $\phi(\cdot,v)$ is periodic, then $\bar{c}(\cdot,v)$ is also periodic by definition.

     Given $t \in \TT$, through the translation invariance of $\phi$ and $\rho$, we know that 
    \begin{equation}\label{eqn:translation_1}
          \begin{aligned}
        \left\|\phi(y+t,\cdot) - \int_{\TT}\bar{c}(y+t,x)\phi(x,\cdot)\dd \brho(x) \right\|_{\rho} 
        =& \left\|\phi(y+t,\cdot+t) - \int_{\TT}\bar{c}(y+t,x)\phi(x,\cdot+t)\dd \brho(x) \right\|_{\rho} \\
        =& \left\|\phi(y,\cdot) - \int_{\TT}\bar{c}(y+t,x)\phi(x-t,\cdot)\dd \brho(x) \right\|_{\rho}\\
        =& \left\|\phi(y,\cdot) - \int_{\TT}\bar{c}(y+t,x+t)\phi(x,\cdot)\dd \brho_t(x) \right\|_{\rho}, 
    \end{aligned}
    \end{equation}
    where we use the $\phi(x-t,v-t) = \phi(x,v)$ and let $\brho_t(  B) := \brho( B - t)$ for all Borel measurable sets $B$ on $\TT$.

    Consider the joint distribution for $P(x \in B_1, t \in B_2) = \int_{B_2}\bar{\rho}_t(B_1)\dd \rho(t)$ for any Borel measurable sets $B_1,B_2$ on $\TT$. Then $P(x|t) = \bar{\rho}_t(x)$ and 
    \begin{align*}
        P(x \in B_1) &= \int_{\TT}\bar{\rho}_t(B_1)\dd\rho(t) = \int_{\TT}\bar{\rho}(B-t) \dd\rho(t) = \int_{\TT} \int_{\TT} \mathrm{1}_{x \in B-t}\dd \bar{\rho}(x) \dd \bar{\rho}(t)\\
        &= \int_{\TT} \int_{\TT} \mathrm{1}_{t \in B-x} \dd \rho(t) \dd \bar{\rho}(x) = \int_{\TT}\rho(B-x)\dd\bar{\rho}(x) = \int_{\TT}\rho(B)\dd\bar{\rho}(x) = \rho(B),
    \end{align*} i.e., the $x$-marginal distribution of $P$ is $\rho$.
    
    %Consider the joint distribution 
    %$
    %P(x,t)=\brho_t(x)\rho(t), 
    %$
    %for which $P(x|t)=\brho_t(x)$ and $\int P(x,t)\dd x=\rho(t)\equiv 1$. Moreover,  $P(x):=\int P(x,t)\dd t= \int_{} \brho(x-t)\rho(t)\dd x=\int_{}\brho(x-t)\dd x=1$, i.e., the $x$-marginal distribution  is uniform. 
    % Now let $(t,x)$ obey the probability distribution $P$ on $\TT\times \TT$: $P(t \in B_1, x \in B_2) = \int_{t \in B_1} \brho_t(B_2)\dd \rho(t)$, then 
    % \begin{align*}
    %     P(x \in B_2) &= \int_{t \in \TT}\brho_t(B_2)\dd \rho(t) = \int_{t \in \TT}\int_{(x+t) \in B_2}\dd \brho(x) \dd \rho(t) \\
    %     = &\int_{x \in \TT}\int_{(x+t)\in B_2}\dd \rho(t)\dd\brho(x) = \int_{x \in \TT}\int_{t\in B_2}\dd \rho(t)\dd\brho(x) = \rho(B_2), 
    % \end{align*}
    % which means that $x$ obeys the uniform distribution $\rho$ on $\TT$. Let the conditional probability distribution $\tilde{\rho}_x$ satisfying that $P(t\in B_1,x\in B_2) = \int_{x \in B_2} \tilde{\rho}_x(B_1)\dd \rho(x)$ and 
    Let $$
    \bar{c}_*(y,x) = \int_{\TT}\bar{c}(y+t,x+t) P(\dd t|x). 
    $$
    Then,
    \begin{equation}\label{eqn:translation_2}
        \begin{aligned}
        \left\|\phi(y,\cdot) - \int_{\TT}\bar{c}_*(y,x)\phi(x,\cdot)\dd \rho(x) \right\|_{\rho}^2 &= \left\|\phi(y,\cdot) - \int_{\TT}\int_{\TT}\bar{c}(y+t,x+t)\phi(x,\cdot) P(\dd t|x)\dd \rho(x) \right\|_{\rho}^2 \\
        &= \left\|\phi(y,\cdot) - \int_{\TT\times \TT}\bar{c}(y+t,x+t)\phi(x,\cdot)\dd P(t,x) \right\|_{\rho}^2\\
        &= \left\|\phi(y,\cdot) - \int_{\TT}\int_{\TT}\bar{c}(y+t,x+t)\phi(x,\cdot)\dd \brho_t(x)\dd \rho(t) \right\|_{\rho}^2\\
        &\le \int_{\TT}\left\|\phi(y,\cdot) - \int_{\TT}\bar{c}(y+t,x+t)\phi(x,\cdot)\dd \brho_t(x) \right\|^2_{\rho}\dd\rho(t) \\
        &\stackrel{\mathrm{(i)}}{=} \int_{\TT}\left\|\phi(y+t,\cdot) - \int_{\TT}\bar{c}(y+t,x)\phi(x,\cdot)\dd \brho(x) \right\|_{\rho}^2 \dd\rho(t) \\
        &\stackrel{\mathrm{(ii)}}{=} \int_{\TT}\left\|\phi(t,\cdot) - \int_{\TT}\bar{c}(t,x)\phi(x,\cdot)\dd \brho(x) \right\|_{\rho}^2 \dd\rho(t),
    \end{aligned}
    \end{equation}
    where $\mathrm{(i)}$ uses \eqref{eqn:translation_1} and $\mathrm{(ii)}$ uses the periodicity of $\phi(\cdot,v)$ and $\bar{c}(\cdot,v)$  for any $v\in\TT$.

    \begin{equation}\label{eqn:translation_3}
        \begin{aligned}
           \|\bar{c}_*(y,\cdot)\|_\rho^2 &=  \int_{\TT}|\bar{c}_*(y,x)|^2\dd\rho(x) =\int_{\TT}\left|\int_{\TT}\bar{c}(y+t,x+t)P(\dd t| x)\right|^2\dd\rho(x) \\
           &\le \int_{\TT}\int_{\TT}|\bar{c}(y+t,x+t)|^2P(\dd t|x)\dd\rho(x) 
           =  \int_{\TT}\int_{\TT}|\bar{c}(y+t,x+t)|^2\dd \brho_t(x)\dd\rho(t)\\& = \int_{\TT}\int_{\TT}|\bar{c}(y+t,x)|^2\dd \brho(x)\dd\rho(t)
           = \int_{\TT}\int_{\TT}|c(t,x)|^2\dd \brho(x)\dd\rho(t) \\
           &= \int_{\TT}\|c(t,\cdot)\|_{\brho}^2\dd\rho(t).
        \end{aligned}
    \end{equation}
    Combining \eqref{eqn:translation_2} and \eqref{eqn:translation_3}, we know that for any $y \in \TT$,
    \begin{align*}
        &\inf_{c \in L^2(\rho)} \left(\left\|\phi(y,\cdot) - \int_{\TT}c(x)\phi(x,\cdot)\dd {\rho}(x) \right\|_{\rho}^2 + n^{-1} \|c\|_{\rho}^2\right) \\
        \le&  \left\|\phi(y,\cdot) - \int_{\TT}\bar{c}_*(y,x)\phi(x,\cdot)\dd {\rho}(x) \right\|_{\rho}^2 + n^{-1} \|\bar{c}_*(y,\cdot)\|_{\rho}^2 \\
        \le&\int_{\TT}\left[\left\|\phi(t,\cdot) - \int_{\TT}\bar{c}(t,x)\phi(x,\cdot)\dd \brho(x) \right\|_{\rho}^2+ n^{-1}\|\bar{c}(t,\cdot)\|_{\brho}^2\right]\dd \rho(t) \\
        \le& \sup_{ y \in \TT}\,\,\inf_{c \in L^2(\brho)}\left( \left\|\phi(y,\cdot) - \int_{\TT}c(x)\phi(x,\cdot)\dd \brho(x) \right\|_{\rho}^2 + n^{-1} \|c\|_{\brho}^2\right),
    \end{align*}
    which concludes the proof.

\section{Auxiliary lemmas}

\begin{lemma}\label{lemma: dof2}
Suppose that $a_j\asymp j^{-\alpha}, b_j\asymp j^{-\gamma}$ for $j\in \NN^+$ with $\alpha > 1$ and $2\gamma-\alpha + 1\geq 0$. Then, $\sum_{j=1}^\infty\frac{ a_j}{(b_j+\lambda)^2}\lesssim_{\alpha,\gamma} \lambda^{-2+(\alpha-1)/\gamma}.$
\end{lemma}
\begin{proof}
Let $m_{\lambda}=\inf\{j\in \NN: b_j\geq \lambda\}$. Then, $m_\lambda\asymp_{\gamma}\lambda^{-1/\gamma}$ and consequently,
\begin{align*}
    \sum_{j=1}^\infty \frac{a_j}{(b_j+\lambda)^2} &= \sum_{j\leq m_\lambda} \frac{a_j}{(b_j+\lambda)^2} + \sum_{j>m_\lambda}\frac{a_j}{(b_j+\lambda)^2}\\ 
    &\lesssim \sum_{j\leq m_\lambda} \frac{a_j}{b_j^2} + \sum_{j>m_\lambda}\frac{a_j}{\lambda^2}\\ 
    &\asymp \sum_{j\leq m_\lambda} j^{-\alpha+2\gamma} + \lambda^{-2}\sum_{j>m_\lambda}j^{-\alpha} \\ 
    &\lesssim_{\alpha,\gamma} (m_\lambda)^{-\alpha+2\gamma + 1} + \lambda^{-2}m_{\lambda}^{-\alpha+1}\\ 
    &\lesssim_{\alpha,\gamma} \lambda^{-2+(\alpha-1)/\gamma} + \lambda^{-2}\lambda^{(\alpha-1)/\gamma}\\ 
    &\lesssim \lambda^{-2+(\alpha-1)/\gamma}.
\end{align*}

\end{proof}

\begin{lemma}\label{lemma: dof-estimate}
Let $\lambda_j\asymp j^{-\gamma}$ for $j\in\NN^+$. Then, $\sum_{j=1}^\infty \frac{\lambda_j^2}{(\lambda_j+\lambda)^2}\lesssim \lambda^{-1/\gamma}$.
\end{lemma}
\begin{proof}
It follows trivially from Lemma~\ref{lemma: dof2} by letting $a_j=\lambda_j^2\asymp j^{-2\gamma}$ and $b_j=\lambda_j\asymp j^{-\gamma}$.
\end{proof}

\begin{lemma}\label{lemma: dof3}
Let $f(t)=\frac{t}{t^s+\lambda}$ where $s\geq 1$. Then, $\sup_{t\in [0,1]}f(t)\lesssim_s \lambda^{1/s-1}$.
\end{lemma}
\begin{proof}
Note $f'(t)=\frac{(t^s+\lambda)-tst^{s-1}}{(t^s+\lambda)^2}$. $f'(t)=0$ gives $t=\bar{t}=(\lambda/(s-1))^{1/s}$. Thus, 
\[
\sup_{t\in [0,1]}f(t)=f(\bar{t})=\frac{(s-1)^{1-1/s}}{s}\lambda^{1/s-1}.
\]
\end{proof}

\begin{lemma}\label{lemma: barron2single-neuron}
Let $\cB\subset C_0(\cX)$ be the Barron space associated with the feature function $\phi:\cX\times \cV\mapsto\RR$. Then, for any $g_1,g_2,\dots,g_n\in L^2(\rho)$ and $\lambda\geq 0$, it holds that
\[
\sup_{f\in \cB(1)}\inf_{c\in\RR^n}\left(\Big\|f-\sumin c_i g_i\Big\|_\rho + \lambda \|c\|_1\right) = \sup_{v\in\cV}\inf_{c\in\RR^n}\left(\Big\|\phi(\cdot,v)-\sumin c_i g_i\Big\|_\rho + \lambda \|c\|_1\right)
\]
\end{lemma}
\begin{proof}
\textbf{Step 1: Verification of ``$\geq$''.}
This direction follows directly as the set of single neurons $\{\phi(\cdot,v): v\in \cV\}$ is a subset of $\cB(1)$. 

\textbf{Step 2: Verification of ``$\leq$''.}
Denote the RHS of the equality by $I$ and for each $v\in\cV$, define the optimal coefficients 
\[
    c(v) =(c_1(v),\dots,c_n(v))= \argmin_{c\in\RR^n}\left(\Big\|\phi(\cdot,v)-\sumin c_i g_i\Big\|_\rho + \lambda \|c\|_1\right)\in\RR^n.
\]
Then, for any $f\in \cB(1)$ and $\epsilon>0$, there exists a finite representation $\{(a_j,v_j)\}_{j=1}^m$ such that $\|f-\sumjm a_j\phi(\cdot,v_j)\|_\rho\leq \epsilon$ and $\sumjm |a_j|\leq 1$. 
Let 
\[
\bar{c} = \sumjm a_j c(v_j)\in\RR^n. 
\]
Then,
\begin{align*}
\inf_{c\in\RR^n}\left(\Big\|f-\sumin c_i g_i\Big\|_\rho + \lambda \|c\|_1\right)&\leq \Big\|f-\sumin \bar{c}_i g_i\Big\|_\rho + \lambda \|\bar{c}\|_1\\ 
&\leq \epsilon + \Big\|\sumjm a_j \phi(\cdot,v_j)-\sumin \sumjm a_jc_i(v_j) g_i\Big\|_\rho + \lambda \sumin\left|\sumjm a_j c_i(v_j)\right|\\ 
&\leq \epsilon + \sumjm |a_j|\left(\Big\|\phi(\cdot,v_j)-\sumin c_i(v_j) g_i\Big\|_\rho + \lambda \sumin\left|c_i(v_j)\right|\right)\\ 
&\leq \epsilon + I.
\end{align*}
Taking $\epsilon\to 0$, we established the desired equality.
\end{proof}

\end{document}